\documentclass[twoside]{article}

\usepackage{ifthen}
\newboolean{showcomments}
\setboolean{showcomments}{true}
\usepackage{color}
\usepackage{todonotes}

\definecolor{bleudefrance}{rgb}{0.19, 0.55, 0.91}
\definecolor{ao(english)}{rgb}{0.0, 0.5, 0.0}

\newcommand{\addcite}[0]{\ifthenelse{\boolean{showcomments}}
{\textcolor{purple}{(add cite(s)) }}{}}%

\newcommand{\addref}[0]{\ifthenelse{\boolean{showcomments}}
{\textcolor{purple}{(add ref) }}{}}%

\newcommand{\myparagraph}[1]{\vspace{1mm}\noindent\textbf{#1.}}

\newcommand{\enrique}[1]{  \ifthenelse{\boolean{showcomments}}
{\todo[inline,color=bleudefrance]{Enrique: #1}}{}}
\newcommand{\rene}[1]{  \ifthenelse{\boolean{showcomments}}
{\todo[inline,color=cyan]{Ren\'e: #1}}{}}
\newcommand{\emmargin}[1]{\ifthenelse{\boolean{showcomments}}{\marginpar{\color{bleudefrance}\tiny EM: #1}}{}}
\newcommand{\hancheng}[1]{  \ifthenelse{\boolean{showcomments}}
{\todo[inline,color=orange]{Hancheng: #1}}{}}
\newcommand{\ziqing}[1]{  \ifthenelse{\boolean{showcomments}}
{\todo[inline,color=red]{Ziqing: #1}}{}}
\newcommand{\salma}[1]{  \ifthenelse{\boolean{showcomments}}
{\todo[inline,color=yellow]{Salma: #1}}{}}
\newcommand{\zxmargin}[1]{\ifthenelse{\boolean{showcomments}}{\marginpar{\color{purple}\tiny ZX: #1}}{}}
\newcommand{\stmargin}[1]{\ifthenelse{\boolean{showcomments}}{\marginpar{\color{red}\tiny ST: #1}}{}}

\newcommand{\hl}[1]{\ifthenelse{\boolean{showcomments}}
{\textcolor{red}{#1}}{#1}}

\newboolean{showedits}
\setboolean{showedits}{false}
\usepackage[markup=underlined]{changes}
\definechangesauthor[color=bleudefrance]{EM}
\newcommand{\aem}[1]{
\ifthenelse{\boolean{showedits}}
{\added[id=EM]{#1}}
{\!#1\hspace{-4.75pt}}
}
\newcommand{\repem}[2]{
\ifthenelse{\boolean{showedits}}
{\replaced[id=EM]{#1}{#2}}
{\!#1\hspace{-4.75pt}}
}
\newcommand{\dem}[1]{
\ifthenelse{\boolean{showedits}}
{\deleted[id=EM]{#1}}
{}
}

\usepackage{amsmath,amsfonts,bm}
\usepackage{amsthm}

\def\1{\bm{1}}

\DeclareMathAlphabet{\mathsfit}{\encodingdefault}{\sfdefault}{m}{sl}
\SetMathAlphabet{\mathsfit}{bold}{\encodingdefault}{\sfdefault}{bx}{n}

\newcommand{\R}{\mathbb{R}}

\newtheorem{thm}{Theorem}[section]
\newtheorem{lem}{Lemma}[section]

\newtheorem{claim}{Claim}[section]
\newtheorem{asmp}{Assumption}[section]

\newtheorem{rem}{Remark}[section]

\usepackage{amssymb}
\usepackage{enumitem}
\usepackage{hyperref}
\usepackage{url}
\usepackage[accepted]{aistats2025}
\usepackage[round]{natbib}

\begin{document}

%
\runningtitle{Understanding the Learning Dynamics of LoRA in Matrix Factorization}

%
\runningauthor{Xu, Min, MacDonald, Tarmoun, Luo, Mallada, Vidal}

\twocolumn[

\aistatstitle{Understanding the Learning Dynamics of LoRA: A Gradient Flow Perspective on Low-Rank Adaptation in Matrix Factorization}

\aistatsauthor{ Ziqing Xu$^\dagger$ \And Hancheng Min$^\dagger$ \And Lachlan Ewen MacDonald$^\dagger$ \And Jinqi Luo$^\dagger$}
\vspace{0.5cm}
\aistatsauthor{Salma Tarmoun$^\dagger$ \And Enrique Mallada$^*$ \And Ren\'e Vidal$^\dagger$ }
\vspace{0.5cm}
\aistatsaddress{ $^\dagger$University of Pennsylvania \And $^*$Johns Hopkins University } 
]

\begin{abstract}
Despite the empirical success of Low-Rank Adaptation (LoRA) in fine-tuning pre-trained models, there is little theoretical understanding of how first-order methods with carefully crafted initialization adapt models to new tasks. 
In this work, we take the first step towards bridging this gap by theoretically analyzing the learning dynamics of LoRA for matrix factorization (MF) under gradient flow (GF), emphasizing the crucial role of initialization. 
For small initialization, we theoretically show that GF converges to a neighborhood of the optimal solution, with smaller initialization leading to lower final error. 
Our analysis shows that the final error is affected by the misalignment between the singular spaces of the pre-trained model and the target matrix, and reducing the initialization scale improves alignment. 
To address this misalignment, we propose a spectral initialization for LoRA in MF and theoretically prove that GF with small spectral initialization converges to the fine-tuning task with arbitrary precision. 
Numerical experiments from MF and image classification validate our findings.
\end{abstract}

\section{INTRODUCTION}\label{sec:intro}
Low-Rank Adaptation~\citep{hu2022lora} (LoRA) has proven to be a highly effective and parameter-efficient method for fine-tuning pre-trained models,
showing significant empirical success in both natural language processing~\citep{meng2020text, liang2022holistic, zhang2023sentiment, yang2024advancing} and computer vision tasks~\citep{zhai2022scalingvit,filatov2023lora}.
This method can be broadly described as follows: given a pre-trained model $f(x; W_1, \dots, W_L)$ parameterized by $W_1, \dots, W_L$, LoRA modifies it to $f(x; W_1 \!+\! B_1 A_1, \dots, W_L \!+\! B_L A_L)$, where $W_i \!\in\! \R^{n_{i+1} \times n_i}, B_i \!\in\! \R^{n_{i+1} \times r_i}, A_i \!\in\! \R^{r_i \times n_i},$ and $r_i \!\ll\! \min(n_i, n_{i+1})$. In the fine-tuning stage, $A_i$ and $B_i$ are the trainable parameters, while $W_i$ remains fixed, with $B_i$ initialized as zero matrices and $A_i$ initialized randomly.

Despite LoRA's empirical success, its theoretical underpinnings remain poorly understood. A key question is why pre-trained models can be efficiently fine-tuned with LoRA using gradient-based methods,
despite the non-convex objective. Moreover, given LoRA’s fine-tuning nature and specific initialization, it is essential to explore how the pre-trained model and initialization affect its learning dynamics. 
Existing theoretical work primarily focuses on LoRA's expressiveness~\citep{zeng2023expressive} or characterizing the optimization landscape and generalization in the Neural Tangent Kernel (NTK) regime~\citep{malladi2023kernel,jang2024lora}. Other studies~\citep{hayou2024impact, hayou2024lora+} suggest different learning rate scales for $A_i$ and $B_i$, and initializing $A_i$ to zero while initializing $B_i$ randomly yields better performance on average compared to the reverse. However, to the best of our knowledge, no prior work has provided a thorough theoretical analysis of LoRA’s learning dynamics with explicit convergence rates or guarantees on the accuracy of the solution.

\myparagraph{Contributions} In this paper, we study LoRA for fine-tuning a matrix factorization (MF) task via gradient flow (GF). Our analysis represents an initial step towards understanding the learning dynamics of LoRA. Our contributions are as follows:

\begin{enumerate}[leftmargin=0.45cm]
    \item We first theoretically analyze the learning dynamics of LoRA under GF with LoRA-based small initialization, focusing on the case where the difference between the target matrices of the pre-training and fine-tuning tasks is {rank-one}, and assuming the pre-training MF task is perfectly solved by the pre-trained weights. Our analysis reveals two key phases: 
    An \textit{alignment} phase, where GF orients the singular vectors of the LoRA weights to correct the misalignment between the model and the fine-tuning task, with smaller initialization scale implying greater alignment.
    A \textit{local convergence} phase, where the loss decreases linearly for a finite time until reaching a threshold determined by the initialization scale, with smaller initialization scale implying smaller final loss. 
    \item 
    Motivated by the dependence of the final loss on model misalignment, we propose a spectral initialization that incorporates information from both the fine-tuning task and the pre-trained model. We theoretically prove that GF with small spectral initialization can converge to the target matrix with arbitrary precision for general MF fine-tuning tasks.
    \item We validate our theoretical findings through extensive experiments on MF and several image classification tasks. In both settings, we observe that smaller scales of LoRA-based initialization lead to better alignment and lower final training error under GD. Additionally, in certain computer vision tasks, we observe improved test performance as the initialization scale decreases. While our focus is on LoRA's optimization process, these results suggest that the initialization scale may also impact generalization performance. Finally, GF with small spectral initialization in MF converges to the target matrix with arbitrary precision.
\end{enumerate}

\subsection{Related Work}
\myparagraph{Theory of LoRA} ~\cite{zeng2023expressive} analyze the expressiveness of LoRA, and show that under certain assumptions, LoRA can approximate any deep linear network, multi-layer feed-forward network and transformer network. However, it is unclear whether LoRA, optimized via gradient-based algorithms, can learn these weights efficiently. Another line of work~\citep{malladi2023kernel, jang2024lora} studies LoRA in the NTK regime. Specifically, ~\cite{malladi2023kernel} characterize the conditions under which one can study LoRA in the NTK regimes, and ~\cite{jang2024lora} show that when the LoRA rank is $r_i\gtrsim\sqrt{N}$ where $N$ is the number of samples, the optimization landscape of LoRA has no spurious local minima, and GD can find $\mathcal{O}(\sqrt{N})$-rank solutions that generalize well. Additional research \citep{hayou2024impact, hayou2024lora+} experimentally explores the effects of learning rates and initialization, recommending different learning rates for $A_i$ and $B_i$, and showing that initializing $B_i$ as zero matrices and $A_i$ randomly improves performance compared to the reverse. However, none of these studies provide explicit convergence rates or consider the influence of pre-trained models in the optimization process, leaving gaps in the understanding of LoRA’s learning dynamics.

\myparagraph{Learning Dynamics of Low-Rank MF with Small Initialization} Our analysis of LoRA builds upon techniques developed for low-rank MF with small initialization, which falls outside the NTK regime. Specifically, \cite{stoger2021small, jin2023understanding, pmlr-v195-soltanolkotabi23a} analyze GD under small initialization, demonstrating the convergence and learning dynamics of GD. However, applying these techniques to LoRA in the context of MF introduces several key differences, such as distinct learning dynamics and initialization methods. We refer readers to \S\ref{subsec:challenge} for a detailed discussion of these differences and the additional challenges they present.

\myparagraph{Spectral Initialization for LoRA} 
Several spectral initialization methods for LoRA~\citep{balazy2024lora, lin2024nora, meng2024pissa, wang2024milora} have been proposed, based on the singular value decomposition (SVD) of pre-trained weights. 
Specifically,~\cite{balazy2024lora, lin2024nora, meng2024pissa} suggest initializing the left (and right) singular spaces of LoRA weights $B_i$ (and $A_i$) using the top-$r_i$ singular spaces of the pre-trained weights, ensuring that during training, $B_i A_i$ aligns with the principal components of $W_i$. 
Conversely,~\cite{wang2024milora} initialize $B_i$ and $A_i$ using the bottom-$r_i$ singular spaces of the pre-trained weights. 
Despite their differences, both approaches report better performance than standard LoRA.
However, these methods overlook the fine-tuning task. 
In Appendix~\ref{app:example}, we provide examples where previous methods fail to fine-tune pre-trained models for MF, underscoring the importance of considering the fine-tuning data when designing spectral initialization for LoRA. 
In this work, we propose a spectral initialization that leverages both the pre-trained weights and the fine-tuning target matrix, demonstrating superior performance compared to standard LoRA-based initialization (see \S\ref{sec:rank-r} for details).

\subsection{Notation}
We use lower case letters $a$ to denote a scalar, and capital letters $A$ and $A^\top$ to denote a matrix and its transpose. We use $\|A\|_F$ and $\|A\|$ to denote the Frobenius and spectral norms of $A$, and $A[i, j]$ to denote its $(i,j)$-th element.  We use $\boldsymbol{a}$ to denote a vector. For a function $f(t)$, we use $\dot f(t):=\frac{d}{dt}f(t)$ to denote its time derivative.

\section{PRELIMINARIES}\label{sec:prelim}
In this section, we first introduce the problem of applying LoRA to fine-tune MF tasks. Then we outline the key assumptions that underlie our analysis, addressing their implications. 
Furthermore, we highlight the critical distinctions between LoRA and traditional MF, and discuss the associated challenges in the analysis.

We consider applying LoRA to the classic MF. Specifically, we assume that we have a solution $(W_1,W_2)$ to a pre-training task of factorizing a target matrix $Y_{\mathrm{pre}}$:
\begin{align}
    (W_2, W_1) \in \arg\min_{U, V} \frac{1}{2}\lVert Y_{\mathrm{pre}}-UV\rVert_F^2 \label{eqn:obj_pretrain}
\end{align}
where $W_2\in\R^{m\times h}, W_1\in\R^{h\times n}$ are the \emph{pre-trained weights}. Problem~\ref{eqn:obj_pretrain} covers low-rank MF~\citep{chi2019nonconvex,koren2009matrix} ($h\!\leq\!\min(m, n)$), and overparametrized MF~\citep{ li2018algorithmic,tarmoun2021understanding, min2021explicit,pmlr-v108-geyer20a,9665358,xu2023linear} ($h>\min(m, n)$).

We are interested in solving a fine-tuning task using LoRA by minimizing the following objective
\begin{align}
    \min_{A_1, A_2, B_1, B_2} \underbrace{\frac{1}{2}\lVert Y_{\mathrm{ft}}\!-\!(W_2\!+\!B_2A_2)(W_1\!+\!B_1A_1)\rVert_F^2}_{\!:=\!L(A_1,A_2,B_1,B_2)}\label{eqn:obj_lora} 
\end{align}
where $Y_{\mathrm{ft}}\in\R^{m\times n}$ is the target data matrix for this fine-tuning task, $A_1\in\R^{r\times n},B_1\in\R^{h\times r}, A_2\in\R^{r\times h}, B_2\in\R^{m\times r}$ are \emph{LoRA weight matrices}, and $r$ is the \emph{LoRA rank}. 
In practice, the chosen \emph{LoRA rank} is typically much smaller than the dimensions of the pre-trained weights, satisfying $r\ll \min(m, h, n)$.

\myparagraph{Gradient Flow as Training Algorithm} We consider solving Problem~\ref{eqn:obj_lora} using GF
\begin{align}\label{eqn:gf}
    \dot A_i(t) \!=\!-\frac{\partial L(t)}{\partial A_i(t)}, \ \dot B_i(t) \!=\!-\frac{\partial L(t)}{\partial B_i(t)}, \ i=1,2
\end{align}
where we use $t$ to denote the time index in the gradient flow, and use $L(t)$ as a shorthand for $L(A_1(t), A_2(t), B_1(t), B_2(t))$. Throughout the paper, we will use $A_i, B_i$ to denote $A_i(t), B_i(t)$  when no confusion is caused by dropping $t$. 

\myparagraph{LoRA-based Initialization} We initialize $A_1(0), A_2(0), B_1(0), B_2(0)$ entry-wise i.i.d. as follows: $A_1(0)[i,j],A_2(0)[i,j]\!\sim\! \mathcal{N}(0, \alpha^2), B_1(0)[i,j]\!=\!B_2(0)[i,j]\!=\!0$.
We are particularly interested in the training dynamics when $\alpha$ is small, or equivalently when $A_1$ and $A_2$ are initialized near zero matrices. The learning dynamics of GF/GD with weights initialized close to zero have been studied in the context of MF~\citep{stoger2021small, jin2023understanding, pmlr-v195-soltanolkotabi23a} and two-layer neural networks~\citep{bui2021inductive, min2024earlyneuronalignmenttwolayer}. Throughout this paper, we use the term \textit{small initialization} to refer to LoRA-based initialization with a small $\alpha$, unless otherwise specified.

We make the following assumptions in the paper:
\begin{asmp}\label{asmp:pre-trained_w}
We assume that $W_2,W_1$ factorize $Y_{\mathrm{pre}}$ perfectly, i.e., $Y_{\mathrm{pre}}\!=\!W_2W_1$. Moreover, we assume that $W_2, W_1$ satisfy: $\ W_2\!=\!U_Y\Sigma_{W_2}G^\top, W_1\!=\!G\Sigma_{W_1}V_Y^\top$, where $G$ is an orthogonal matrix.
\end{asmp}
\begin{asmp}\label{asmp:target_Y}
Consider the SVD of $\ Y_{\mathrm{pre}}\!=\!U_Y\Sigma_{\mathrm{pre}}V_Y^\top$. Then, we assume the SVD of $\Delta Y\!:=\!Y_{\mathrm{ft}}\!-\!Y_{\mathrm{pre}}$ has the same left and right singular matrices as $Y_{\mathrm{pre}}$, i.e., $\Delta Y\!=\!U_Y \Sigma_{\Delta Y}V_Y^\top$.
\end{asmp}
Assumption~\ref{asmp:pre-trained_w} ensures that the pre-training task is solved perfectly, and that the left and right singular matrices of $W_1$ and $W_2$, respectively, are perfectly aligned with the singular matrices of $Y_{\mathrm{pre}}$. This assumption is satisfied if the pre-trained task is solved using GF with either spectral initialization~\citep{li2019nonconvex, luo2019optimal, lu2020phase} or balanced initialization~\citep{bui2021inductive, min2024earlyneuronalignmenttwolayer}.

Assumption~\ref{asmp:target_Y} ensures certain similarities between pre-trained and fine-tuning tasks. Empirical evidence from several studies \citep{yosinski2014transferable,peters2019tune, matsoukas2022makes} demonstrates that fine-tuning yields optimal results when there exists a substantial correlation between the pre-training and fine-tuning tasks. In the context of MF, we formally characterize this similarity as a $\mathrm{rank}(\Delta Y)$ modification of the pre-training target matrix to the fine-tuning target matrix. Specifically, this modification occurs along the subspace spanned by a subset of singular vectors of the pre-trained target matrix, $Y_{\mathrm{pre}}$, and the magnitude of this modification along the subspaces is quantified by the parameter $\Sigma_{\Delta Y}$.

\subsection{Challenges in the Analysis of Learning Dynamics of Problem~\ref{eqn:obj_lora} Optimized via GF}\label{subsec:challenge}
In this section, we discuss the difference between Problem~\ref{eqn:obj_lora} and MF, and how this difference leads to additional challenges in analyzing the learning dynamics of Problem~\ref{eqn:obj_lora} optimized via GF.

\myparagraph{Initialization of LoRA Weights Near Saddle Points} In LoRA, the weights $A_1$ and $A_2$ are randomly initialized from a Gaussian distribution, while $B_1$ and $B_2$ are initialized as zero matrices. This contrasts with MF, where both factor matrices are typically initialized randomly. Upon inspection, it is evident that when both $A_i$ and $B_i$ are zero matrices, they correspond to a saddle point in the objective.
Thus, the LoRA weights are initialized near these saddle points when the initialization scale $\alpha$ is small. It remains unclear whether GF will converge to a global optimum or become trapped in a local minimum or saddle point for Problem~\ref{eqn:obj_lora}.

\myparagraph{Complex Learning Dynamics} In low-rank MF, the objective is to approximate a target matrix by the product of two low-rank matrices. 
However, in LoRA, the model adopts a distinct architecture due to the influence of the pre-trained weights, as shown below:
\begin{align}
    &(W_2\!+\!B_2A_2)(W_1\!+\!B_1A_1)\nonumber\\
    =&W_2B_1A_1\!+\!B_2A_2W_1\!+\!B_2A_2B_1A_1
\end{align}
As one can see, the products of the low-rank matrices are influenced by the pre-trained weights $W_2$ and $W_1$, resulting in terms $W_2B_1A_1 \!+\! B_2A_2W_1$. 
Moreover, the fine-tuning model includes a fourth-order term, $B_2A_2B_1A_1$. 
These adjustments create dependencies between the low-rank factors and the fixed pre-trained weights and introduce higher-order terms of the LoRA weights, which complicates the learning dynamics and poses new challenges for analyzing the optimization process.

The aforementioned challenges indicate that Problem~\ref{eqn:obj_lora} is more complex than classical MF. In the following sections, we address these challenges and derive convergence results for LoRA.
\section{LEARNING DYNAMICS UNDER SMALL INITIALIZATION}\label{sec:rank-one}
In this section, we present our main theorem on the learning dynamics of Problem~\ref{eqn:obj_lora} trained via GF for the case where $\mathrm{rank}(\Delta Y) \!=\! 1$. In addition, we provide an overview of the proof strategy by formally characterizing the \textit{alignment} and \textit{local convergence} phases, and examine how key parameters—such as pre-trained weights, initialization scale, and the target matrix—affect the learning behavior.

%
\begin{figure*}[!t]
    \centering
    \includegraphics[height=0.3125\textwidth]{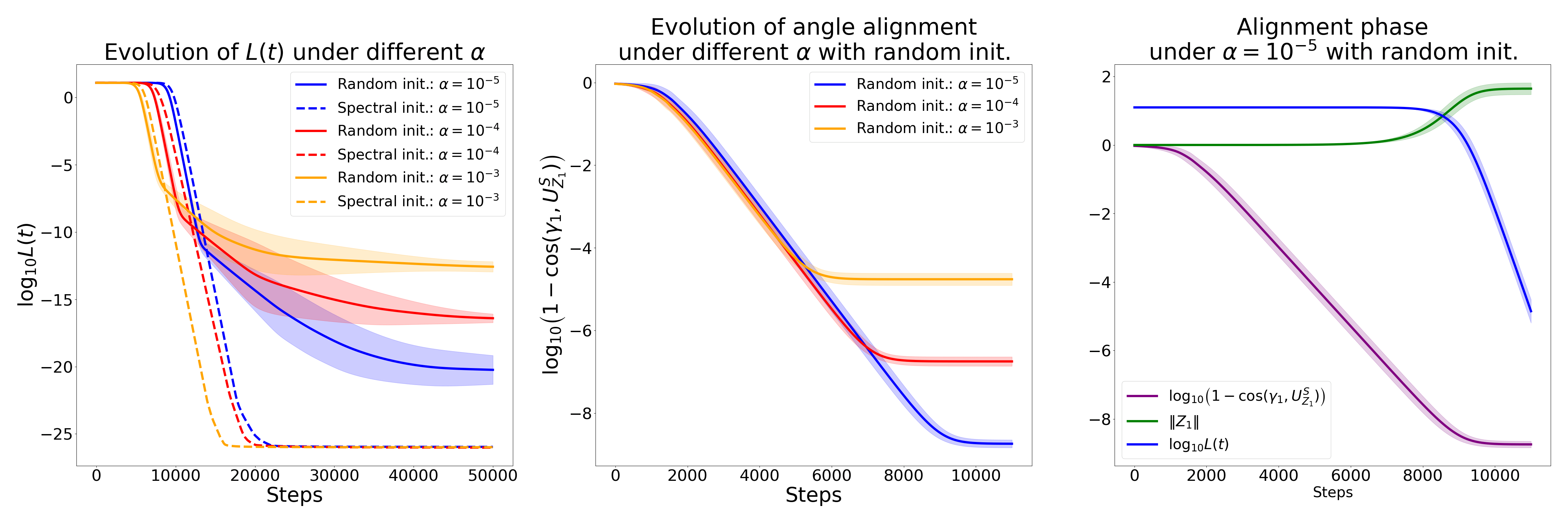}
    \caption{We simulate Problem~\ref{eqn:obj_lora} in the context of $\delta_w\!<\!1$ using both small initialization (see \S\ref{sec:prelim}) and small spectral initialization (see \S\ref{sec:rank-r}). Each simulation is repeated thirty times, with shaded regions representing one standard deviation above and below the mean (see \S\ref{subsec:simulation_mf} for details). The left panel shows the evolution of the loss for different initialization scales $\alpha$ with small and spectral initialization. The middle panel tracks the alignment quality between $\boldsymbol{U_{Z_1}^S}$ and $\boldsymbol{\gamma_1}$, measured by $\log_{10}(1 - \cos(\gamma_1, \boldsymbol{U_{Z_1}^S}(t)))$, where smaller values indicate better alignment. The right panel focuses on small initialization with $\alpha = 10^{-5}$, illustrating how the reconstruction loss, alignment between $\boldsymbol{U_{Z_1}^S}$ and $\boldsymbol{\gamma_1}$, and $\lVert Z_1 \rVert$ evolve during the \textit{alignment} phase.
    }
    \label{fig:dynamics_mf}
\end{figure*}

\subsection{Main Theorem}
In this section, we first introduce some notation that will be used later. Then, we present our main results in the context of $\mathrm{rank}(\Delta Y) = 1$. 

\myparagraph{Notation} Let $Z_1 \!=\! \begin{pmatrix} B_1 \\ A_1^\top \end{pmatrix}$ and $Z_2 \!=\! \begin{pmatrix} B_2 \\ A_2^\top \end{pmatrix}$ denote the concatenation of the LoRA weights. 
We use $\boldsymbol{U_{Z_1}^S}$ and $\boldsymbol{U_{Z_2}^S}$ to represent the top left singular vectors of $Z_1$ and $Z_2$, respectively. 
When $\mathrm{rank}(\Delta Y) \!=\! 1$, let $\Delta Y \!=\! \sigma \boldsymbol{u}\boldsymbol{v}^\top$ where $\boldsymbol{u}$ and $\boldsymbol{v}$ are a pair of singular vectors of $Y_{\mathrm{pre}}$ (Assumption~\ref{asmp:target_Y}). 
From Assumption~\ref{asmp:pre-trained_w}, we can infer that $u$ and $v$ are left and right singular vectors of $W_2$ and $W_1$, respectively. 
Let $(\sigma_{W_1}, g, v)$ and $(\sigma_{W_2}, u, g)$ be the pairs of singular values and their corresponding singular vectors for $W_1$ and $W_2$. Finally, we define $\delta_w \!=\! \frac{\sigma_{W_1}}{\sigma_{W_2}}$.

%
\begin{thm}\label{thm:rank-one}
Assume $\delta_w\!\not=\!1$; without loss of generality take $\delta_w<1$. Then, 
for any LoRA rank $r$, there exists constants $c_1, c_2\!=\!\mathrm{polylog}(\frac{1}{\lvert1\!-\!\delta_w\rvert}, \sigma_{W_2}, \sigma_{W_1}),$ and $ c_3,  \alpha^*\!=\!\mathrm{polylog}(\lvert1\!-\!\delta_w\rvert, \sigma_{W_2}, \sigma_{W_1})$ 
such that for any $0<\alpha\leq \alpha^*$, after time $T\!=\!\frac{c_1\log\bigl(\frac{1}{\alpha}\bigl)}{\sigma\sigma_{W_2}} \!+\! \frac{c_2\log\bigl(\frac{L(0)}{\alpha}\bigl)}{\sigma\sigma_{W_2}}$,
we have $L(T) \leq 2\alpha^{c_3}$. 
\end{thm}
The proof of Theorem can be found in Appendix~\ref{app:thm-rank-one}. 
In Theorem~\ref{thm:rank-one}, we assume without loss of generality that $\delta_w < 1$. This assumption introduces a dependency of the training time $T$ on $\sigma_{W_2}$.  
Symmetrically, if one assumes $\delta_w > 1$, the training time would instead be given by:  
\begin{align}
    T \!=\! \frac{c_1\log\bigl(\frac{1}{\alpha}\bigr)}{\sigma\sigma_{W_1}} \!+\! \frac{c_2\log\bigl(\frac{L(0)}{\alpha}\bigr)}{\sigma\sigma_{W_1}}\,.
\end{align}  
For consistency and ease of presentation, we will assume $\delta_w < 1$ throughout the paper.

Theorem~\ref{thm:rank-one} states that after training for time $T$, the training loss $L(T)$ decreases to a value depending on the initialization scale $\alpha$. Notice that $L(T)$ can be made arbitrarily small by selecting a sufficiently small $\alpha$, and the dependence of $T$ on $\alpha$ is logarithmic. 
Thus, reducing $\alpha$ leads to only a mild increase in the required training time to achieve a low training loss. 
Furthermore, in Figure~\ref{fig:dynamics_mf}, we numerically validate that, for different initialization scales $\alpha$, the final loss to which GF converges decreases as $\alpha$ decreases.

\myparagraph{Interpretation of Training Time $T$} Recall in \S\ref{sec:intro}, the learning dynamics can be decomposed into two phases: alignment and local convergence. As will become clear in \S\ref{subsec:alignment} and \S\ref{subsec:convergence}, the sum decomposition of the training time $T$ in Theorem~\ref{thm:rank-one} in fact derives from this two-phase structure:
%
\begin{align}\label{eqn:require_training_time}
    T = \underbrace{\frac{c_1\log\bigl(\frac{1}{\alpha}\bigl)}{\sigma\sigma_{W_2}}}_{\text{Phase I: alignment}} 
    + \underbrace{\frac{c_2\log\bigl(\frac{L(0)}{\alpha}\bigl)}{\sigma\sigma_{W_2}}}_{\text{Phase II: local convergence}}\,.
\end{align}
As one can observe, the training time required for each phase scales inversely with $\sigma \sigma_{W_2}$. In the subsequent sections (\S\ref{subsec:alignment} and \S\ref{subsec:convergence}), we demonstrate that this is because, during the \textit{alignment }phase, $\boldsymbol{U_{Z_1}^S}(t)$ aligns with $\boldsymbol{\gamma_1} \!=\! \begin{pmatrix} \boldsymbol{g} \\ \boldsymbol{v} \end{pmatrix}$, and the speed of this alignment is governed by $\sigma \sigma_{W_2}$.
Moreover, smaller initialization leads to longer \textit{alignment} phase and better quality of the alignment as illustrated in the middle panel in Figure~\ref{fig:dynamics_mf}. In the \textit{local convergence} phase, we demonstrate that Problem~\ref{eqn:obj_lora} satisfies a local Polyak-Łojasiewicz (PL) condition, with the PL constant being proportional to $\sigma \sigma_{W_2}$. Consequently, the training time for each phase is inversely proportional to $\sigma \sigma_{W_2}$.

In the following sections, we present a proof sketch for Theorem~\ref{thm:rank-one} by formally characterizing the \textit{alignment} phase and \textit{local convergence} phase. 

\subsection{Proof Strategy of Alignment Phase}\label{subsec:alignment}
In this section, we argue that during the initial phase of training, GF implicitly performs spectral initialization by aligning $\boldsymbol{U_{Z_1}^S}(t)$ with $\boldsymbol{\gamma_1}$. Simultaneously, $\lVert  Z_1\rVert$ continues to grow, and the difference between $\sigma_{W_1}$ and $\sigma_{W_2}$ leads to a progressive increase in the imbalance between $\lVert Z_1 \rVert$ and $\lVert Z_2 \rVert$. Moreover, the initialization scale affects the quality of alignment, with smaller initialization leading to better alignment.


.

We now present the main intuition behind the alignment phase. 
Our starting point is the following observation that the learning dynamics can be decomposed as follows
\begin{align}\label{eqn:early_stage_dynamics}
    \dot Z_1 &\!=\! \biggl(\sigma\sigma_{W_{2}}H_1\!+\!\begin{pmatrix}
        0&D_1\\D_1^\top&0
    \end{pmatrix}\biggl) Z_1\,,\nonumber\\
    \dot Z_2 \!&=\! \biggl(\sigma\sigma_{W_1}H_2\!+\!\begin{pmatrix}
        0&D_2\\D_2^\top&0
    \end{pmatrix}\biggl) Z_2 \,,
\end{align}

where $H_1, H_2, D_1,D_2$ are defined as follows
\begin{align}   
    &H_1 \!=\! \begin{pmatrix}
        0 & \boldsymbol{g}\boldsymbol{v}^\top \\ \boldsymbol{v}\boldsymbol{g}^\top & 0
    \end{pmatrix}, \quad H_2 \!=\!\begin{pmatrix}
        0 & \boldsymbol{u}\boldsymbol{g}^\top \\ \boldsymbol{g}\boldsymbol{u}^\top & 0
    \end{pmatrix}\nonumber\\
    &F \!=\! W_2B_1A_1\!+\!B_2A_2W_1\!+\!B_2A_2B_1A_1, \quad E\!=\!\Delta Y\!-\!F \nonumber\\
    &D_1\!=\!A_2^\top B_2^\top E\!-\!W_2^\top F, \quad D_2 \!=\! EA_1^\top B_1^\top\!-\! FW_1^\top \,.
\end{align}
%
At initialization, the magnitude of $\sigma\sigma_{W_2}H_1, \sigma\sigma_{W_1}H_2$ are significantly larger than those of $D_1, D_2$, since $\lVert D_1\rVert_F, \lVert D_2\rVert_F \!\sim\! \mathcal{O}(\alpha^2)$.  Therefore, when $\alpha$ is sufficiently small, the learning dynamics shown in \eqref{eqn:early_stage_dynamics} can be approximated as follows
\begin{align}\label{eqn:early_stage_dynamics_approximation}
    \dot Z_1 \approx \sigma\sigma_{W_2}H_1 Z_1, \quad 
    \dot Z_2 \approx \sigma\sigma_{W_1}H_2 Z_2 \,.
\end{align}
Our analysis of the alignment phase centers on interpreting \eqref{eqn:early_stage_dynamics} as a \emph{perturbation} of the approximate learning dynamics in \eqref{eqn:early_stage_dynamics_approximation}. To build intuition, we summarize the simplified dynamics \eqref{eqn:early_stage_dynamics_approximation} in the following claim, deferring the perturbation analysis to Appendix \ref{app:proof_alignment}.

\begin{claim}\label{lem:ode_alignment}
Consider the following ODE system: $\dot Z_1 \!=\! \sigma\sigma_{W_2}H_1 Z_1$, $\dot Z_2 \!=\! \sigma\sigma_{W_1}H_2 Z_2$. 
Taking $Z_1$ as an example, one can first show that the eigenvalue decomposition of $H_1$ is $H_1 \!=\! \boldsymbol{\gamma_1}\boldsymbol{\gamma_1}^\top \!-\! \boldsymbol{\bar\gamma_1}\boldsymbol{\bar\gamma_1}^\top$, where $\boldsymbol{\bar\gamma_1} \!=\! \begin{pmatrix} \boldsymbol{g} \\ - \boldsymbol{v} \end{pmatrix}$. Based on this, the following conditions hold:

\begin{enumerate}[leftmargin=0.45cm]
    \item Closed-form Solution for $Z_1$:
    \begin{align}\label{eqn:sol_early_stage_dynamics_ode}
        Z_1(t) =& \exp(\sigma\sigma_{W_2}t)\boldsymbol{\gamma_1}\boldsymbol{\gamma_1}^\top Z_1(0)\nonumber\\
        &+ \exp(-\sigma\sigma_{W_2}t)\boldsymbol{\bar\gamma_1}\boldsymbol{\bar\gamma_1}^\top Z_1(0)\,.
    \end{align}
    
    \item Growth of $\lVert \boldsymbol{\gamma_1}\boldsymbol{\gamma_1}^\top Z_1\rVert$: $\frac{d}{dt}\log\lVert \boldsymbol{\gamma_1}\boldsymbol{\gamma_1}^\top Z_1\rVert^2 \!=\! 2 \sigma\sigma_{W_2}\,.$
    
    \item Imbalance Between $\lVert \boldsymbol{\gamma_1}\boldsymbol{\gamma_1}^\top Z_1\rVert$ and $\lVert \boldsymbol{\gamma_2}\boldsymbol{\gamma_2}^\top Z_2\rVert$:
    \begin{align}\label{eqn:imbalance_growth_ode}
        \frac{d}{dt}\log\left(\frac{\lVert \boldsymbol{\gamma_1}\boldsymbol{\gamma_1}^\top Z_1(t)\rVert^2}{\lVert \boldsymbol{\gamma_2}\boldsymbol{\gamma_2}^\top Z_2(t)\rVert^2}\right) = 2 \sigma(\sigma_{W_2} \!-\! \sigma_{W_1})\,,
    \end{align}
    where $\boldsymbol{\gamma_2}\!=\!\begin{pmatrix} \boldsymbol{u} \\ \boldsymbol{g} \end{pmatrix}$.
\end{enumerate}
\end{claim}

Claim~\ref{lem:ode_alignment} demonstrates that the simplified dynamics of $Z_1$ and $Z_2$ lead to the alignment of $\boldsymbol{U_{Z_1}^S}$ with $\boldsymbol{\gamma_1}$, the growth of the spectrum of $Z_1$ projected onto the space spanned by $\boldsymbol{\gamma_1}\boldsymbol{\gamma_1}^\top$, and the increasing imbalance between $\lVert Z_1 \rVert$ and $\lVert Z_2 \rVert$. Moreover, as training progresses, $\boldsymbol{U_{Z_1}^S}$ will be sufficiently aligned with $\boldsymbol{\gamma_1}$, leading to $\lVert Z_1\rVert\!\approx\!\lVert \boldsymbol{\gamma_1}\boldsymbol{\gamma_1}^\top Z_1\rVert$ (respectively $Z_2$).


\myparagraph{LoRA Weights Escape Saddle Points} In \S\ref{subsec:challenge}, we discussed that at initialization, the LoRA weights are close to a saddle point where both $A_i$ and $B_i$, for $i \!= 1 \text{ or } 2$, are zero matrices. Claim~\ref{lem:ode_alignment} provides insights into how the LoRA weights gradually escape this saddle point during the \textit{alignment} phase, through the growth of $\lVert Z_1\rVert$. Consequently, GF will not be trapped in this saddle point, which ensures convergence. The right panel in Figure~\ref{fig:dynamics_mf} numerically demonstrates that $\lVert Z_1 \rVert$ moves away from zero at the end of the \textit{alignment} phase, following the actual dynamics of LoRA.

The simplified dynamics summarized in Claim~\ref{lem:ode_alignment} accurately approximate the true LoRA dynamics only up to a finite time $T_1$, due to the growth of perturbation terms $D_1$ and $D_2$ (see again Appendix~\ref{app:proof_alignment}). The time $T_1$ is the end of the alignment phase, at which point the alignment and imbalance of the true dynamics are characterized by the following claim.

\begin{claim}\label{cor:end_of_alignment}
Under the same setting in Theorem~\ref{thm:rank-one}, there exists constants $c_4, c_5$ such that at $T_1$, one has
\begin{enumerate}[leftmargin=0.45cm]
    \item Good Alignment: $\cos^2{(\boldsymbol{U_{Z_1}^S}(T_1), \boldsymbol{\gamma_1})} \!\geq\! 1\!-\!c_4 \alpha^{\frac{1+\delta_w}{3+\delta_w}}$.
    \item Sufficient Imbalance: $\frac{\lVert Z_1(T_1)\rVert_2^2}{\lVert Z_2(T_1)\rVert_2^2} \!\geq\! c_{5}\alpha^{-\frac{4(1-\delta_w)}{(5-\delta_w)}}\,.$
\end{enumerate}
\end{claim}
The Proof of Claim~\ref{cor:end_of_alignment} can be found in Appendix~\ref{app:proof_alignment}.

\subsection{Proof Strategy of Local Convergence Phase}\label{subsec:convergence}
At the end of the \textit{alignment} phase, we established sufficient alignment between $\boldsymbol{U_{Z_1}^S}(T_1)$ and $\gamma_1$, as well as a significant imbalance between the norms of $\lVert Z_1(T_1) \rVert$ and $\lVert Z_2(T_1) \rVert$ (see Claim~\ref{cor:end_of_alignment}). However, it remains unclear whether these properties persist into the \textit{local convergence} phase, or how the quality of alignment and imbalance affect the convergence results.

In this section, we first demonstrate how the large imbalance between $\lVert Z_1 \rVert$ and $\lVert Z_2 \rVert$ leads to simplified learning dynamics for gradient flow (GF). Specifically, with $\lVert Z_1 \rVert \gg \lVert Z_2 \rVert$, the objective function can be approximated as:
\begin{align}\label{eqn:obj_approx}
L(t) &= \frac{1}{2} \lVert \Delta Y - W_2 A_1 B_1 - A_2 B_2 W_1 - A_2 B_2 A_1 B_1 \rVert_F^2 \nonumber \\
&\approx \frac{1}{2} \lVert \Delta Y - W_2 A_1 B_1 \rVert_F^2.
\end{align}
This approximation reduces the objective to a form similar to MF, where the factors are $A_1$ and $B_1$. Consequently, we can leverage the techniques developed in~\cite{stoger2021small,jin2023understanding, pmlr-v195-soltanolkotabi23a} to derive the linear convergence of GF in this setting.

The following theorem formally characterizes the persistence of the alignment and the imbalance throughout the \textit{local convergence} phase, and establishes the linear convergence rate of GF.
\begin{thm}[Local convergence]\label{thm:local_convergence_rank-one}
Under the same setting as Theorem~\ref{thm:rank-one}, let $T_2\!=\!\frac{c_2\log\bigl(\frac{\sqrt{2L(0)}}{\alpha}\bigl)}{\sigma\sigma_{W_2}}$. Then, there exists constants $c_6, c_7$ such that for $\forall t\in[T_1,T_1\!+\!T_2]$, the following holds
\begin{enumerate}[leftmargin=0.45cm]
    \item Good Alignment of $\boldsymbol{U_{Z_1}^S}(t)$ with $\boldsymbol{\gamma_1}$:
    \begin{align}\label{eqn:alignment_lower_bdd}
        \cos^2{(\boldsymbol{U_{Z_1}^S}(t), \boldsymbol{\gamma_1})} \!\geq \!\cos^{2c_6}{(\boldsymbol{U_{Z_1}^S}(T_1), \boldsymbol{\gamma_1})}\,.
    \end{align}
    \item Imbalance Between $\lVert Z_1(t)\rVert$ and $\lVert Z_2(t)\rVert$ Persists:
    \begin{align}\label{eqn:imbalance_lower_bdd}
        \frac{\lVert Z_1(t)\rVert}{\lVert Z_2(t)\rVert} \!\geq\!\biggl( \frac{\lVert Z_1(T_1)\rVert}{\lVert Z_2(T_1)\rVert}\biggl)^{c_7}\,.
    \end{align}
    \item Loss Converges Linearly:
    \begin{align}\label{eqn:convergence_dependency}
        &L(t)\nonumber\\
        \leq&2\exp\biggl(-\frac{\bigl(1\!-\!\frac{\lVert Z_2(T_1)\rVert}{\lVert Z_1(T_2)\rVert}\bigr)(1\!-\!\delta_w)\sigma\sigma_{W_2}(t\!-\!T_1)}{16}\biggl)L(T_1)\nonumber\\
        &\!+\!c_8\bigl(1\!-\!\cos^2{(\boldsymbol{U_{Z_1}^S}(T_1), \boldsymbol{\gamma_1})}\bigl)\,.
    \end{align}
\end{enumerate}
Moreover, by substituting the alignment and imbalance results from Claim~\ref{cor:end_of_alignment} and assuming $\alpha\!\leq\!\alpha^*$, we can simplify convergence rate in~\eqref{eqn:convergence_dependency} as follows:
\begin{align}\label{eqn:convergence_simplified}
    L(t)\!\leq\! 2\exp\biggl(\frac{\!-\!(1\!-\!\delta_w)\sigma\sigma_{W_2}(t\!-\!T_1)}{32}\biggl)L(T_1) \!+\! 2\alpha^{c_3}\,.
\end{align}
\end{thm}
The proof of Theorem~\ref{thm:local_convergence_rank-one} can be found in Appendix~\ref{app:proof_local_convergence}.
Theorem~\ref{thm:local_convergence_rank-one} establishes that the alignment and imbalance achieved during the \textit{alignment} phase persist throughout the \textit{local convergence} phase, and shows how these properties influence the rate of convergence. Specifically, \eqref{eqn:convergence_dependency} demonstrates that the upper bound on $L(t)$ consists of two components: the first part converges to zero at a linear rate, while the second part is a constant that depends on the quality of the alignment between $\boldsymbol{U_{Z_1}^S}(T_1)$ and $\boldsymbol{\gamma_1}$ at the end of the \textit{alignment} phase, with better alignment leading to smaller final loss. By applying the result from Claim~\ref{cor:end_of_alignment}, we conclude that the rate of convergence is determined by the pre-trained weights and the difference between the target matrices of the pre-trained and fine-tuning tasks, i.e., $\sigma \sigma_{W_2}$. Moreover, this rate is independent of the initialization scale $\alpha$. 

In \S\ref{sec:rank-one}, we analyzed GF with small initialization for $\mathrm{rank}(\Delta Y) \!=\! 1$, showing that smaller initialization improves alignment between $\boldsymbol{U_{Z_1}^S}$ and $\boldsymbol{\gamma_1}$, leading to a lower final loss. This raises the question: if $\boldsymbol{U_{Z_1}^S}$ is perfectly aligned with $\boldsymbol{\gamma_1}$ at initialization, can GF achieve zero loss? The following section confirms this by carefully designing spectral initialization for LoRA in MF.

\section{LEARNING DYNAMICS UNDER SPECTRAL INITIALIZATION}\label{sec:rank-r}
In this section, we introduce the spectral initialization for LoRA applied to Problem~\ref{eqn:obj_lora} when 
LoRA rank $r \!\geq\! \mathrm{rank}(\Delta Y)\!>\! 1$. Then, we present the main theorem on its learning dynamics, and 
compare them to the rank-one case analyzed in \S\ref{sec:rank-one}.

\myparagraph{Spectral Initialization} We initialize $B_1, B_2$ as zero matrices. 
For $A_1, A_2$, we initialize them using the singular vector matrices of $\Delta Y, W_2$. 
In Assumptions~\ref{asmp:target_Y} and~\ref{asmp:pre-trained_w}, we assume the full SVD of $\Delta Y, W_2$ are $U_Y\Sigma_{\Delta Y}V_Y^\top, U_Y\Sigma_{W_2}G^\top$ respectively. 
Let $U_Y^S, V_Y^S$ be the singular matrices of $\Delta Y$ corresponding to the non-zero singular values. 
Then, we define $G^S$ to be the left singular vectors of $W_2$ corresponding to the subcolumns $U_Y^S$ of $U_Y$.
Based on the definition above, we initialize $A_1, A_2$ as follows: $A_1\!=\!G_{11}\Sigma_{A_1}G_{12}^\top, A_2\!=\!G_{21}\Sigma_{A_2}G_{22}^\top$,
where $\Sigma_{A_1}, \Sigma_{A_2}$ are diagonal matrices and initialized entry-wise i.i.d. as $\Sigma_{A_1}[i,i],\Sigma_{A_2}[i,i]\!\sim\!\mathcal{N}(0, \alpha^2)$, and $G_{11},G_{21}\in\R^{r\times r}$ are arbitrary orthogonal matrices. 
The matrices $G_{12}\!=\![V_Y^S, V_{Y,\perp}^S], G_{22}\!=\![G^S, G^S_{\perp}]$ are constructed from the orthogonal matrices $V_{Y,\perp}^S\in\R^{r\times(n-r)}, G^S_{\perp}\in\R^{r\times(h-r)}$ where  $V_Y^S\!\perp\! V_{Y,\perp}^S, G^S\!\perp\! G^S_{\perp}$.

\begin{rem}
\cite{balazy2024lora, lin2024nora, meng2024pissa, wang2024milora} propose spectral initialization for LoRA based solely on the pre-trained weights. In contrast, our spectral initialization depends on both the pre-trained weights and the fine-tuning target matrix. In Appendix~\ref{app:example}, we provide examples where methods built purely on pre-trained weights fail to fine-tune pre-trained models for MF, highlighting the importance of incorporating the fine-tuning target matrix when designing spectral initialization for LoRA.
\end{rem}
\begin{figure*}[!t]
    \centering
    \includegraphics[height=0.31\textwidth]{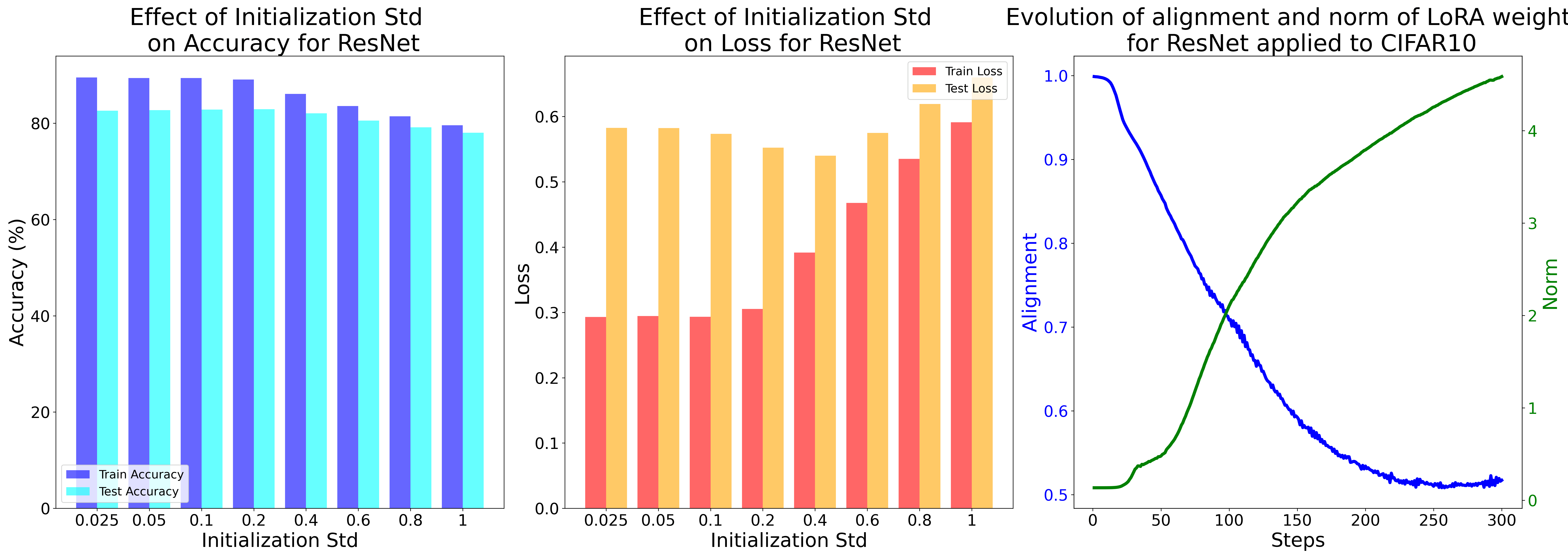}
    \caption{The left and middle panels report the loss and accuracy evaluated on the training and evaluation datasets for ResNet on the CIFAR-10 dataset. The right panel shows the evolution of the alignment between the singular matrices of the LoRA weights and the target directions (with smaller values indicating better alignment), as well as the norm of the LoRA weights. We repeat the simulation thirty times, with the shaded regions representing one standard deviation above and below the mean.
    }
    \label{fig:image-classification}
\end{figure*}

\myparagraph{Learning Dynamics of LoRA Under Spectral Initialization} Under spectral initialization, the learning dynamics of GF can be decoupled into several scalar dynamics. 
Specifically, let $\bar B_1\!=\!G^\top B_1G_{11}^\top, \bar A_1 \!=\! G_{11}^\top A_1G_{12}, \bar B_2\!=\! U_Y^\top B_2G_{21}^\top, \bar A_2\!=\!G_{21}^\top A_2 G_{22}$, then one can write the learning dynamics of $\bar A_1, \bar B_1$ as follows
\begin{align}\label{eqn:ode_diagonal}
\begin{split}
    \frac{d}{dt}\bar B_1=(\Sigma_{W_2}\!+\!\bar B_2\bar A_2)(\Sigma_{\Delta Y}\!-\!\bar F)\bar A_1^\top\,, \nonumber\\
    \frac{d}{dt}\bar A_1=\bar B_1^\top(\Sigma_{W_2}\!+\!\bar B_2\bar A_2)^\top(\Sigma_{\Delta Y}\!-\!\bar F) \,.
\end{split}
\end{align}
where $\bar F\!:=\!\Sigma_{W_2}\bar B_1\bar A_1\!+\!\bar B_2\bar A_2\Sigma_{W_1}\!+\!\bar B_2\bar A_2\bar B_1\bar A_1$. 
At initialization, $\bar B_i, \bar A_i, i\!=\!1,2$ are diagonal matrices due to spectral initialization, then they remain diagonal during the training based on~\eqref{eqn:ode_diagonal}. 
Thus, the learning dynamics of LoRA weights decouple into the learning dynamics of the singular values of LoRA weights. 
For ease of presentation in the following theorem, we will use $\sigma_{A_1}^{(i)}$ to represent the ith diagonal element of $\bar A_1$ (similar definition for $\sigma_{A_2}^{(i)},\sigma_{B_1}^{(i)},\sigma_{B_2}^{(i)}, \sigma_{\Delta Y}^{(i)},\sigma_{W_2}^{(i)},\sigma_{W_1}^{(i)}, \sigma_{\bar F}^{(i)}$).
\begin{thm}\label{thm:rank-r}
Let $\delta_{w}^{(i)}\!=\!\frac{\sigma_{W_2}^{(i)}}{\sigma_{W_1}^{(i)}}, \ell^{(i)}(t)\!=\!\frac{1}{2}\bigl(\sigma_{\Delta Y}^{(i)}\!-\!\sigma_{\bar F}^{(i)}\bigl)^2, z_1^{(i)}\!=\!(\sigma_{A_1}^{(i)})^2\!+\!(\sigma_{B_1}^{(i)})^2$ (respectively $z_2^{(i)}$). In the case where $\delta_w^{(i)}\!\not=\!1$, we assume $\delta_w^{(i)}\!<\!1$ WLOG, then the learning dynamics has two phases which can be separated by $T_1^{(i)}\!=\!\frac{2}{(3\!+\!\delta_w^{(i)})\sigma_{\Delta Y}^{(i)}\sigma_{W_2}^{(i)}}\log\bigl(\frac{\sigma_{\Delta Y}^{(i)}}{16z_1^{(i)}(0)}\bigl)$
\begin{enumerate}[leftmargin=0.45cm]
    \item \textit{Growth of Norm and Imbalance}: $\forall t\!\leq\! T_1^{(i)}$
    \begin{align}
        \frac{d}{dt} \log z_1^{(i)}\!&\geq\!\frac{3\!+\!\delta_w^{(i)}}{2}\sigma_{\Delta Y}^{(i)}\sigma_{W_2}^{(i)}\,,\nonumber\\
        \frac{d}{dt}\log\biggl(\frac{z_1^{(i)}}{z_2^{(i)}}\biggl)\!&\geq\!\frac{3(1\!-\!\delta_w^{(i)})}{2}\sigma_{\Delta Y}^{(i)}\sigma_{W_2}^{(i)}\,.
    \end{align}
    \item \textit{Local Convergence}: for $\forall t\!\geq\! T_1^{(i)}$, the loss converges linearly
    \begin{align}
        \ell^{(i)}(t)\!\leq\!\exp\biggl(-\frac{(1\!-\!\delta_w^{(i)})\sigma_{\Delta Y}^{(i)} \sigma_{W_2}^{(i)}(t\!-\!T_1)}{8}\biggl) \ell^{(i)}(T_1)\,.\nonumber
    \end{align}
\end{enumerate}
\end{thm}
The proof of Theorem~\ref{thm:rank-r} can be found in Appendix~\ref{app:proof_spectral}.

\myparagraph{Comparison with Small Initialization}
Theorem~\ref{thm:rank-r} shows that the dynamics of each singular value go through two distinct phases: a \textit{growth of norm and imbalance} phase, where the norm of the singular values increases from zero, and one pair of singular values of the LoRA weights becomes dominant depending on the singular values of the pre-trained weights along the same direction, i.e., $\delta_w^{(i)}$. A \textit{local convergence} phase, where the singular values of the model converge to the singular values of the target matrix, with the loss decreasing at a linear rate. The phenomena described above can also be observed in Theorem~\ref{thm:rank-one}. Due to the spectral initialization of the LoRA weights, the model is perfectly aligned with the singular spaces of $\Delta Y$ from the beginning. Thus, unlike Theorem~\ref{thm:rank-one}, Theorem~\ref{thm:rank-r} does not require an \textit{alignment} phase. A key difference of Theorem~\ref{thm:rank-r} is that the loss provably converges to arbitrary precision, whereas in Theorem~\ref{thm:rank-one}, the final loss is constrained by the initialization scale (see Figure~\ref{fig:dynamics_mf} for numerical validation).

\myparagraph{Comparison with Incremental Learning in Low-rank MF}~\cite{jin2023understanding, pesme2023saddle} theoretically establish the incremental learning phenomenon in MF/matrix sensing problems. They show that GD/GF with small initialization gradually learns solutions with increasing ranks, with the order of learned singular values determined by their magnitudes. 
In Theorem~\ref{thm:rank-r}, we characterize the duration of the \textit{growth of norm and imbalance} phase and the convergence rate in the \textit{local convergence} phase, both determined by the product of the singular values of $\Delta Y$ and $W_2$ (or $W_1$, depending on $\delta_w^{(i)}$). Specifically, larger values of $\sigma_{\Delta Y}^{(i)}\sigma_{W_2}^{(i)}$ (or $\sigma_{\Delta Y}^{(i)}\sigma_{W_1}^{(i)}$, depending on $\delta_w^{(i)}$) lead to a shorter \textit{growth of norm and imbalance} phase and faster convergence in the \textit{local convergence} phase, indicating that the singular values of the target matrix along this direction are learned faster.

\section{EXPERIMENTAL RESULTS}
In this section, we run experiments on LoRA applied to MF and image classification problems to validate the theoretical findings in \S\ref{sec:rank-one}. 
Specifically, we are interested in the following questions: First, does decreasing the initialization scale lead to a smaller final training error? 
Second, does the alignment phenomenon occur in the early stages of training? 
In both experiments, we provide affirmative answers to these questions.

\subsection{Experiments for MF}\label{subsec:simulation_mf}
In this section, we solve Problem~\ref{eqn:obj_lora} with $\mathrm{rank}(\Delta Y) \!=\! 1$, LoRA rank $r \!=\! 4$, and the factor matrix sizes are $W_2 \in \R^{10 \times 100}$ and $W_1 \in \R^{100 \times 10}$. 
We optimize the objective using GD with a small step size of $10^{-4}$. 
We generate the pre-trained and fine-tuning target matrices as follows: $Y_{\mathrm{pre}} \in \R^{10 \times 10}$, where $Y_{\mathrm{pre}}[i,j] \sim \mathcal{N}(0, 1)$, and $Y_{\mathrm{ft}} \!=\! Y_{\mathrm{pre}} \!+\! 5u_1v_1^\top$, where $u_1$ and $v_1$ are the top singular vectors of $Y_{\mathrm{pre}}$. 
For the pre-trained weights, we generate them as follows: $W_2 \!=\! 1.05 \times U_Y \Sigma_{\text{pre}}^{1/2} G^\top$, and $W_1 \!=\! \frac{1}{1.05} G \Sigma_{\text{pre}}^{1/2} V_Y^\top$, where $G \in \R^{100 \times 100}$ is an orthogonal matrix. 
We initialize the LoRA weights using small initialization and spectral initialization, with different initialization scales $\alpha \in \{10^{-5}, 10^{-4}, 10^{-3}\}$.

Figure~\ref{fig:dynamics_mf} numerically validates our theoretical findings from Theorem~\ref{thm:rank-one}, Theorem~\ref{thm:rank-r} and Claim~\ref{cor:end_of_alignment}. 
The right panel shows that when the initialization scale is small, the final loss converges to a value that depends on initialization scale, with smaller scales leading to lower final loss, as predicted by Theorem~\ref{thm:rank-one}. 
Moreover, when using spectral initialization, GF converges to machine-precision loss regardless of the initialization scale supporting Theorem~\ref{thm:rank-r}.
The middle panel illustrates that smaller initialization results in a longer \textit{alignment} phase and better alignment, consistent with Claim~\ref{cor:end_of_alignment}. Finally, the left panel demonstrates that during the \textit{alignment} phase, $\boldsymbol{U_{Z_1}^S}$ first aligns with $\gamma_1$, followed by an increase in the norm of the LoRA weights. By the end of the \textit{alignment} phase, sufficient alignment is achieved, and the LoRA weights move away from saddle points, where $A_i$ and $B_i$ are zero matrices.

\vspace{-2mm}
\subsection{Experiments for Image Classification}
\vspace{-1mm}
In this section, we fine-tune \texttt{ResNet-50}~\citep{He_2016_CVPR}, pre-trained on ImageNet~\citep{imagenet}, for CIFAR-10 classification by adjusting the final fully-connected layer to $10$ classes. LoRA is applied to this layer with initialization standard deviations ranging from $1.0$ to $0.025$, and training is conducted using the Adam optimizer with a learning rate of 0.01. We measure the effect of initialization scales on LoRA weights by evaluating training and validation losses and accuracies, as shown in Figure~\ref{fig:image-classification}. Alignment measures how well LoRA's singular vectors align with target directions, with smaller values indicating better alignment. 
Appendix~\ref{app:simulation} details how target directions are determined, and includes more experiments on VGG~\citep{simonyan2014very} and \texttt{ViT-base-patch16}~\citep{dosovitskiy2021vit}.

The left and middle panel in Figure~\ref{fig:image-classification} demonstrate that smaller initialization leads to lower training and evaluation losses for ResNet. 
The left panel shows that in the first hundred steps of training, the singular vectors of LoRA align with the target directions, and the norm grows. 
The simulation results for ResNet match our theoretical findings for Problem~\ref{eqn:obj_lora}. 
Moreover, although our theory focuses on training error of Problem~\ref{eqn:obj_lora}, the improved test performance as the initialization scale decreases suggests that the initialization scale may also affect generalization performance.

\vspace{-1.2mm}
\section{Conclusion}
\vspace{-1.2mm}
This paper studied the learning dynamics of LoRA for MF under GF, highlighting the critical role of initialization. We theoretically derive convergence results for LoRA with both small initialization and spectral initialization. Our analysis reveals different behaviors under these two initializations: GF with small initialization converges to a neighborhood around the target matrix, with smaller initialization scales leading to more accurate convergence, while GF with spectral initialization converges to the target matrix with arbitrary precision. Numerical results from MF and computer vision validate our theoretical findings.

\subsubsection*{Acknowledgements}
The authors acknowledge the support of 
the Office of Naval Research (grant 503405-78051),
the National Science Foundation (grants 2031985, 2330450
),
and the Simons Foundation (grant 814201).

\bibliography{ref.bib}
\bibliographystyle{apalike}
\section*{Checklist}

 \begin{enumerate}

 \item For all models and algorithms presented, check if you include:
 \begin{enumerate}
   \item A clear description of the mathematical setting, assumptions, algorithm, and/or model. [\textbf{Yes}/No/Not Applicable]
   \item An analysis of the properties and complexity (time, space, sample size) of any algorithm. [\textbf{Yes}/No/Not Applicable]
   \item (Optional) Anonymized source code, with specification of all dependencies, including external libraries. [Yes/No/Not Applicable]
 \end{enumerate}

 \item For any theoretical claim, check if you include:
 \begin{enumerate}
   \item Statements of the full set of assumptions of all theoretical results. [\textbf{Yes}/No/Not Applicable]
   \item Complete proofs of all theoretical results. [\textbf{Yes}/No/Not Applicable]
   \item Clear explanations of any assumptions. [\textbf{Yes}/No/Not Applicable]     
 \end{enumerate}

 \item For all figures and tables that present empirical results, check if you include:
 \begin{enumerate}
   \item The code, data, and instructions needed to reproduce the main experimental results (either in the supplemental material or as a URL). [\textbf{Yes}/No/Not Applicable]
   \item All the training details (e.g., data splits, hyperparameters, how they were chosen). [\textbf{Yes}/No/Not Applicable]
         \item A clear definition of the specific measure or statistics and error bars (e.g., with respect to the random seed after running experiments multiple times). [\textbf{Yes}/No/Not Applicable]
         \item A description of the computing infrastructure used. (e.g., type of GPUs, internal cluster, or cloud provider). [\textbf{Yes}/No/Not Applicable]
 \end{enumerate}

 \item If you are using existing assets (e.g., code, data, models) or curating/releasing new assets, check if you include:
 \begin{enumerate}
   \item Citations of the creator If your work uses existing assets. [\textbf{Yes}/No/Not Applicable]
   \item The license information of the assets, if applicable. [\textbf{Yes}/No/Not Applicable]
   \item New assets either in the supplemental material or as a URL, if applicable. [\textbf{Yes}/No/Not Applicable]
   \item Information about consent from data providers/curators. [Yes/No/\textbf{Not Applicable}]
   \item Discussion of sensible content if applicable, e.g., personally identifiable information or offensive content. [Yes/No/\textbf{Not Applicable}]
 \end{enumerate}

 \item If you used crowdsourcing or conducted research with human subjects, check if you include:
 \begin{enumerate}
   \item The full text of instructions given to participants and screenshots. [Yes/No/\textbf{Not Applicable}]
   \item Descriptions of potential participant risks, with links to Institutional Review Board (IRB) approvals if applicable. [Yes/No/\textbf{Not Applicable}]
   \item The estimated hourly wage paid to participants and the total amount spent on participant compensation. [Yes/No/\textbf{Not Applicable}]
 \end{enumerate}

 \end{enumerate}

\appendix
\onecolumn
\aistatstitle{Supplementary Materials}
\section{Preliminary Lemma}
In this section, we present several preliminary lemma which will be used in the following sections.

\begin{lem}\label{applem:base}
For matrix $A, B$, we have
\begin{align}
    &\sigma_{\min }^2(A)\|B\|_{F}^{2} \leq\|A B\|_{F}^{2} \leq \sigma_{\max }^2(A)\|B\|_{F}^{2}\nonumber \\
&\sigma_{\min }^2(B)\|A\|_{F}^{2} \leq\|A B\|_{F}^{2} \leq \sigma_{\max}^2(B)\|A\|_{F}^{2}.
\end{align}
\end{lem}
\begin{proof}
\begin{align}
    \|A B\|_{F}^{2}&=\operatorname{tr}\left(A B B^{\top} A^{\top}\right)\nonumber \\
&=\operatorname{tr}\left(A^{\top} A B B^{\top}\right)\qquad\text{use cyclic property of trace}\nonumber \\
& \leq \lambda_{\max} \left(A^{\top} A\right)\|B\|_{F}^{2}\qquad\text{use trace inequality}\nonumber \\
&=\sigma_{\max}^2(A)\|B\|_F^2.
\end{align}
For the other way
\begin{align}
    \|A B\|_{F}^{2} &=\operatorname{tr}\left(A B B^{\top} A^{\top}\right)\nonumber \\
&=\operatorname{tr}\left(A^{\top} A B B^{\top}\right)\nonumber \\
& \leq \lambda_{\max} \left(B B^{\top}\right)\|A\|_{F}^{2}\nonumber\\
&=\sigma_{\max}^2(B)\|A\|_F^2.
\end{align}
The lower bound is similar.
\end{proof}

\begin{lem}\label{applem:2-fro}
For matrix $A\in\R^{m\times r}, B^{r\times n}$, we have
\begin{align}
    \lVert AB\rVert \leq \lVert A\rVert\cdot \lVert B\rVert \leq \frac{1}{2}\bigl( \lVert A\rVert^2+\lVert B\rVert^2\bigl) \leq \frac{1}{2}\bigl( \lVert A\rVert_F^2+\lVert A\rVert_F^2\bigl)\,.
\end{align}
\end{lem}
The proof of Lemma~\ref{applem:2-fro} follows the basic property of norm and Cauchy-Swartz inequality. 

\begin{lem}[Singular Space Perturbation Bound]\label{applem:wedin-sin}
Let $M^*$ and $M=M^*+E$ be two matrices in $\R^{m\times n}$ (WLOG, we assume $m\leq n$) whose SVDs are given respectively by
\begin{align}
    M^*&=\begin{pmatrix}
        U^* & U^*_{\perp}
    \end{pmatrix}\begin{pmatrix}
        \Sigma^* & 0&0\\
        0&\Sigma^*_{\perp}&0
    \end{pmatrix}\begin{pmatrix}
        (V^*)^\top \\
        (V^*_{\perp})^\top
    \end{pmatrix}\,,\\
    M&=\begin{pmatrix}
        U & U_{\perp}
    \end{pmatrix}\begin{pmatrix}
        \Sigma & 0&0\\
        0&\Sigma_{\perp}&0
    \end{pmatrix}\begin{pmatrix}
        (V)^\top \\
        (V_{\perp})^\top
    \end{pmatrix}\,.
\end{align}
Here, $\sigma_1\geq \cdots\geq\sigma_{m}$ (respectively $\sigma_1^*\geq \cdots\geq\sigma_{m}^*$), and $\Sigma, \Sigma^*\in\R^{r\times r}$. If $\lVert E\rVert<\sigma_r^*-\sigma_{r+1}^*$, then one has
\begin{align}
    \lVert UU^\top-U^*(U^*)^\top\rVert\leq\frac{\sqrt{2}\max\lVert E\rVert}{\sigma_r^*-\sigma_{r+1}^*-\lVert E\rVert}\,.
\end{align}
\end{lem}
We refer the readers to the proof of this lemma to Theorem2.9 in~\cite{Chen_2021}.

\begin{lem}[Theorem 7.3.3 from~\cite{horn2012matrix}]\label{applem:hermitian_A}
Let $A\in\R^{m\times n}$, let $q=\min(m,n)$, let $\sigma_1\geq\sigma_2\geq\cdots\geq\sigma_q$ be the ordered singular values of $A$, and define the Hemitian matrix
\begin{align}
    \mathcal{A}=\begin{pmatrix}
        0 & A\\A^* & 0
    \end{pmatrix}\,.
\end{align}
WLOG, assume $n\geq m$. Then, let the SVD of $A$ is $A=V\Sigma W^*, \Sigma=[\Sigma_m, 0], V=[V_1, V_2], \hat V_1=\frac{1}{\sqrt{2}}V_1, \hat W=\frac{1}{\sqrt{2}}W$, and $U=\begin{pmatrix}
    \hat V&-\hat V & V_2\\ \hat W&\hat W&0
\end{pmatrix}$. Then, the eigenvalue decomposition of $\mathcal{A}$ is
\begin{align}
    \mathcal{A}=U\begin{pmatrix}
        \Sigma_m & 0 & 0\\
        0&-\Sigma_m&0\\
        0&0&0
    \end{pmatrix}U^*\,.
\end{align}
\end{lem}
We refer the readers to ~\cite{horn2012matrix} for detailed proof.

\begin{lem}[Learning dynamics of LoRA]
When applying GF in~\eqref{eqn:gf} to Problem~\ref{eqn:obj_lora}, the learning dynamics of LoRA weights satisfy the following ODEs
\begin{align}
    \frac{d}{dt}\begin{pmatrix}
        B_1\\A_1^\top
    \end{pmatrix}&=\begin{pmatrix}
        0 & (W_2+B_2A_2)^\top E\\ E^\top (W_2+B_2A_2) & 0
    \end{pmatrix} \begin{pmatrix}
        B_1\\A_1^\top
    \end{pmatrix}\,,\nonumber\\
    \frac{d}{dt}\begin{pmatrix}
        B_2\\A_2^\top
    \end{pmatrix}&=\begin{pmatrix}
        0 &  E(W_1+B_1A_1)^\top\\ (W_1+B_1A_1)E^\top & 0
    \end{pmatrix} \begin{pmatrix}
        B_1\\A_1^\top
    \end{pmatrix}
\end{align}
\end{lem}
\begin{proof}
We take $A_1, B_1$ as an example. Direct calculations based on~\eqref{eqn:gf} yields the following
\begin{align}
    \dot A_1 = B_1^\top (W_2+B_2A_2)^\top E , \quad \dot B_1 =  (W_2+B_2A_2)^\top E A_1^\top.
\end{align}
Then, one can combine the above equations to derive the result.
\end{proof}

\begin{lem}\label{applem:column_fro}
Given matrix $A\in\R^{m\times n}$, and use $a_i$ to denotes its i-th column. Let $\gamma\in\R^{m}$ be any unit vector. If $\cos\biggl(\gamma, \frac{a_i}{\lVert ai\rVert}\biggl)\geq q$ holds for all $i\in[r]$,
then the following holds
\begin{align}
    \gamma^\top AA^\top \gamma \geq q^2 \lVert A\rVert_F^2\,.
\end{align}
\end{lem}
\begin{proof}
\begin{align}
    \gamma^\top AA^\top \gamma \geq &= \gamma^\top \sum_{i=1}^r a_ia_i^\top \gamma \nonumber\\
    &= \sum_{i=1}^r \lVert a_i\rVert^2 \cdot \cos^2\biggl(\gamma, \frac{a_i}{\lVert a_i\rVert}\biggl)\nonumber\\
    &\geq q^2 \sum_{i=1}^r \lVert a_i\rVert^2\nonumber\\
    &= q^2 \lVert A\rVert_F^2\,.
\end{align}
\end{proof}
\section{Proof of Theorem~\ref{thm:rank-one}}\label{app:thm-rank-one}
Our proof strategy for Theorem~\ref{thm:rank-one} involves decomposing the learning dynamics into two distinct phases: the \textit{alignment} phase and the \textit{local convergence} phase. Theorem~\ref{thm:rank-one} builds on Claim~\ref{cor:end_of_alignment} and Theorem~\ref{thm:local_convergence_rank-one}, with Theorem~\ref{thm:local_convergence_rank-one} itself relying on Claim~\ref{cor:end_of_alignment}. In this section, we assume Claim~\ref{cor:end_of_alignment} and Theorem~\ref{thm:local_convergence_rank-one} hold, and prove Theorem~\ref{thm:rank-one} based on this. Detailed proofs for Claim~\ref{cor:end_of_alignment} and Theorem~\ref{thm:local_convergence_rank-one} are provided separately in Appendix~\ref{app:proof_alignment} and Appendix~\ref{app:proof_local_convergence}. We state a detailed version of Theorem~\ref{thm:rank-one} as follows.
\begin{thm}\label{appthm:rank-one}
Assuming $\delta_w\!\not=\!1$, 
for any LoRA rank $r$, there exists constants $c_1, c_2\!=\!\mathrm{polylog}(\frac{1}{\lvert1\!-\!\delta_w\rvert}, \sigma_{W_2}, \sigma_{W_1}),$ and $ c_3,  \alpha^*\!=\!\mathrm{polylog}(\lvert1\!-\!\delta_w\rvert, \sigma_{W_2}, \sigma_{W_1})$ 
such that for any $0<\alpha\leq \alpha^*$, after time $T\!=\!\frac{c_1\log\bigl(\frac{1}{\alpha}\bigl)}{\sigma\sigma_{W_2}} \!+\! \frac{c_2\log\bigl(\frac{L(0)}{\alpha}\bigl)}{\sigma\sigma_{W_2}}$,
we have $L(T) \leq 2\alpha^{c_3}$. 
\end{thm}
\begin{proof}
Under the assumption that Theorem~\ref{thm:local_convergence_rank-one} holds, the following holds (See~\eqref{eqn:convergence_simplified})
\begin{align}
    L(t)\!\leq\! \exp\biggl(-\frac{(1\!-\!\delta_w)\sigma\sigma_{W_2}(t\!-\!T_1)}{32}\biggl)L(T_1) \!+\! \alpha^{c_3}\!\leq\! \exp\biggl(-\frac{(1\!-\!\delta_w)\sigma\sigma_{W_2}(t\!-\!T_1)}{32}\biggl)L(0) \!+\! \alpha^{c_3}\,,
\end{align}
where the last inequality holds because $L$ under GF is non-increasing.
Then, we choose $t=T_2+T_1$ such that $\exp\biggl(-\frac{(1\!-\!\delta_w)\sigma\sigma_{W_2}(t-T_1)}{32}\biggl)L(0) \!=\! \alpha^{c_3}$, or equivalently 
\begin{align}
    T_2= \frac{32}{(1\!-\!\delta_w)\sigma\sigma_{W_2}}\log\biggl(\frac{L(0)}{\alpha^{c_3}}\biggl)\,,
\end{align}
then we have $L(T_1+T_2)\leq 2\alpha^{c_3} $. Finally, we choose $c_2$ in Theorem~\ref{thm:rank-one} such that
\begin{align}
    c_2=\frac{32c_3\log\bigl(\frac{L(0)}{\alpha}\bigl)}{(1-\delta_w)\log\bigl(\frac{L(0)}{\alpha}\bigl)}=\frac{32c_3}{1-\delta_w}\,,
\end{align}
which completes the proof.
\end{proof}
\section{Proof of Claim~\ref{cor:end_of_alignment}}\label{app:proof_alignment}
In this section, we provide the proof of Claim~\ref{cor:end_of_alignment}. Our proof strategy is to divide the \textit{alignment} phase 
into two distinct subphases: \textit{early alignment} phase and \textit{escape saddle} phase. In the \textit{early alignment}, we will show that
$\boldsymbol{U_{Z_1}^S}(t)$ aligns with $\gamma_1$, and the norm of the LoRA weights stay small. 
In the \textit{escape saddle} phase, we will show that $\boldsymbol{U_{Z_1}^S}(t)$ continues to align with $\gamma_1$ while the norm of the $Z_1$ move away from zero. Claim~\ref{cor:end_of_alignment} can be viewed as a consequence of the \textit{escape saddle} phase.

\subsection{Proof of \textit{early alignment} phase}
In this section, we first provide a characterization of the \textit{early alignment} phase, and its proof.
\begin{thm}[\textit{early alignment} phase]\label{appthm:early_alignment_rank-one}
Under the same setting as Theorem~\ref{appthm:rank-one}, we define
\begin{align}
    \hat T_1=\frac{2}{(5-\delta_w)\sigma\sigma_{W_2}}\log\biggl(\frac{1}{\alpha z_{\max}^2\bigl(\lVert E\rVert\!+\!\lVert W_1\rVert^2\!+\!\lVert W_1\rVert \lVert W_2\rVert+\!\lVert W_2\rVert^2\! \!+\!\sqrt{\lVert W_1\rVert+\lVert W_2\rVert}\bigl)}\biggl)\,.
\end{align}
There exists constants $\beta_1, \beta_2, \beta_3$ that is independent of the initialization scale $\alpha$ such that the following holds
\begin{align}
    \frac{\lVert Z_1(\hat T_1)\rVert_F^2}{\lVert Z_2(\hat T_1)\rVert_F^2} &\geq \beta_1 \alpha^{-\frac{2(1-\delta_w)}{5-\delta_w}}\,,\nonumber\\
    \cos^2\biggl (\gamma_1, \boldsymbol{U_{Z_1}^S}(\hat T_1)\biggl)&\geq 1-\beta_2 \alpha^{\frac{3+\delta_w}{5-\delta_w}}\,,\nonumber\\
    \lVert Z_1(\hat T_1)\rVert_F, \lVert Z_2(\hat T_1)\rVert_F &\leq \beta_3 \alpha\,.
\end{align}
\end{thm}
\begin{proof}
Our proof for Theorem~\ref{appthm:early_alignment_rank-one} is based on the following decomposition of the learning dynamics of LoRA (which is also shown in~\eqref{eqn:early_stage_dynamics}),
\begin{align}
    \dot Z_1 &\!=\! \bigl(\sigma\sigma_{W_{2}}H_1\!+\!\hat D_1\bigl) Z_1\,,\nonumber\\
    \dot Z_2 \!&=\! \bigl(\sigma\sigma_{W_1}H_2\!+\!\hat D_2\bigl) Z_2 \,,
\end{align}
where $H_1, H_2, D_1,D_2$ are defined as follows
\begin{align}
    H_1 &\!=\! \begin{pmatrix}
        0 & gv^\top \\ vg^\top & 0
    \end{pmatrix}, \quad H_2 \!=\!\begin{pmatrix}
        0 & ug^\top \\ gu^\top & 0
    \end{pmatrix}, \quad
    F \!=\! W_2B_1A_1\!+\!B_2A_2W_1\!+\!B_2A_2B_1A_1, \quad E\!=\!\Delta Y\!-\!F,\nonumber \\
    \hat D_1 &\!=\! \begin{pmatrix}
        0&D_1\\D_1^\top&0
    \end{pmatrix}, \quad \hat D_2 \!=\! \begin{pmatrix}
        0&D_2\\D_2^\top&0
    \end{pmatrix},\quad D_1 \!=\!A_2^\top B_2^\top E\!-\!W_2^\top F, \quad D_2 \!=\! EA_1^\top B_1^\top\!-\! FW_1^\top \,.
\end{align}
Our analysis is built on the observation that during the early stage of the training, $\lVert \hat D_1\rVert, \lVert \hat D_2\rVert\sim\mathcal{O}(\alpha^2)$, which is extremely small compared with the magnitude of $\sigma\sigma_{W_1}H_2, \sigma\sigma_{W_2}H_1$.
Claim~\ref{lem:ode_alignment} characterizes

We first characterize the learning dynamics of the angle and norm of each column of $Z_1$ (or equivalently $Z_2$) as follows
\begin{align}\label{appeqn:ode_norm_angle}
    \frac{d}{dt}\lVert w_{1j}\rVert^2 &= 2\langle \dot w_{1j}, w_{1j}\rangle=2\langle w_{1j}, (\sigma\sigma_{W_2}H_1+\hat D_1)w_{1j}\rangle \nonumber\\
    \frac{d}{dt}\frac{w_{1j}}{\lVert w_{1j}\rVert} &= \frac{\dot w_{1j}}{\lVert w_{1j}\rVert}-\frac{w_{1j}}{\lVert w_{1j}\rVert^2}\cdot \frac{\dot w_{1j}^\top w_{1j}}{\lVert w_{1j}\rVert}
    =\sigma\sigma_{W_2}\biggl(I-\frac{w_{1j}w_{1j}^\top}{\lVert w_{1j}\rVert^2}\biggl)\frac{H_1 w_{1j}}{\lVert w_{1j}\rVert}+\biggl(I-\frac{w_{1j}w_{1j}^\top}{\lVert w_{1j}\rVert^2}\biggl)\frac{\hat D_1w_{1j}}{\lVert w_{1j}\rVert}\,.
\end{align}
Based on the above equation, one can characterize the growth of $\lVert w_{1j}\rVert, \frac{\lVert Z_1\rVert_F}{\lVert Z_2\rVert_F}$ and angle alignment between $w_{1j}$ with $\gamma_1$, as is shown the following lemma
\begin{lem}\label{applem:grow_norm_imbalance}
For any unit vector $\hat\gamma_1\perp \gamma_1$, the following conditions hold 
\begin{align}
    \frac{d}{dt}\log(\lVert w_{1j}\rVert^2)&\leq 2(\sigma\sigma_{W_2}+\lVert D_1\rVert)\,,\quad \frac{d}{dt}\log(\lVert Z_1\rVert_F^2)\leq 2(\sigma\sigma_{W_2}+\lVert D_1\rVert)\nonumber\\
    \frac{d}{dt}\log\biggl(\frac{\lVert Z_1\rVert_F^2}{\lVert Z_2\rVert_F^2}\biggl) &\geq 2((1-\delta_w)\sigma\sigma_{W_2}-\lVert D_1\rVert-\lVert D_2\rVert)\,,\nonumber\\
    \frac{d}{dt}\log\biggl(\frac{\cos(\gamma_1, \frac{w_{1j}}{\lVert w_{1j}\rVert})}{\cos(\hat\gamma_1, \frac{w_{1j}}{\lVert w_{1j}\rVert})}\biggl)&\geq\sigma\sigma_{W_2}-2\lVert D_1\rVert\,.
\end{align}
\end{lem}
We refer the readers to Appendix~\ref{appprooflem:grow_norm_imbalance} for the proof of Lemma~\ref{applem:grow_norm_imbalance}. 
Based on Lemma~\ref{applem:grow_norm_imbalance}, one can see that when $\lVert D_1\rVert, \lVert D_2\rVert$ is sufficiently small, one can control the growth of $\lVert w_{1j}\rVert, \lVert Z_1\rVert_F$ (and $\lVert w_{2j}\rVert, \lVert Z_2\rVert_F$ respectively) and $\frac{\lVert Z_1\rVert_F^2}{\lVert Z_2\rVert_F^2}, \cos(\gamma_1, \frac{w_{1j}}{\lVert w_{1j}\rVert})$ monotonically increase.
The following lemma formally characterize these properties.
\begin{lem}\label{applem:grow_norm_imbalance2}
We first define the following quantities
\begin{align}
    w_{\max}&\!=\!\frac{1}{\alpha}\max_{j\in[r]}\biggl\{{\lVert w_{1j}\rVert, \lVert w_{2j}\rVert}\biggl\}, \quad z_{\max} \!=\! \frac{1}{\alpha}\max\biggl\{\lVert Z_1(0)\rVert_F, \lVert Z_2(0)\rVert_F\biggl\}\nonumber\\
    c_{\min}&\!=\!\min_{j\in[r]}\biggl\{\frac{\cos\biggl (\gamma_1, \frac{w_{1j}(0)}{\lVert w_{1j}(0)\rVert}\biggl)}{\cos\biggl(\hat\gamma_1, \frac{w_{1j}(0)}{\lVert w_{1j}(0)\rVert}\biggl)}\biggl\}\nonumber\\
    d_1&\!=\!\frac{\lVert Z_1(0)\rVert_F}{\lVert Z_2(0)\rVert_F}, \quad d_2=\frac{(m+h-1)}{c_{\min}}\,.
\end{align}
Assume until $T_1(\delta_1, \delta_2)$, the following holds $\lVert D_1(T_1(\delta_1, \delta_2))\rVert\leq\delta_1, \lVert D_2(T_1(\delta_1, \delta_2))\rVert\leq\delta_2$. 
Then, the following holds for all $0\leq t\leq T_1(\delta_1, \delta_2)$
\begin{align}
    \lVert w_{1j}(t)\rVert^2 &\leq \alpha^2\exp\biggl(2(\sigma\sigma_{W_2}+\delta_1)t\biggl) w_{\max}^2\,, \nonumber\\
    \lVert Z_1(t)\rVert_F^2 &\leq \alpha^2\exp\biggl(2(\sigma\sigma_{W_2}+\delta_1)t\biggl) z_{\max}^2\,,\nonumber\\
    \frac{\lVert Z_1(t)\rVert_F^2}{\lVert Z_2(t)\rVert_F^2} &\geq d_1\exp\biggl (2\bigl ((1-\delta_w)\sigma\sigma_{W_2}-\delta_1-\delta_2\bigl)t \biggl)\,,\nonumber\\
    \cos^2\biggl (\gamma_1, \frac{w_{1j}}{\lVert w_{1j}\rVert}\biggl) &\geq \frac{1}{1+d_2\exp\bigl(-2(\sigma\sigma_{W_2}-2\delta_1)t\bigl)}\,.
\end{align}
\end{lem}
We refer the readers to Appendix~\ref{appprooflem:grow_norm_imbalance2} for the proof of Lemma~\ref{applem:grow_norm_imbalance2}. 
Since we initialize $A_i=0, B_i\sim\mathcal{N}(0, \alpha^2)$, therefore the value of $w_{\max}, z_{\max}, c_{\min}, d_1, d_2$ are all independent of $\alpha$.
Moreover, when $T_1(\delta_1,\delta_2)$ is large, one can show there will be sufficient imbalance and alignment at $T_1(\delta_1,\delta_2)$. 
Next, we will show that when $\delta_1=\delta_2=\alpha$, we can characterize a lower bound on $T_1(\delta_1, \delta_2)$ which depends on $\alpha$, and it further leads to an characterization of the alignment and imbalance between $\lVert Z_1\rVert_F,\lVert Z_2\rVert_F$ evaluated at $T_1(\delta_1,\delta_2)$ depending on $\alpha$.
\begin{lem}\label{applem:grow_norm_imbalance3}
Under the same setting as Theorem~\ref{appthm:early_alignment_rank-one}, we define $\hat T_1:=T_1(\alpha, \alpha)$ as follows
\begin{align}
    \hat T_1=\frac{2}{(5-\delta_w)\sigma\sigma_{W_2}}\log\biggl(\frac{1}{\alpha z_{\max}^2(\lVert E\rVert\!+\!\lVert W_2\rVert^2\!+\!\lVert W_1\rVert \lVert W_2\rVert \!+\!\sqrt{\lVert W_2\rVert})}\biggl)\,.
\end{align}
There exists constants $\beta_1, \beta_2, \beta_3$ that is independent of the initialization scale $\alpha$ such that the following holds
\begin{align}
    \frac{\lVert Z_1(\hat T_1)\rVert_F^2}{\lVert Z_2(\hat T_1)\rVert_F^2} &\geq d_1 d_3^{\frac{2(1-\delta_w)}{5-\delta_w}} \alpha^{-\frac{2(1-\delta_w)}{5-\delta_w}}\,,\nonumber\\
    \cos^2\biggl (\gamma_1, \frac{w_{1j}(\hat T_1)}{\lVert w_{1j}(\hat T_1)\rVert}\biggl)&\geq 1-d_2d_3^{-\frac{3+\delta_w}{5-\delta_w}} \alpha^{\frac{3+\delta_w}{5-\delta_w}}\,,\nonumber\\
    \lVert Z_1(\hat T_1)\rVert_F, \lVert Z_2(\hat T_1)\rVert_F &\leq \frac{\alpha}{d_3}\,,
\end{align}
where $d_3=\frac{1}{z_{\max}^2(\lVert E\rVert\!+\!\lVert W_2\rVert^2\!+\!\lVert W_1\rVert \lVert W_2\rVert \!+\!\sqrt{\lVert W_2\rVert})}$.
\end{lem}
We refer the readers to Appendix~\ref{appprooflem:grow_norm_imbalance3} for details. Then, in the remaining of the proof of Theorem~\ref{appthm:early_alignment_rank-one},
we will show that $\cos^2\biggl (\gamma_1, \frac{w_{1j}(\hat T_1)}{\lVert w_{1j}(\hat T_1)\rVert}\biggl)\geq 1-d_2d_3^{-\frac{3+\delta_w}{5-\delta_w}} \alpha^{\frac{3+\delta_w}{5-\delta_w}}$ ensures sufficiently alignment between $\gamma_1$ and $\boldsymbol{U_{Z_1}^S}(\hat T_1)$.
\end{proof}

Our starting point is the observation that 
\begin{align}
    \lVert \gamma_1\gamma_1^\top Z_1(\hat T_1)\rVert_F^2 &= \sum_{j=1}^2\lVert \gamma_1\gamma_1^\top w_{1j}(\hat T_1)\rVert^2 \nonumber\\
    &= \sum_{j=1}^2 \cos^2\biggl (\gamma_1, \frac{w_{1j}(\hat T_1)}{\lVert w_{1j}(\hat T_1)\rVert}\biggl) \lVert w_{1j}(\hat T_1)\rVert^2\nonumber\\
    &\geq \biggl(1-d_2d_3^{-\frac{3+\delta_w}{5-\delta_w}} \alpha^{\frac{3+\delta_w}{5-\delta_w}}\biggl)\lVert Z_1(\hat T_1)\rVert_F^2\,.
\end{align}
and $\lVert (I-\gamma_1\gamma_1^\top) Z_1(\hat T_1)\rVert_F^2\leq d_2d_3^{-\frac{3+\delta_w}{5-\delta_w}} \alpha^{\frac{3+\delta_w}{5-\delta_w}}\lVert Z_1(\hat T_1)\rVert_F^2$.

Then, we can apply Lemma~\ref{applem:wedin-sin} (singular space perturbation bound) with $M=Z_1(\hat T_1), M^*=\gamma_1\gamma_1^\top Z_1(\hat T_1)$ and $r=1$:
\begin{align}
    \lVert \gamma_1\gamma_1^\top - \boldsymbol{U_{Z_1}^S}(\hat T_1)\bigl( \boldsymbol{U_{Z_1}^S}(\hat T_1)\bigl)^\top\rVert &\leq \frac{\lVert (I-\gamma_1\gamma_1^\top)Z_1(\hat T_1)\rVert}{\lVert\gamma_1\gamma_1^\top Z_1(\hat T_1)\rVert-\lVert (I-\gamma_1\gamma_1^\top)Z_1(\hat T_1)\rVert}\nonumber\\
    &\leq \frac{d_2d_3^{-\frac{3+\delta_w}{5-\delta_w}} \alpha^{\frac{3+\delta_w}{5-\delta_w}}}{1-2d_2d_3^{-\frac{3+\delta_w}{5-\delta_w}} \alpha^{\frac{3+\delta_w}{5-\delta_w}}}\nonumber\\
    &\leq 2d_2d_3^{-\frac{3+\delta_w}{5-\delta_w}} \alpha^{\frac{3+\delta_w}{5-\delta_w}}\,,
\end{align}
where the last inequality holds due to the assumption that $d_2d_3^{-\frac{3+\delta_w}{5-\delta_w}}\alpha^{\frac{3+\delta_w}{5-\delta_w}}\leq\frac{1}{2}$.
\subsection{Proof of \textit{escape saddle} phase}
At the end of the \textit{early alignment} phase, we have shown sufficient alignment between each column of $Z_1(\hat T_1)$ (and $\boldsymbol{U_{Z_1}^S}(\hat T_1)$) with $\gamma_1$, 
along with a significant imbalance between $\lVert Z_1(\hat T_1)\rVert_F$ and $\lVert Z_2(\hat T_1)\rVert_F$. 
Moreover, the norm of the LoRA weights remains small at this stage. 
In the \textit{escape saddle} phase, we show that this alignment and imbalance persist, while the norm of the LoRA weights gradually increases until reaching a certain threshold.

The following theorem characterizes these properties formally.
\begin{thm}[\textit{escape saddle} phase]
Under the same setting as Theorem~\ref{appthm:rank-one}, when $\alpha$ is sufficiently small, let
\begin{align}
    \tilde T_1 = \frac{2\log\biggl(1+\frac{(1-\delta_w^2)\sigma^2}{4\alpha^2z_{\max}^2}\biggl)}{(1+\delta_w)\sigma\sigma_{W_2}}\,.
\end{align}
Then, there exists a time $T_1\in[\hat T_1, \hat T_1+\tilde T_1]$ such that the following holds
\begin{align}
    g^\top B_1(T_1)A_1(T_1) v &\geq \frac{(1-\delta_w)\sigma}{4\sigma_{W_2}}\,, \nonumber\\
    \frac{\lVert Z_1(T_1)\rVert_F^2}{\lVert Z_2(T_1)\rVert_F^2} &\geq \beta_1 \alpha^{-\frac{2(1-\delta_w)}{5-\delta_w}}\,,\nonumber\\
    \cos^2\biggl (\gamma_1, \boldsymbol{U_{Z_1}^S}(T_1)\biggl)&\geq 1-\beta_2 \alpha^{\frac{3+\delta_w}{5-\delta_w}}\,.
\end{align}
\end{thm}
\begin{proof}
In this regime, we first show that when $\lVert D_1\rVert, \lVert D_2\rVert\leq\frac{(1-\delta_w)\sigma\sigma_{W_2}}{2}$, 
each column of $Z_1$ continues to align with $\gamma_1$, 
and $\frac{\lVert Z_1\rVert_F}{\lVert Z_2\rVert_F}$ continues to grow. 
This is because in Lemma~\ref{applem:grow_norm_imbalance} and Lemma~\ref{applem:grow_norm_imbalance2}, we have shown that
\begin{align}
    \frac{d}{dt}\log\biggl(\frac{\lVert Z_1\rVert_F^2}{\lVert Z_2\rVert_F^2}\biggl) &\geq 2((1-\delta_w)\sigma\sigma_{W_2}-\lVert D_1\rVert-\lVert D_2\rVert)\,,\nonumber\\
    \frac{d}{dt}\log\biggl(\frac{\cos(\gamma_1, \frac{w_{1j}}{\lVert w_{1j}\rVert})}{\cos(\hat\gamma_1, \frac{w_{1j}}{\lVert w_{1j}\rVert})}\biggl)&\geq\sigma\sigma_{W_2}-2\lVert D_1\rVert\,.
\end{align}
When $\lVert D_1\rVert, \lVert D_2\rVert\leq\frac{(1-\delta_w)\sigma\sigma_{W_2}}{2}$, we have
\begin{align}
    \frac{d}{dt}\log\biggl(\frac{\lVert Z_1\rVert_F^2}{\lVert Z_2\rVert_F^2}\biggl) &\geq 0\,,\nonumber\\
    \frac{d}{dt}\log\biggl(\frac{\cos(\gamma_1, \frac{w_{1j}}{\lVert w_{1j}\rVert})}{\cos(\hat\gamma_1, \frac{w_{1j}}{\lVert w_{1j}\rVert})}\biggl)&\geq\sigma\sigma_{W_2}-2\lVert D_1\rVert \geq \delta\sigma\sigma_{W_2}>0\,.
\end{align}
Then, we characterize speed of growth of  $g^\top B_1A_1 v$.
\begin{lem}\label{applem:ab}
One can characterize the growing speed of $g^\top B_1A_1 v$ as follows
\begin{align}
    \frac{d}{dt} g^\top B_1A_1 v \geq \bigl (\sigma\sigma_{W_2}-\lVert D_1\rVert\bigl) \lVert Z_1\rVert_F^2 \min_{j\in [r]}\cos^2\biggl (\gamma_1, \frac{w_{1j}}{\lVert w_{1j}\rVert}\biggl)\,.
\end{align}
\end{lem}
The detailed proof of Lemma~\ref{applem:ab} can be found in Appendix~\ref{appprooflem:ab}

Based on Lemma~\ref{applem:ab}, we can show that for all $t\geq \hat T_1$, we have
\begin{align}
    \frac{d}{dt} g^\top B_1A_1 v &\geq \frac{(1+\delta_w)\sigma\sigma_{W_2}}{2} \lVert Z_1\rVert_F^2 \bigl(1-\beta_2 \alpha^{\frac{3+\delta_w}{5-\delta_w}}\bigl)\nonumber\\
    &\geq \alpha^2 z_{\max}^2\exp\biggl(\frac{(1+\delta_w)\sigma\sigma_{W_2}t}{2}\biggl)\frac{(1+\delta_w)\sigma\sigma_{W_2}}{2} \bigl(1-\beta_2 \alpha^{\frac{3+\delta_w}{5-\delta_w}}\bigl)\nonumber\\
    \iff& g^\top B_1A_1 v \geq \alpha^2 z_{\max}^2\frac{(1+\delta_w)\sigma\sigma_{W_2}}{2} \bigl(1-\beta_2 \alpha^{\frac{3+\delta_w}{5-\delta_w}}\bigl)
    \frac{\exp\biggl(\frac{(1+\delta_w)\sigma\sigma_{W_2}t}{2}\biggl)-1}{\frac{(1+\delta_w)\sigma\sigma_{W_2}}{2}}\,.
\end{align}
Thus, for the time for $g^\top B_1A_1 v$ to reach $\frac{(1-\delta_w)\sigma}{4\sigma_{W_2}}$ is upper bounded by
\begin{align}
    \tilde T_1 = \frac{2\log\biggl(1+\frac{(1-\delta_w^2)\sigma^2}{4\alpha^2z_{\max}^2}\biggl)}{(1+\delta_w)\sigma\sigma_{W_2}}\,.
\end{align}
Notice the above $\tilde T_1$ is derived under the assumption that $\lVert D_1\rVert, \lVert D_2\rVert\leq\frac{(1-\delta_w)\sigma\sigma_{W_2}}{2}$ for all $\hat T_1\leq t\leq\hat T_1$.
However, it is possible for $\lVert D_1\rVert, \lVert D_2\rVert$ to reach $\frac{(1-\delta_w)\sigma\sigma_{W_2}}{2}$ during this time. 
In the following section, we assume $\lVert D_1\rVert$ reaches $\frac{(1-\delta_w)\sigma\sigma_{W_2}}{2}$ first at time $t_1^*$ where $\hat T_1\leq t_1^*\leq\hat T_1$. 
Then, we show how to derive a lower bound on $g^\top B_1A_1 v$ based on this condition. 

We first provide a lemma that will be used in the remaining of the proof.
\begin{lem}\label{applem:ab_vs_z}
Let $\Lambda_1 = \frac{1}{\alpha^2}(B_1(0)^\top B_1(0) - A_1(0)A_1(0)^\top)$. Then, when $\alpha$ is sufficiently small, the following condition holds for all $t\geq 0$
\begin{align}
    \frac{\lVert Z_1\rVert_F^2-2\alpha^2 r^2 \lVert \Lambda_1\rVert}{2r} \leq \lVert B_1A_1\rVert \leq \frac{1}{2}\lVert Z_1\rVert_F^2\,. 
\end{align}
\end{lem}
\begin{rem}
Notice when LoRA weights are trained via GF, $D_1$ is constant during the training. Similar argument has been shown in ~\cite{saxe2013exact, tarmoun2021understanding}. 
Moreover, since $B_1(0)$ is initialized as a zero matrix,
and $A_1(0)$ is initialized entry-wise i.i.d. using $\mathcal{N}(0, \alpha^2)$, $D_1$ defined as above is determined by a random matrix whose entry are $\mathcal{N}(0, 1)$.
Thus, when $\alpha$ is chosen to be small, it does not affect the magnitude of $D_1$.
\end{rem}
We refer the readers to Appendix~\ref{appprooflem:ab_vs_z} for detailed proof of Lemma~\ref{applem:ab_vs_z}.

Now, we show an lower bound on $g^\top B_1A_1 v$ as follows
\begin{align}
    \lVert D_1\rVert =& \lVert A_2^\top B_2^\top E - W_2^\top (W_2B_1A_1+B_2A_2W_1+B_2A_2B_1A_1)\rVert \nonumber\\
    \leq& \lVert W_2^\top W_2 B_1A_1\rVert + \lVert B_2A_2\rVert \biggl(\lVert E\rVert + \lVert W_1\rVert \lVert W_2\rVert + \lVert W_2\rVert \lVert B_1A_1\rVert\biggl)\nonumber\\
    =& \lVert uu^\top W_2^\top W_2 B_1A_1 vv^\top \rVert + \lVert uu^\top W_2^\top W_2 B_1A_1 v_{\perp}v_{\perp}^\top \rVert + \lVert u_{\perp}u_{\perp}^\top W_2^\top W_2 B_1A_1 \rVert\nonumber\\
     &+ \lVert B_2A_2\rVert \biggl(\lVert E\rVert + \lVert W_1\rVert \lVert W_2\rVert + \lVert W_2\rVert \lVert B_1A_1\rVert\biggl) \nonumber\\
    =&\lVert \sigma_{W_2}^2 ug^\top  B_1A_1 vv^\top \rVert + \lVert \sigma_{W_2}^2 ug^\top B_1A_1 v_{\perp}v_{\perp}^\top \rVert + \lVert u_{\perp}u_{\perp}^\top W_2^\top W_2 B_1A_1 \rVert\nonumber\\
    &+ \lVert B_2A_2\rVert \biggl(\lVert E\rVert + \lVert W_1\rVert \lVert W_2\rVert + \lVert W_2\rVert \lVert B_1A_1\rVert\biggl) \nonumber\\
    \leq & \sigma_{W_2}^2 g^\top B_1A_1 v + \sigma_{W_2}^2\lVert B_1A_1 v_{\perp}v_{\perp}^\top \rVert + \lVert W_2\rVert^2 \lVert u_{\perp}u_{\perp}^\top B_1A_1\rVert \nonumber\\
    &+ \lVert B_2A_2\rVert \biggl(\lVert E\rVert + \lVert W_1\rVert \lVert W_2\rVert + \lVert W_2\rVert \lVert B_1A_1\rVert\biggl)\,.
\end{align}
The first term on the RHS is our target quantity, we will show that the rest of the terms on the RHS is small.

For $\lVert B_1A_1 v_{\perp}v_{\perp}^\top \rVert$, we can use the alignment between each column of $Z_1$ with $\gamma_1$ to lower bound it.
Specifically, in the \textit{early alignment} phase, we have proved that each column of $Z_1$ aligns with $\gamma_1$, i.e., $\cos(\gamma_1, \frac{w_{1j}}{\lVert w_{1j}\rVert}) \geq 1-\beta_2 \alpha^{\frac{3+\delta_w}{5-\delta_w}}$.
Moreover, in the previous discussion, we can see that this alignment persist in the \textit{escape saddle} phase. 
This property ensures that each column of $Z_1$ aligns well with $\gamma_1$. 
The following lemma characterizes this property ensures $\lVert g_{\perp}g_{\perp}^\top B_1A_1\rVert^2, \lVert B_1A_1 v_{\perp}v_{\perp}^\top\rVert^2$ are small.

\begin{lem}\label{applem:alignment_z1_to_a1b1}
Under the assumption that $\cos^2(\gamma_1, \frac{w_{1j}}{\lVert w_{1j}\rVert}) \geq 1-\beta_2 \alpha^{d}$, 
then one can show that 
\begin{align}
    \lVert g_{\perp}g_{\perp}^\top B_1A_1\rVert, \lVert B_1A_1 v_{\perp}v_{\perp}^\top\rVert \leq \sqrt{\beta_2 \alpha^d}\lVert Z_1\rVert_F^2
\end{align}
\end{lem}
We refer the readers to Appendix~\ref{appprooflem:alignment_z1_to_a1b1} for the proof.
Based on Lemma~\ref{applem:alignment_z1_to_a1b1}, we can show that
\begin{align}
    \lVert B_1A_1 v_{\perp}v_{\perp}^\top\rVert^2, \lVert g_{\perp}g_{\perp}^\top B_1A_1\rVert^2\leq \beta_2 \alpha^{\frac{3+\delta_w}{5-\delta_w}}\lVert Z_1\rVert_F^4\,.
\end{align}
Based on these, we can further show that
\begin{align}
    \lVert D_1\rVert &\leq \sigma_{W_2}^2 g^\top B_1A_1 v + \sigma_{W_2}^2\lVert B_1A_1 v_{\perp}v_{\perp}^\top \rVert + \lVert W_2\rVert^2 \lVert u_{\perp}u_{\perp}^\top B_1A_1\rVert \nonumber\\
    &+ \lVert B_2A_2\rVert \biggl(\lVert E\rVert + \lVert W_1\rVert \lVert W_2\rVert + \lVert W_2\rVert \lVert B_1A_1\rVert\biggl)\nonumber\\
    &\leq \sigma_{W_2}^2 g^\top B_1A_1 v + 2\lVert W_2\rVert^2 \sqrt{\beta_2} \alpha^{\frac{3+\delta_w}{10-2\delta_w}}\lVert Z_1\rVert_F^2\nonumber\\
    &+ \frac{1}{2}\lVert Z_2\rVert_F^2 \biggl(\lVert E\rVert + \lVert W_1\rVert \lVert W_2\rVert + \lVert W_2\rVert \frac{1}{2}\lVert Z_1\rVert_F^2\biggl) \nonumber\\
    &\leq \sigma_{W_2}^2 g^\top B_1A_1 v + 2\lVert W_2\rVert^2 \sqrt{\beta_2} \alpha^{\frac{3+\delta_w}{10-2\delta_w}}\lVert Z_1\rVert_F^2\nonumber\\
    &+ \frac{1}{2\beta_1} \alpha^{\frac{2(1-\delta_w)}{5-\delta_w}}\lVert Z_1\rVert_F^2 \biggl(\lVert E\rVert + \lVert W_1\rVert \lVert W_2\rVert + \lVert W_2\rVert \frac{1}{2}\lVert Z_1\rVert_F^2\biggl)
\end{align}
If one can show that $\lVert Z_1\rVert_F^2$ is upper bounded, i.e., $\lVert Z_1\rVert_F^2 \leq d_4$, then we can show 
\begin{align}
    g^\top B_1A_1 v &\geq \frac{(1-\delta)\sigma}{2\sigma_{W_2}} - \frac{2\lVert W_2\rVert^2 \sqrt{\beta_2} \alpha^{\frac{3+\delta_w}{10-2\delta_w}}d_4^2
    + \frac{1}{2\beta_1} \alpha^{\frac{2(1-\delta_w)}{5-\delta_w}}d_4^2 \biggl(\lVert E\rVert + \lVert W_1\rVert \lVert W_2\rVert + \lVert W_2\rVert \frac{1}{2}d_4^2\biggl)}{\sigma_{W_2}^2}\nonumber\\
    &\geq \frac{(1-\delta)\sigma}{4\sigma_{W_2}}\,,
\end{align}
where the last inequality holds when $\alpha$ is sufficiently small, and the absolute value of all the negative term are less or equal than $\frac{(1-\delta)\sigma}{4\sigma_{W_2}}$.

Finally, we show at $t_1^*$, $\lVert Z_1\rVert$ is upper bounded.
By the assumption that $\lVert D_1\rVert$ reaches $\frac{(1-\delta_w)\sigma\sigma_{W_2}}{2}$ before $g^\top B_1A_1 v$ reaches $\frac{(1-\delta_w)\sigma}{4\sigma_{W_2}}$.
We start with the following observation
\begin{align}\label{appeqn:gammaz_z}
    \lVert \gamma_1\gamma_1^\top Z_1 \rVert_F^2 &= \sum_{j=1}^2\lVert \gamma_1\gamma_1^\top w_{1j} \rVert^2 \nonumber\\
    &= \sum_{j=1}^2 \cos^2\biggl (\gamma_1, \frac{w_{1j}}{\lVert w_{1j}\rVert}\biggl) \lVert w_{1j}\rVert^2\nonumber\\
    &\geq \biggl(1-d_2d_3^{-\frac{3+\delta_w}{5-\delta_w}} \alpha^{\frac{3+\delta_w}{5-\delta_w}}\biggl)\lVert Z_1 \rVert_F^2\,.
\end{align}
Thus, if one can show that $\lVert \gamma_1\gamma_1^\top Z_1 \rVert_F^2$ is upper bounded, then when $\alpha$ is small, it directly implies $\lVert Z_1\rVert_F^2$ is bounded.
\begin{align}\label{appeqn:prove_z1_bounded1}
    \lVert \gamma_1\gamma_1^\top Z_1 \rVert_F^2 &= g^\top B_1B_1^\top g + v^\top A_1^\top A_1 v + 2g^\top B_1A_1 v\leq g^\top B_1B_1^\top g + v^\top A_1^\top A_1 v + \frac{(1-\delta_w)\sigma}{2\sigma_{W_2}}\,.
\end{align}
Due to the property that $A_1 A_1^\top-B_1^\top B_1 = \alpha^2 \Lambda_1$ is small, one can connect $g^\top B_1B_1^\top g + g^\top B_1B_1^\top g$ with $g^\top B_1A_1 v$. 
The following Lemma characterizes this formally.
\begin{lem}\label{applem:ab_square}
\begin{align}
    (g^\top B_1B_1^\top g)^2 &\leq(v^\top A_1^\top B_1^\top g)^2 + \alpha^2 \lVert D_1\rVert  g^\top B_1B_1^\top g 
    + 2\beta_2 \alpha^{\frac{3+\delta_w}{5-\delta_w}}\lVert Z_1\rVert_F^4\nonumber\\
    (v^\top A_1^\top A_1 v)^2 &\leq (v^\top B^\top A_1^\top g)^2 + \alpha^2 \lVert D_1\rVert  v^\top A_1^\top A_1 v 
    + 2\beta_2 \alpha^{\frac{3+\delta_w}{5-\delta_w}}\lVert Z_1\rVert_F^4\,.
\end{align}
\end{lem}
We refer the readers to Appendix~\ref{appprooflem:ab_square} for detailed proof of Lemma~\ref{applem:ab_square}. 
Based on Lemma~\ref{applem:ab_square}, one has
\begin{align}
    \biggl(v^\top A_1^\top A_1 v + g^\top B_1B_1^\top g\biggl)^2 &\leq 2(v^\top A_1^\top A_1 v)^2+2(g^\top B_1B_1^\top g)^2 \nonumber\\
    &\leq 2 (v^\top A_1^\top B_1^\top g)^2 + \alpha^2 \lVert D_1\rVert\biggl(v^\top A_1^\top A_1 v 
    + g^\top B_1B_1^\top g\biggl)+4\beta_2 \alpha^{\frac{3+\delta_w}{5-\delta_w}}\lVert Z_1\rVert_F^4\,.
\end{align}
And it leads to the following upper bound on $g^\top B_1B_1^\top g + v^\top A_1^\top A_1 v$
\begin{align}\label{appeqn:gammaZ_ab}
    g^\top B_1B_1^\top g + v^\top A_1^\top A_1 v &\leq \frac{\alpha^2 \lVert D_1\rVert + \sqrt{\alpha^4 \lVert D_1\rVert^2 + 8(g^\top B_1 A_1 v)^2 + 16\beta_2 \alpha^{\frac{3+\delta_w}{5-\delta_w}}\lVert Z_1\rVert_F^4}}{2}\nonumber\\
    &\leq \alpha^2 \lVert D_1\rVert + \sqrt{2} g^\top B_1 A_1 v+2\sqrt{\beta}\lVert Z_1\rVert_F^2 \alpha^{\frac{3+\delta_w}{10-2\delta_w}}
\end{align}
Then, we plug in the above equation to~\eqref{appeqn:prove_z1_bounded1}.
\begin{align}
    \lVert \gamma_1\gamma_1^\top Z_1\rVert_F^2 &\leq g^\top B_1B_1^\top g + v^\top A_1^\top A_1 v + \frac{(1-\delta_w)\sigma}{2\sigma_{W_2}} \nonumber\\
    &\leq \alpha^2 \lVert D_1\rVert + 2\sqrt{\beta}\lVert Z_1\rVert^2 \alpha^{\frac{3+\delta_w}{10-2\delta_w}} + \frac{(2+\sqrt{2})(1-\delta_w)\sigma}{4\sigma_{W_2}}\,.
\end{align}
Combine this result with ~\eqref{appeqn:gammaz_z}, we can derive the following upper bound on $\lVert Z_1\rVert$
\begin{align}
    \lVert Z_1\rVert &\leq \frac{\alpha^2 \lVert D_1\rVert+\frac{(2+\sqrt{2})(1-\delta_w)\sigma}{4\sigma_{W_2}}}{1-d_2d_3^{-\frac{3+\delta_w}{5-\delta_w}} \alpha^{\frac{3+\delta_w}{5-\delta_w}} - 2\sqrt{\beta}\lVert Z_1\rVert \alpha^{\frac{3+\delta_w}{10-2\delta_w}}} \nonumber\\
    &\leq 2\lVert D_1\rVert + \frac{(2+\sqrt{2})(1-\delta_w)\sigma}{2\sigma_{W_2}}:=d_4\,,
\end{align}
where the last inequality holds when $\alpha$ is sufficiently small.

The second case is when $\lVert D_2\rVert=\frac{(1-\delta_w)\sigma\sigma_{W_2}}{2}$ happens first. 
We argue that if one lets the training goes on, one of the following happens: 
either $\lVert D_1\rVert=\frac{(1-\delta_w)\sigma\sigma_{W_2}}{2}$ or $g^\top B_1A_1 v=\frac{(1-\delta)\sigma}{4\sigma_{W_2}}$. 
Let us use time $T_1'$ to denote the time that the above event happens.
If one can show that there exists constant $d$ such that $\lVert Z_1(t)\rVert_F^2 \alpha^d \geq \lVert Z_2(t)\rVert_F^2$ holds for all $\hat T_1\leq t\leq T_1'$,
then $\lVert D_1\rVert=\frac{(1-\delta_w)\sigma\sigma_{W_2}}{2}$ implies $g^\top B_1A_1 v=\frac{(1-\delta_w)\sigma}{4\sigma_{W_2}}$.
Moreover, it is obvious that $T_1' \leq \hat T_1+\hat T_2$.
In the rest of the proof, we will show that one can actually find the constant $d$ which is independent of $\alpha$.

Our starting point is that
\begin{align}
    &\frac{d}{dt}\log\biggl(\frac{\lVert Z_1\rVert_F^2}{\lVert Z_2\rVert_F^2}\biggl) \geq 2((1-\delta_w)\sigma\sigma_{W_2}-\lVert D_1\rVert-\lVert D_2\rVert)\,,\nonumber\\
    \Rightarrow&\log\biggl(\frac{\lVert Z_1(t)\rVert_F^2}{\lVert Z_2(t)\rVert_F^2}\biggl)-\log\biggl(\frac{\lVert Z_1(\hat T_1)\rVert_F^2}{\lVert Z_2(\hat T_1)\rVert_F^2}\biggl)
    \geq (1-\delta_w)\sigma\sigma_{W_2}(t-\hat T_1)-\int_{\hat T_1}^t \lVert D_2(s)\rVert ds\nonumber\\
    \Rightarrow& \frac{\lVert Z_1(t)\rVert_F^2}{\lVert Z_2(t)\rVert_F^2} \geq \frac{\lVert Z_1(\hat T_1)\rVert_F^2}{\lVert Z_2(\hat T_1)\rVert_F^2}\exp\bigl(-\int_{\hat T_1}^t \lVert D_2(s)\rVert ds\bigl)
\end{align}
Therefore, it suffices to show that $\int_{\hat T_1}^{\hat T_1+T_1'} \lVert D_2(s)\rVert ds$ is upper bounded by a constant independent of $\alpha$. 
To show this, we first bound $\int_{\hat T_1}^{\hat T_1+T_1'} \lVert D_2(s)\rVert ds$ as follows
\begin{align}
    \lVert D_2\rVert_F^2 =& \lVert EA_1^\top B_1^\top -FW_1^\top\rVert_F^2\nonumber\\
    =&\lVert (uu^\top+u_{\perp}u_{\perp}^\top)(EA_1^\top B_1^\top -FW_1^\top)(gg^\top+g_{\perp}g_{\perp}^\top)\rVert_F^2\nonumber\\
    \leq& \lVert uu^\top(EA_1^\top B_1^\top -FW_1^\top)gg^\top\rVert_F^2\nonumber\\
    &+\lVert u_{\perp}u_{\perp}^\top(EA_1^\top B_1^\top -FW_1^\top)\rVert_F^2\nonumber\\
    &+\lVert uu^\top(EA_1^\top B_1^\top -FW_1^\top)g_{\perp}g_{\perp}^\top\rVert_F^2\nonumber\\
    \leq&\lVert uu^\top(E vv^\top A_1^\top B_1^\top -FW_1^\top)gg^\top\rVert_F^2 + \lVert uu^\top E v_{\perp}v_{\perp}^\top A_1^\top B_1^\top gg^\top\rVert_F^2\nonumber\\
    &+\lVert u_{\perp}u_{\perp}^\top F A_1^\top B_1^\top\rVert_F^2 + \lVert u_{\perp}u_{\perp}^\top FW_1^\top \rVert_F^2
    +\lVert EA_1^\top B_1^\top g_{\perp}g_{\perp}^\top\rVert_F^2 +\lVert FW_1^\top g_{\perp}g_{\perp}^\top\rVert_F^2\nonumber\\
    \leq&\lVert uu^\top((\Delta Y-W_2B_1A_1) vv^\top A_1^\top B_1^\top -W_2B_1A_1 W_1^\top)gg^\top\rVert_F^2 + \lVert uu^\top E v_{\perp}v_{\perp}^\top A_1^\top B_1^\top gg^\top\rVert_F^2\nonumber\\
    &+ \lVert B_2A_2W_1+B_2A_2B_1A_1\rVert_F^2\nonumber\\
    &+\lVert u_{\perp}u_{\perp}^\top F A_1^\top B_1^\top\rVert_F^2 + \lVert u_{\perp}u_{\perp}^\top FW_1^\top \rVert_F^2
    +\lVert EA_1^\top B_1^\top g_{\perp}g_{\perp}^\top\rVert_F^2 +\lVert FW_1^\top g_{\perp}g_{\perp}^\top\rVert_F^2\nonumber\\
    \leq&\lvert (\sigma-\sigma_{W_2} g^\top B_1A_1 v)g^\top B_1A_1 v-\sigma_{W_2}\sigma_{W_1} g^\top B_1A_1 v\rvert^2 + \lVert uu^\top E v_{\perp}v_{\perp}^\top A_1^\top B_1^\top gg^\top\rVert_F^2\nonumber\\
    &+ \lVert B_2A_2W_1+B_2A_2B_1A_1\rVert_F^2\nonumber\\
    &+\lVert u_{\perp}u_{\perp}^\top F A_1^\top B_1^\top\rVert_F^2 + \lVert u_{\perp}u_{\perp}^\top FW_1^\top \rVert_F^2
    +\lVert EA_1^\top B_1^\top g_{\perp}g_{\perp}^\top\rVert_F^2 +\lVert FW_1^\top g_{\perp}g_{\perp}^\top\rVert_F^2\nonumber\\
    \leq&\biggl \lvert (\sigma-\sigma_{W_2} g^\top B_1A_1 v)g^\top B_1A_1 v-\sigma_{W_2}\sigma_{W_1} g^\top B_1A_1 v\biggl \rvert^2 + \lVert uu^\top F v_{\perp}v_{\perp}^\top A_1^\top B_1^\top gg^\top\rVert_F^2\nonumber\\
    &+ \lVert B_2A_2\rVert_F^2 (\lVert W_1\rVert_F^2+\lVert B_1A_1\rVert_F^2)\nonumber\\
    &+\lVert u_{\perp}u_{\perp}^\top F A_1^\top B_1^\top\rVert_F^2 + \lVert u_{\perp}u_{\perp}^\top FW_1^\top \rVert_F^2
    +\lVert EA_1^\top B_1^\top g_{\perp}g_{\perp}^\top\rVert_F^2 +\lVert FW_1^\top g_{\perp}g_{\perp}^\top\rVert_F^2\,.
\end{align}
Notice since $g^\top B_1A_1 v\leq\frac{(1-\delta_w)\sigma}{4\sigma_{W_2}}$, we can show that 
\begin{align}
    \biggl \lvert (\sigma-\sigma_{W_2} g^\top B_1A_1 v)g^\top B_1A_1 v-\sigma_{W_2}\sigma_{W_1} g^\top B_1A_1 v\biggl \rvert 
    &\leq (\sigma-\sigma_{W_2} g^\top B_1A_1 v)g^\top B_1A_1 v + \sigma_{W_2}\sigma_{W_1} g^\top B_1A_1 v \nonumber\\
    &\leq \biggl(\frac{(3+\delta_w)\sigma}{4}+\sigma_{W_2}\sigma_{W_1}\biggl)g^\top B_1A_1 v\,.
\end{align}
Thus, we derive the following upper bound on $\lVert D_1\rVert$
\begin{align}
    \lVert D_1\rVert \leq \lVert D_1\rVert_F \leq& \biggl(\frac{(3+\delta_w)\sigma}{4}+\sigma_{W_2}\sigma_{W_1}\biggl)g^\top B_1A_1 v\nonumber\\
    &+\lVert uu^\top F v_{\perp}v_{\perp}^\top A_1^\top B_1^\top gg^\top\rVert_F+\lVert B_2A_2\rVert_F (\lVert W_1\rVert_F+\lVert B_1A_1\rVert_F)\nonumber\\
    &+\lVert u_{\perp}u_{\perp}^\top F A_1^\top B_1^\top\rVert_F + \lVert u_{\perp}u_{\perp}^\top FW_1^\top \rVert_F
    +\lVert EA_1^\top B_1^\top g_{\perp}g_{\perp}^\top\rVert_F +\lVert FW_1^\top g_{\perp}g_{\perp}^\top\rVert_F\,.
\end{align}

Notice except for the first term, the rest of the term either depends on $B_2A_2$ or $g_{\perp}g_{\perp}^\top B_1A_1$ or $B_1A_1v_{\perp}v_{\perp}^\top$.
Based on Lemma~\ref{applem:alignment_z1_to_a1b1} and the conditon that $\lVert Z_1(t)\rVert_F^2 \alpha^d \geq \lVert Z_2(t)\rVert_F^2$, we can conclude that 
the rest of the term is extremely small compared with the first term when $\alpha$ is sufficiently small which is formally characterized by the following lemma.
\begin{lem}\label{applem:bound_d2}
Under the following conditions
\begin{align}
    \lVert Z_1(t)\rVert_F^2 \alpha^d &\geq \lVert Z_2(t)\rVert_F^2 \nonumber\\
    \cos^2\biggl (\gamma_1, \frac{w_{1j}}{\lVert w_{1j}\rVert}\biggl)&\geq 1-d_2d_3^{-\frac{3+\delta_w}{5-\delta_w}} \alpha^{\frac{3+\delta_w}{5-\delta_w}}\,,\nonumber\\
    g^\top B_1A_1 v&\leq\frac{(1-\delta_w)\sigma}{4\sigma_{W_2}}
\end{align}
we can show that 
\begin{align}
    &\lVert uu^\top F v_{\perp}v_{\perp}^\top A_1^\top B_1^\top gg^\top\rVert_F+\lVert B_2A_2\rVert_F (\lVert W_1\rVert_F+\lVert B_1A_1\rVert_F)\nonumber\\
    &+\lVert u_{\perp}u_{\perp}^\top F A_1^\top B_1^\top\rVert_F + \lVert u_{\perp}u_{\perp}^\top FW_1^\top \rVert_F
    +\lVert EA_1^\top B_1^\top g_{\perp}g_{\perp}^\top\rVert_F +\lVert FW_1^\top g_{\perp}g_{\perp}^\top\rVert_F\nonumber\\
    &\leq \biggl(\frac{(3+\delta_w)\sigma}{4}+\sigma_{W_2}\sigma_{W_1}\biggl)g^\top B_1A_1 v\,.
\end{align}
\end{lem}
We refer the readers to Appendix~\ref{appprooflem:bound_d2} for the proof of Lemma~\ref{applem:bound_d2}. 
Thus, based on Lemma~\ref{applem:bound_d2} one can conclude that
\begin{align}
    \lVert D_2\rVert_F 
    &\leq \biggl(\frac{(3+\delta_w)\sigma}{2}+2\sigma_{W_2}\sigma_{W_1}\biggl)g^\top B_1A_1 v\,.
\end{align}
Thus, $\int_{\hat T_1}^{\hat T_1+T_1'} \lVert D_2(s)\rVert ds \leq \biggl(\frac{(3+\delta_w)\sigma}{2}+2\sigma_{W_2}\sigma_{W_1}\biggl) \int_{\hat T_1}^{\hat T_1+T_1'} g^\top B_1(s)A_1(s) v ds$.

Therefore, it suffices to show that $\int_{\hat T_1}^{\hat T_1+T_1'} g^\top B_1(s)A_1(s) v ds$ is bounded.
\begin{align}
    \frac{d}{dt} g^\top B_1A_1 v&= g^\top \dot B_1 A_1 v+   g^\top B_1 \dot A_1 v\nonumber\\
    &= g^\top \bigl(\sigma\sigma_{W_2}gv^\top+D_1\bigl)A_1^\top A_1 v+g^\top B_1 B_1^\top \bigl(\sigma\sigma_{W_2}gv^\top+D_1\bigl)v\nonumber\\
    &= \sigma\sigma_{W_2}(g^\top B_1B_1^\top g+v^\top A_1^\top A_1 v) + g^\top D_1 A_1^\top A_1 v+g^\top B_1B_1^\top D_1^\top v \nonumber\\
    &\geq \frac{(1+\delta_w)\sigma\sigma_{W_2}}{2} (g^\top B_1B_1^\top g+v^\top A_1^\top A_1^\top v) \nonumber\\
    &\geq (1+\delta_w)\sigma\sigma_{W_2} g^\top B_1A_1 v
\end{align}
where in the last inequality, we use Cauchy Swartz. One directly has
\begin{align}
    &g^\top A_1(T_1')B_1(T_1') v \geq g^\top A_1(\hat T_1)B_1(\hat T_1) v \geq (1+\delta_w)\sigma\sigma_{W_2} \int_{\hat T_1}^{T_1'+\hat T_1} g^\top B_1(s)A_1(s) v ds \nonumber\\
    \Rightarrow & \int_{\hat T_1}^{T_1'+\hat T_1} g^\top B_1(s)A_1(s) v ds \leq g^\top A_1(T_1')B_1(T_1') v =\frac{(1-\delta_w)\sigma}{4\sigma_{W_2}}\,,
\end{align}
which completes the proof.
\end{proof}
\section{Proof of Theorem~\ref{thm:local_convergence_rank-one}}\label{app:proof_local_convergence}
In this section, we provide proof of Theorem~\ref{thm:local_convergence_rank-one}. 
In Appendix~\ref{app:thm-rank-one}, we have shown that at the end of \textit{alignment} phase, there is sufficient alignment and imbalance, and $g^\top B_1A_1 v$ has reached a constant order which is independent of $\alpha$.
In this section, we will show that how these properties lead to linear convergence of the loss until it reaches a neighourhood around the global minimum. 
We first present a detailed version of Theorem~\ref{thm:local_convergence_rank-one}.
\begin{thm}
Under the same setting as Theorem~\ref{thm:rank-one}, let $T_2\!=\!\frac{c_2\log\bigl(\frac{\sqrt{2L(0)}}{\alpha}\bigl)}{\sigma\sigma_{W_2}}$. Then, there exists constants $c_6, c_7$ such that for $\forall t\in[T_1,T_1\!+\!T_2]$, the following holds
\begin{enumerate}[leftmargin=0.45cm]
    \item Good Alignment of $\boldsymbol{U_{Z_1}^S}(t)$ with $\gamma_1$:
    \begin{align}
        \cos^2{(\boldsymbol{U_{Z_1}^S}(t), \gamma_1)} \!\geq \!\cos^{2c_6}{(\boldsymbol{U_{Z_1}^S}(T_1), \gamma_1)}\,.
    \end{align}
    \item Imbalance Between $\lVert Z_1(t)\rVert$ and $\lVert Z_2(t)\rVert$ Persists:
    \begin{align}
        \frac{\lVert Z_1(t)\rVert}{\lVert Z_2(t)\rVert} \!\geq\!\biggl( \frac{\lVert Z_1(T_1)\rVert}{\lVert Z_2(T_1)\rVert}\biggl)^{c_7}\,.
    \end{align}
    \item Loss Converges Linearly:
    \begin{align}
        L(t)\leq&2\exp\biggl(-\frac{\bigl(1\!-\!\frac{\lVert Z_2(T_1)\rVert^2_F}{\lVert Z_1(T_2)\rVert^2_F}\bigr)(1\!-\!\delta_w)\sigma\sigma_{W_2}(t\!-\!T_1)}{16}\biggl)L(T_1)\nonumber\\
        &\!+\!c_8\bigl(1-\cos^2{(\boldsymbol{U_{Z_1}^S}(T_1), \gamma_1)}\bigl)\,.
    \end{align}
\end{enumerate}
Moreover, by substituting the alignment and imbalance results from Claim~\ref{cor:end_of_alignment} and assuming $\alpha\!\leq\!\alpha^*$, we can simplify convergence rate in~\eqref{eqn:convergence_dependency} as follows:
\begin{align}
    L(t)\leq 2\exp\biggl(\frac{\!-\!(1\!-\!\delta_w)\sigma\sigma_{W_2}(t\!-\!T_1)}{32}\biggl)L(T_1) \!+\! 2\alpha^{c_3}\,.
\end{align}
\end{thm}
\begin{proof}
We start with a decomposition of the loss into \textit{signal} and \textit{noise} part.
\begin{align}
    L \!=\!& \frac{1}{2} \lVert Y_{\mathrm{ft}} \!-\! (W_2\!+\!B_2A_2)(W_1\!+\!B_1A_1)\rVert_F^2 \nonumber\\
    \!=\!& \frac{1}{2} \lVert (uu^\top \!+\! u_{\perp}u_{\perp}^\top)\bigl( Y_{\mathrm{ft}} \!-\! (W_2\!+\!B_2A_2)(W_1\!+\!B_1A_1)\bigl)(vv^\top \!+\! v_{\perp}v_{\perp}^\top)\rVert_F^2 \nonumber\\
    \!=\!& \underbrace{\frac{1}{2}\lVert \sigma uv^\top\!-\!\sigma_{W_2}g^\top B_1A_1v\!-\!\sigma_{W_2} u^\top B_2A_2 g \!-\! uu^\top B_2A_2B_1A_1 vv^\top\rVert_F^2\,,}_{\text{Signal loss: } L_S}
    \!+\! \underbrace{\frac{1}{2}\lVert u_{\perp}u_{\perp}^\top F \!+\! uu^\top F v_{\perp}v_{\perp}^\top\rVert_F^2}_{\text{Noise loss: } L_N}
\end{align}
where $F=W_2B_1A_1+B_2A_2W_1+B_2A_2B_1A_1$. 
Our proof strategy it to show that $L_S$ converges linearly until it reaches a small value depending on $\alpha$ while $L_N$ remains small. 
Moreover, we will show that the alignment and imbalance we achieve in the \textit{alignment} phase remains good in the \textit{local convergence} phase.

\myparagraph{Bounds on several key quantities}
We first assume that for any $T_1\leq t\leq T_2$, there is sufficient alignment between each column of $Z_1$ with $\gamma_1$, and sufficient imbalance between $\lVert Z_1\rVert, \lVert Z_2\rVert$.
Moreover, the norm of $Z_1$ is bounded. 
\begin{align}
    \lVert Z_1\rVert_F \leq d_5, \quad \lVert Z_1\rVert_F^2 \geq \alpha^{-d_6}\lVert Z_2\rVert_F^2, \quad 
    \min_{j\in[r]}\cos^2\biggl (\gamma_1, \frac{w_{1j}}{\lVert w_{1j}\rVert}\biggl)&\geq 1-\alpha^{d_7}\,,
\end{align}
where $d_5, d_6, d_7$ are constants only depending on $W_1, W_2, \Delta Y$ and independent of $\alpha$. We will derive convergence results based on these constants.
In the end of this section, we will provide expressions for $d_5, d_6, d_7$.

\myparagraph{Bound on the noise loss $L_N$} 
\begin{align}
    L_N =& \frac{1}{2}\lVert u_{\perp}u_{\perp}^\top F + uu^\top F v_{\perp}v_{\perp}^\top\rVert_F^2 \nonumber \\
    \leq& \frac{1}{2}\lVert u_{\perp}u_{\perp}^\top F\rVert_F^2 + \frac{1}{2}\lVert uu^\top F v_{\perp}v_{\perp}^\top\rVert_F^2 \nonumber \\
    \leq& \frac{1}{2}\lVert u_{\perp}u_{\perp}^\top W_2 B_1A_1\rVert_F^2 + \frac{1}{2}\lVert u_{\perp}u_{\perp}^\top B_2A_2W_1\rVert_F^2 + \frac{1}{2}\lVert u_{\perp}u_{\perp}^\top B_2A_2B_1A_1\rVert_F^2\nonumber \\
    &+ \frac{1}{2}\lVert uu^\top W_2B_1A_1 v_{\perp}v_{\perp}^\top\rVert_F^2 + \frac{1}{2}\lVert uu^\top B_2A_2W_1 v_{\perp}v_{\perp}^\top\rVert_F^2 + \frac{1}{2}\lVert uu^\top B_2A_2B_1A_1 v_{\perp}v_{\perp}^\top\rVert_F^2 \nonumber \\
    \leq& \frac{1}{2}\lVert u_{\perp}u_{\perp}^\top W_2 B_1A_1\rVert_F^2 + \frac{1}{2}\lVert uu^\top W_2B_1A_1 v_{\perp}v_{\perp}^\top\rVert_F^2 + \biggl(\lVert W_1\rVert^2 + \lVert B_1A_1\rVert_F^2\biggl) \lVert B_2A_2\rVert_F^2\nonumber \\
    \leq& \frac{1}{2}\lVert W_2\rVert^2 \biggl(\lVert g_{\perp}g_{\perp}^\top B_1A_1\rVert_F^2+\lVert B_1A_1v_{\perp}v_{\perp}^\top\rVert_F^2\biggl) + \frac{1}{2}\lVert Z_2\rVert_F^2\biggl(\lVert W_1\rVert^2 + \frac{1}{2}\lVert Z_1\rVert_F^2\biggl) \nonumber\\
    \leq& r^2\lVert W_2\rVert^2 \sqrt{\alpha^{d_7}} d_5^4 + 
    \biggl(\lVert W_2\rVert^2 + \frac{d_5^2}{2}\biggl) \times \frac{1}{2}\alpha^{d_6} d_5^2\,,
\end{align}
%
The above upper bound on $L_N$ demonstrates that the noise loss is small if the initialization scale $\alpha$ is small.

\myparagraph{Convergence on the signal loss $L_S$} 
For the convergence of $L_S$, we study $\dot L_S$ first.
\begin{align}
    \dot L_S &= \sum_{i=1}^2\bigl\langle \frac{L_S}{\partial A_i}, \dot A_i\bigl\rangle + \bigl\langle \frac{L_S}{\partial B_i}, \dot B_i\bigl\rangle \nonumber \\
    &= -\sum_{i=1}^2\bigl\langle \frac{L_S}{\partial A_i}, \frac{L}{\partial A_i}\bigl\rangle + \bigl\langle \frac{L_S}{\partial B_i}, \frac{L}{\partial B_i}\bigl\rangle\nonumber \\
    &= -\sum_{i=1}^2 \lVert\frac{L_S}{\partial A_i}\rVert_F^2 +\lVert\frac{L_S}{\partial B_i}\rVert_F^2 + \langle \frac{L_S}{\partial A_i}, \frac{L_N}{\partial A_i}\bigl\rangle + \bigl\langle \frac{L_S}{\partial B_i}, \frac{L_N}{\partial B_i}\bigl\rangle\nonumber \\
    &\leq -\lVert\frac{L_S}{\partial A_1}\rVert_F^2 -\lVert\frac{L_S}{\partial B_1}\rVert_F^2 - \sum_{i=1}^2 \langle \frac{L_S}{\partial A_i}, \frac{L_N}{\partial A_i}\bigl\rangle + \bigl\langle \frac{L_S}{\partial B_i}, \frac{L_N}{\partial B_i}\bigl\rangle\,.
\end{align}
In the following section, we provide lemmas that provide bounds for $\lVert\frac{L_S}{\partial A_1}\rVert_F^2 \!+\!\lVert\frac{L_S}{\partial B_1}\rVert_F^2$ and $\sum_{i=1}^2 \langle \frac{L_S}{\partial A_i}, \frac{L_N}{\partial A_i}\bigl\rangle + \bigl\langle \frac{L_S}{\partial B_i}, \frac{L_N}{\partial B_i}\bigl\rangle$ separately.
\begin{lem}\label{applem:lower_bound_Ls}
Under the condition that $\lVert Z_1\rVert_F \leq d_5, \quad \lVert Z_1\rVert_F^2 \geq \alpha^{-d_6}\lVert Z_2\rVert_F^2$, one has
\begin{align}
    \lVert\frac{L_S}{\partial A_1}\rVert_F^2+\lVert\frac{L_S}{\partial B_1}\rVert_F^2 \geq \lVert \gamma_1^\top Z_1\rVert_F^2\biggl(\sigma_{W_2}^2 \!-\! \frac{\alpha^{d_6} d_5^2}{2}\sigma_{W_2}\biggl) L_S\,.
\end{align}
\end{lem}
\begin{lem}\label{applem:lower_bound_gammaZ}
At the end of \textit{alignment} phase, one can show that the error satisfies
\begin{align}
    \lVert E(T_1)\rVert_F \leq \biggl(1-\frac{(1-\delta_w)}{8}\biggl)\sigma\,.
\end{align}
Then, since the loss under GF is non-increasing, we have $\lVert E(t)\rVert_F\leq \lVert E(T_1)\rVert_F$.
Based on this, one can show that for all $T_1 \leq t\leq T_2$, 
\begin{align}
    \lVert \gamma_1^\top Z_1\rVert_F^2 \geq\frac{(1-\delta_w)\sigma}{8\sigma_{W_2}}\,.
\end{align}
\end{lem}
\begin{rem}
In Lemma~\ref{applem:lower_bound_Ls}, we the following is implied in the intermediate step
\begin{align}
    \lVert\frac{L_S}{\partial A_1}\rVert_F^2+\lVert\frac{L_S}{\partial B_1}\rVert_F^2 \geq \sigma_{W_2}^2\lVert \gamma_1^\top Z_1\rVert_F^2\biggl(1 \!-\! \frac{\lVert Z_2\rVert_F^2}{\sigma_{W_2}}\biggl) L_S\,.
\end{align}
To bound $\lVert Z_2\rVert_F^2$, we use $\lVert Z_2\rVert_F^2 = \frac{\lVert Z_2\rVert_F^2}{\lVert Z_1\rVert_F^2}\lVert Z_1\rVert_F^2$, and we highlight the role of $\frac{\lVert Z_2\rVert_F^2}{\lVert Z_1\rVert_F^2}$ in the theorem.
\end{rem}
We refer the readers to Appendix~\ref{appprooflem:lower_bound_Ls} and Appendix~\ref{appprooflem:lower_bound_gammaZ} for the proof of Lemma~\ref{applem:lower_bound_Ls} and Lemma~\ref{applem:lower_bound_gammaZ}. 
By combining these results together, we can show that
\begin{align}\label{appeqn:Ls_bound}
    \lVert\frac{L_S}{\partial A_1}\rVert_F^2+\lVert\frac{L_S}{\partial B_1}\rVert_F^2 \geq \frac{(1-\delta_w)\sigma}{8} \biggl(\sigma_{W_2} \!-\! \frac{\alpha^{d_6} d_5^2}{2}\biggl) L_S 
    \geq \frac{(1-\delta_w)\sigma\sigma_{W_2}}{16} L_S\,.
\end{align}
Then, we show that $\sum_{i=1}^2 \langle \frac{L_S}{\partial A_i}, \frac{L_N}{\partial A_i}\bigl\rangle + \bigl\langle \frac{L_S}{\partial B_i}, \frac{L_N}{\partial B_i}\bigl\rangle$ is small.
\begin{lem}\label{applem:LsLn_bound}
One can derive the following upper bound
\begin{align}
    &\biggl\lvert \bigl\langle \frac{\partial L_S}{\partial A_1}, \frac{\partial L_N}{\partial A_1}\bigl\rangle\biggl\rvert \!+\!
    \biggl\lvert \bigl\langle \frac{\partial L_S}{\partial B_1}, \frac{\partial L_N}{\partial B_1}\bigl\rangle\biggl\rvert \!+\!
    \biggl\lvert \bigl\langle \frac{\partial L_S}{\partial A_2}, \frac{\partial L_N}{\partial A_2}\bigl\rangle\biggl\rvert \!+\!
    \biggl\lvert \bigl\langle \frac{\partial L_S}{\partial B_2}, \frac{\partial L_N}{\partial B_2}\bigl\rangle\biggl\rvert \nonumber\\
    \leq &  2\sqrt{2L_S}\bigl(2\lVert W_2\rVert_F^2 + 2\alpha^{d_6}d_5^2\bigl) \times d_5^2 \times \biggl(\lVert W_2\rVert \sqrt{\alpha^{d_7}}d_5^2 + \frac{1}{2}\alpha^{d_6}d_5^2 \times (\lVert W_1\rVert + \frac{d_5^2}{2})\biggl)\nonumber\\
    &+4\sqrt{2L_S} (\lVert W_1\rVert_F^2\!+\!\frac{1}{2}d_5^2)\alpha^{d_6}d_5^2 \biggl(\lVert W_2\rVert \sqrt{\alpha^{d_7}}d_5^2 \!+\! \frac{1}{2}\alpha^{d_6}d_5^2 \times (\lVert W_1\rVert \!+\! \frac{d_5^2}{2})\biggl)\,.
\end{align}
\end{lem}
We refer the readers to Appendix~\ref{appprooflem:LsLn_bound} for the proof of Lemma~\ref{applem:LsLn_bound}. For convenience, we use $f_1$ to denote 
\begin{align}
    f_1 =&  2\bigl(2\lVert W_2\rVert_F^2 + 2\alpha^{d_6}d_5^2\bigl) \times d_5^2 \times \biggl(\lVert W_2\rVert \sqrt{\alpha^{d_7}}d_5^2 + \frac{1}{2}\alpha^{d_6}d_5^2 \times (\lVert W_1\rVert + \frac{d_5^2}{2})\biggl)\nonumber\\
    &+4(\lVert W_1\rVert_F^2\!+\!\frac{1}{2}d_5^2)\alpha^{d_6}d_5^2 \biggl(\lVert W_2\rVert \sqrt{\alpha^{d_7}}d_5^2 \!+\! \frac{1}{2}\alpha^{d_6}d_5^2 \times (\lVert W_1\rVert \!+\! \frac{d_5^2}{2})\biggl)\,.
\end{align}
\begin{rem}
We are particular interested in the dependence of $f_1$ on $\alpha$. One can see that $f_1^2\sim\mathcal{O}(\alpha^{\min(d_7, 2d_6)})$.
\end{rem}
Then, based on ~\eqref{appeqn:Ls_bound} and Lemma~\ref{applem:LsLn_bound}, we can conclude
\begin{align}
    \dot L_S \leq -\frac{(1-\delta_w)\sigma\sigma_{W_2}}{16} L_S+f_1\sqrt{2L_S}\,.
\end{align}
For this ODE system, we can show that 
\begin{align}
    \sqrt{L_S} \leq \exp\biggl(-\frac{(1-\delta_w)\sigma\sigma_{W_2}(t-T_1)}{32}\biggl)\sqrt{L(T_1)} + \frac{32f_1}{(1-\delta_w)\sigma\sigma_{W_2}}\,,
\end{align}
which leads to
\begin{align}
    L_S(t) \leq 2\exp\biggl(-\frac{(1-\delta_w)\sigma\sigma_{W_2}(t-T_1)}{16}\biggl)L(T_1) + 2\biggl(\frac{32f_1}{(1-\delta_w)\sigma\sigma_{W_2}}\biggl)^2\,,
\end{align}
We choose $T_2$ such that
\begin{align}
    &\exp\biggl(-\frac{(1-\delta_w)\sigma\sigma_{W_2}(t-T_1)}{16}\biggl)L(T_1) = \biggl(\frac{32f_1}{(1-\delta_w)\sigma\sigma_{W_2}}\biggl)^2 \nonumber\\
    \iff & T_2=\frac{16}{(1-\delta_w)\sigma\sigma_{W_2}}\log\biggl(\frac{(1-\delta_w)^2\sigma^2\sigma_{W_2}^2L(T_1)}{1024f_1^2}\biggl)
\end{align}

In the remaining of the proof, we show the existence of $d_5, d_6, d_7$. First, for convenience, we use $f_2$ to denote the upper bound on $L_N$.
\begin{align}
    \frac{d}{dt}\lVert Z_1\rVert_F^2 &= 2\langle B_1, (W_2+B_2A_2)^\top E A_1^\top\rangle + 2\langle A_1, B_1^\top (W_2+B_2A_2)^\top E\rangle \nonumber\\
    &\leq 4\lVert A_1\rVert_F \lVert B_1\rVert_F \lVert W_2+B_2A_2\rVert_F \lVert E\rVert_F\nonumber\\
    &\leq 2\lVert Z_1\rVert_F^2 \lVert E\rVert_F (\lVert W_2\rVert_F + \frac{1}{2}d_5^2\alpha^{d_6})\,.
\end{align}
Thus, one can show that
\begin{align}
    \log (\lVert Z_1\rVert_F^2) \leq& \int_{T_1}^{T_1+T_2} (2\lVert W_2\rVert_F + \alpha^{d_6}d_5^2)\lVert E(s)\rVert_F ds\nonumber\\
    \leq& \int_{T_1}^{T_1+T_2} (2\lVert W_2\rVert_F + \alpha^{d_6} d_5^2)\biggl(\exp\biggl(-\frac{(1-\delta_w)\sigma\sigma_{W_2}(t-T_1)}{32}\biggl)\sqrt{L(T_1)} + \frac{32f_1}{(1-\delta_w)\sigma\sigma_{W_2}}\biggl) ds\nonumber\\
    =&\log (\lVert Z_1(T_1)\rVert_F^2) + (2\lVert W_2\rVert_F + \alpha^{d_6}d_5^2) \sqrt{L(T_1)} \frac{1-\frac{32f_1}{(1-\delta_w)\sigma\sigma_{W_2}}}{\frac{(1-\delta_w)\sigma\sigma_{W_2}}{32}} \nonumber\\
    &+ \biggl(\frac{32f_1L(T_1)}{(1-\delta_w)\sigma\sigma_{W_2}}+f_2\biggl) \frac{16}{(1-\delta_w)\sigma\sigma_{W_2}}\log\biggl(\frac{(1-\delta_w)^2\sigma^2\sigma_{W_2}^2L(T_1)}{1024f_1^2}\biggl)\,.
\end{align}
We set $\log(d_5)$ larger or equal to the RHS of the above inequality, denoted by $R_1(\alpha)$. 
Notice when $\alpha$ goes to zero, $R_1$ converges to $\log (\lVert Z_1(T_1)\rVert_F^2)+\frac{64\lVert W_2\rVert}{(1-\delta_w)\sigma\sigma_{W_2}}$.
Moreover, one can show that when $\alpha$ is sufficiently small, $\lVert Z_1(T_1)\rVert_F^2\leq 4g^\top B_1A_1v=(1-\delta_w)\sigma\sigma_{W_2}$.
We set $d_5^*=\exp\biggl(2\log ((1-\delta_w)\sigma\sigma_{W_2})+\frac{64\lVert W_2\rVert}{(1-\delta_w)\sigma\sigma_{W_2}}\biggl)$
Thus, when $\alpha$ is small enough, $R_1(\alpha)\leq \log(d_5^*)$, which verifies $\lVert Z_1(t)\rVert_F\leq d_5^*$ for all $T_1\leq t\leq T_1+T_2$.

For $d_6$, we have
\begin{align}
    \frac{d}{dt}\log\biggl(\frac{\lVert Z_1\rVert_F^2}{\lVert Z_2\rVert_F^2}\biggl) \geq& 2(1-\delta_w)\sigma_{W_2}-2\lVert D_1\rVert - 2\lVert D_2\rVert\nonumber\\
    \geq& 2(1-\delta_w)\sigma_{W_2} - 2\lVert A_2^\top B_2^\top E\rVert - 2\lVert W_2\rVert \lVert F\rVert - 2\lVert EA_1^\top B_1^\top\rVert - 2\lVert W_1\rVert \lVert F\rVert\nonumber\\
    \geq& 2(1-\delta_w)\sigma_{W_2}-(\lVert Z_1\rVert_F^2+\lVert Z_2\rVert_F^2)\rVert E\rVert_F-2(\lVert W_1\rVert+\lVert W_2\rVert)(\lVert \Delta Y\rVert+\lVert E(T_1)\rVert)\nonumber\\
    \geq& 2(1-\delta_w)\sigma_{W_2}-2(\lVert W_1\rVert+\lVert W_2\rVert)\lVert \Delta Y\rVert \nonumber\\
    &-2\bigl((d_5^*)^2(1+\alpha^{d_6})+\lVert W_1\rVert+\lVert W_2\rVert\bigl) \biggl(\exp\biggl(-\frac{(1-\delta_w)\sigma\sigma_{W_2}(t-T_1)}{32}\biggl)\sqrt{L(T_1)} + \frac{32f_1}{(1-\delta_w)\sigma\sigma_{W_2}}\biggl)
\end{align}
Thus, we have
\begin{align}
    \log\biggl(\frac{\lVert Z_1\rVert_F^2}{\lVert Z_2\rVert_F^2}\biggl) \geq& \biggl(2(1-\delta_w)\sigma_{W_2}-2(\lVert W_1\rVert+\lVert W_2\rVert)\Delta Y\rVert\biggl)\frac{16}{(1-\delta_w)\sigma\sigma_{W_2}}\log\biggl(\frac{(1-\delta_w)^2\sigma^2\sigma_{W_2}^2L(T_1)}{1024f_1^2}\biggl)\nonumber\\
    &-2\bigl((d_5^*)^2(1+\alpha^{d_6})+\lVert W_1\rVert+\lVert W_2\rVert\bigl) \frac{32f_1}{(1-\delta_w)\sigma\sigma_{W_2}}\frac{16}{(1-\delta_w)\sigma\sigma_{W_2}}\log\biggl(\frac{(1-\delta_w)^2\sigma^2\sigma_{W_2}^2L(T_1)}{1024f_1^2}\biggl)\nonumber\\
    &-2\bigl((d_5^*)^2(1+\alpha^{d_6})+\lVert W_1\rVert+\lVert W_2\rVert\bigl)\sqrt{L(T_1)}\frac{32f_1}{(1-\delta_w)\sigma\sigma_{W_2}} +  \log\biggl(\frac{\lVert Z_1(T_1)\rVert_F^2}{\lVert Z_2(T_1)\rVert_F^2}\biggl)
\end{align}
We use $R_2(\alpha)$ to denote the RHS of the above inequality. Then, when $\alpha$ goes to zero, the LHS is of the following order 
\begin{align}
    R_2(\alpha) \sim \log\biggl(\frac{1}{\alpha}\biggl)\biggl(\frac{2(1-\delta_w)}{5-\delta_w}+\frac{32\min(2d_6, d_7)\biggl(2(1-\delta_w)\sigma_{W_2}-2(\lVert W_1\rVert+\lVert W_2\rVert)\Delta Y\rVert\biggl)}{(1-\delta_w)\sigma\sigma_{W_2}}\biggl)
\end{align}

For $d_7$, for all $j\in[r]$, for any $\hat\gamma_1\perp \gamma_1$we have
\begin{align}
    \frac{d}{dt}\log\biggl(\frac{\cos(\gamma_1, \frac{w_{1j}}{\lVert w_{1j}\rVert})}{\cos(\hat\gamma_1, \frac{w_{1j}}{\lVert w_{1j}\rVert})}\biggl)\geq&\sigma\sigma_{W_2}-2\lVert D_1\rVert\nonumber\\
    \geq& \sigma\sigma_{W_2}-2\lVert A_2^\top B_2^\top E\rVert - 2\lVert W_2\rVert \lVert F\rVert  \nonumber\\
    \geq& \sigma\sigma_{W_2}-2\alpha^{d_6} (d_5^*)^2 \lVert E \rVert_F -2\lVert W_2\rVert (\lVert\Delta Y\rVert+\lVert E\rVert)\,.
\end{align}
Thus, one can show that
\begin{align}
    &\log\biggl(\frac{\cos(\gamma_1, \frac{w_{1j}}{\lVert w_{1j}\rVert})}{\cos(\hat\gamma_1, \frac{w_{1j}}{\lVert w_{1j}\rVert})}\biggl)\nonumber\\
    \geq&\log\biggl(\frac{\cos(\gamma_1, \frac{w_{1j}(T_1)}{\lVert w_{1j}(T_1)\rVert})}{\cos(\hat\gamma_1, \frac{w_{1j}(T_1)}{\lVert w_{1j}(T_1)\rVert})}\biggl)
    +(\sigma\sigma_{W_2}-2\lVert W_2\rVert \lVert\Delta Y\rVert)T_2\nonumber\\
    &-(2\alpha^{d_6} (d_5^*)^2+2\lVert W_2\rVert) T_2\biggl(\frac{32f_1L(T_1)}{(1-\delta_w)\sigma\sigma_{W_2}}+f_2\biggl) \frac{16}{(1-\delta_w)\sigma\sigma_{W_2}}\log\biggl(\frac{(1-\delta_w)^2\sigma^2\sigma_{W_2}^2L(T_1)}{1024f_1^2}\biggl)\nonumber\\
    &-(2\alpha^{d_6} (d_5^*)^2+2\lVert W_2\rVert)E(T_1)\frac{1-\frac{32f_1}{(1-\delta_w)\sigma\sigma_{W_2}}}{\frac{(1-\delta_w)\sigma\sigma_{W_2}}{32}}\,.
\end{align}
Let the RHS of the above inequality be $R_3(\alpha)$. We can show that when $\alpha$ goes to zero, $R_3$ is of the following order
\begin{align}
    R_3(\alpha) \sim \log(\frac{1}{\alpha})\biggl(\frac{3+\delta_w}{5-\delta_w} -
    \sigma(2\lVert W_2\rVert-\sigma_{W_2})\frac{32\min(2d_6, d_7)}{(1-\delta_w)\sigma\sigma_{W_2}}\biggl)
\end{align}
Following the same argument of providing lower bound on $d_5$, one only needs to ensure
\begin{align}
    \frac{2(1-\delta_w)}{5-\delta_w}+\frac{32\min(2d_6, d_7)\biggl(2(1-\delta_w)\sigma_{W_2}-2(\lVert W_1\rVert+\lVert W_2\rVert)\Delta Y\rVert\biggl)}{(1-\delta_w)\sigma\sigma_{W_2}} \geq \frac{1}{2}d_6\nonumber\\
    \frac{3+\delta_w}{5-\delta_w} -
    \sigma(2\lVert W_2\rVert-\sigma_{W_2})\frac{32\min(2d_6, d_7)}{(1-\delta_w)\sigma\sigma_{W_2}} \geq \frac{1}{2}d_7 \,.
\end{align}
It is obvious that there exists $d_6, d_7$ such that the above inequality holds.
\end{proof}

\section{Proof of Several Lemmas Used in Appendix~\ref{app:proof_alignment} and Appendix~\ref{app:proof_local_convergence}}\label{app:proof_lemma}
\subsection{Proof of Lemma~\ref{applem:grow_norm_imbalance}}\label{appprooflem:grow_norm_imbalance}
\begin{proof}
We start with characterizing each column of $Z_1$ align with $\gamma_1$. Let $W_{1j}=Z_1e_j, \forall j\in[r]$ where $e_j$ is the standard basis. Then, we have
\begin{align}\label{appeqn:dynamics_wij}
    \dot w_{1j} = \sigma\sigma_{W_2}H_1w_{1j}+\hat D_1w_{1j}\,.
\end{align}
Then, based on~\eqref{appeqn:dynamics_wij}, we can characterize the growth of $\lVert w_{1j}\rVert$ as follows
\begin{align}
    \frac{d}{dt} \lVert w_{1j}\rVert^2 &= 2\langle w_{1j}, \dot w_{1j}\rangle\nonumber\\
    &\leq 2(\sigma\sigma_{W_2}+\lVert \hat D_1\rVert) \lVert w_{1j}\rVert^2\nonumber\\
    &=2(\sigma\sigma_{W_2}+\lVert D_1\rVert) \lVert w_{1j}\rVert^2\,, && \text{Apply Lemma~\ref{applem:hermitian_A}}
\end{align}
Moreover, one can use a similar argument to derive the lower bound on the growth of $\lVert w_{1j}\rVert^2$
\begin{align}
    \frac{d}{dt} \lVert w_{1j}\rVert^2 &= 2\langle w_{1j}, \dot w_{1j}\rangle\nonumber\\
    &\geq 2(\sigma\sigma_{W_2}-\lVert \hat D_1\rVert) \lVert w_{1j}\rVert^2\nonumber\\
    &=2(\sigma\sigma_{W_2}-\lVert D_1\rVert) \lVert w_{1j}\rVert^2\,, && \text{Apply Lemma~\ref{applem:hermitian_A}}
\end{align}
Since $\lVert Z_1\rVert_F^2=\sum_{j=1}^r \lVert w_{1j}\rVert^2$, therefore one has
\begin{align}
    \frac{d}{dt} \lVert Z_1\rVert_F^2&=\sum_{j=1}^r \frac{d}{dt} \lVert w_{1j}\rVert^2 \leq 2(\sigma\sigma_{W_2}+\lVert D_1\rVert) \sum_{j=1}^r\lVert w_{1j}\rVert^2=2(\sigma\sigma_{W_2}+\lVert D_1\rVert)\lVert Z_1\rVert_F^2\,,\nonumber\\
    \frac{d}{dt} \lVert Z_1\rVert_F^2&=\sum_{j=1}^r \frac{d}{dt} \lVert w_{1j}\rVert^2 \geq 2(\sigma\sigma_{W_2}-\lVert D_1\rVert) \sum_{j=1}^r\lVert w_{1j}\rVert^2=2(\sigma\sigma_{W_2}-\lVert D_1\rVert)\lVert Z_1\rVert_F^2\,,
\end{align}
The above equations yield the following
\begin{align}
    2(\sigma\sigma_{W_2}-\lVert D_1\rVert) \leq \frac{d}{dt}\log(\lVert Z_1\rVert_F^2) \leq 2(\sigma\sigma_{W_2}+\lVert D_1\rVert)\,.
\end{align}
A similar results hold for $\lVert Z_2\rVert_F^2$ respectively. 
Thus, one can show the following characterization of the growth of $\frac{\lVert Z_1\rVert_F^2}{\lVert Z_2\rVert_F^2}$ as follows
\begin{align}
    \frac{d}{dt}\log\biggl(\frac{\lVert Z_1\rVert_F^2}{\lVert Z_2\rVert_F^2}\biggl) \geq 2(\sigma\sigma_{W_2}-\lVert D_1\rVert-\sigma\sigma_{W_1}-\lVert D_2\rVert)=2\bigl((1-\delta_w)\sigma\sigma_{W_2}-\lVert D_1\rVert-\lVert D_2\rVert\bigl)\,.
\end{align}
Now, we move on to the study of the angular dynamics of $w_{1j}$.
\begin{align}
    \frac{d}{dt}\frac{w_{1j}}{\lVert w_{1j}\rVert} &= \frac{\dot w_{1j}}{\lVert w_{1j}\rVert}-\frac{w_{1j}}{\lVert w_{1j}\rVert^2}\cdot \frac{\dot w_{1j}^\top w_{1j}}{\lVert w_{1j}\rVert}\nonumber\\
    &=\sigma\sigma_{W_2}\biggl(I-\frac{w_{1j}w_{1j}^\top}{\lVert w_{1j}\rVert^2}\biggl)\frac{H_1 w_{1j}}{\lVert w_{1j}\rVert}+\biggl(I-\frac{w_{1j}w_{1j}^\top}{\lVert w_{1j}\rVert^2}\biggl)\frac{\hat D_1w_{1j}}{\lVert w_{1j}\rVert}\,.
\end{align}
For $\gamma_1$, we have
\begin{align}\label{appeqn:angle_1}
    \frac{d}{dt}\cos\biggl(\gamma_1, \frac{w_{1j}}{\lVert w_{1j}\rVert}\biggl) =& \biggl\langle \gamma_1, \sigma\sigma_{W_2}\biggl(I-\frac{w_{1j}w_{1j}^\top}{\lVert w_{1j}\rVert^2}\biggl)\frac{H_1 w_{1j}}{\lVert w_{1j}\rVert}+\biggl(I-\frac{w_{1j}w_{1j}^\top}{\lVert w_{1j}\rVert^2}\biggl)\frac{\hat D_1w_{1j}}{\lVert w_{1j}\rVert}\biggl\rangle \nonumber\\
    =&\sigma\sigma_{W_2}\gamma_1^\top \biggl(I-\frac{w_{1j}w_{1j}^\top}{\lVert w_{1j}\rVert^2}\biggl) \biggl(\frac{(\gamma_1\gamma_1^\top-\bar\gamma_1\bar\gamma_1^\top)w_{1j}}{\lVert w_{1j}\rVert}\biggl) + \frac{\gamma_1^\top\hat D_1 w_{1j}}{\lVert w_{1j}\rVert} - \cos\biggl(\gamma_1, \frac{w_{1j}}{\lVert w_{1j}\rVert}\biggl)\frac{w_{1j}^\top\hat D_1 w_{1j}}{\lVert w_{1j}\rVert^2}\nonumber\\
    =&\sigma\sigma_{W_2}\cos\biggl(\gamma_1, \frac{w_{1j}}{\lVert w_{1j}\rVert}\biggl)\biggl[1-\cos^2\biggl(\gamma_1, \frac{w_{1j}}{\lVert w_{1j}\rVert}\biggl)\biggl] + \sigma\sigma_{W_2}\cos\biggl(\gamma_1, \frac{w_{1j}}{\lVert w_{1j}\rVert}\biggl)\cos^2\biggl(\bar\gamma_1, \frac{w_{1j}}{\lVert w_{1j}\rVert}\biggl)\nonumber\\
    &+\frac{\gamma_1^\top\hat D_1 w_{1j}}{\lVert w_{1j}\rVert} - \cos\biggl(\gamma_1, \frac{w_{1j}}{\lVert w_{1j}\rVert}\biggl)\frac{w_{1j}^\top\hat D_1 w_{1j}}{\lVert w_{1j}\rVert^2}
\end{align}
For $\bar\gamma_1$, we have
\begin{align}\label{appeqn:angle_2}
    \frac{d}{dt}\cos\biggl(\bar\gamma_1, \frac{w_{1j}}{\lVert w_{1j}\rVert}\biggl) =& \biggl\langle \bar\gamma_1, \sigma\sigma_{W_2}\biggl(I-\frac{w_{1j}w_{1j}^\top}{\lVert w_{1j}\rVert^2}\biggl)\frac{H_1 w_{1j}}{\lVert w_{1j}\rVert}+\biggl(I-\frac{w_{1j}w_{1j}^\top}{\lVert w_{1j}\rVert^2}\biggl)\frac{\hat D_1w_{1j}}{\lVert w_{1j}\rVert}\biggl\rangle \nonumber\\
    =&\sigma\sigma_{W_2}\bar\gamma_1^\top \biggl(I-\frac{w_{1j}w_{1j}^\top}{\lVert w_{1j}\rVert^2}\biggl) \biggl(\frac{(\gamma_1\gamma_1^\top-\bar\gamma_1\bar\gamma_1^\top)w_{1j}}{\lVert w_{1j}\rVert}\biggl) + \frac{\bar\gamma_1^\top\hat D_1 w_{1j}}{\lVert w_{1j}\rVert} - \cos\biggl(\bar\gamma_1, \frac{w_{1j}}{\lVert w_{1j}\rVert}\biggl)\frac{w_{1j}^\top\hat D_1 w_{1j}}{\lVert w_{1j}\rVert^2}\nonumber\\
    =&-\sigma\sigma_{W_2}\cos\biggl(\bar\gamma_1, \frac{w_{1j}}{\lVert w_{1j}\rVert}\biggl)\biggl[1-\cos^2\biggl(\bar\gamma_1, \frac{w_{1j}}{\lVert w_{1j}\rVert}\biggl)\biggl] - \sigma\sigma_{W_2}\cos\biggl(\bar\gamma_1, \frac{w_{1j}}{\lVert w_{1j}\rVert}\biggl)\cos^2\biggl(\gamma_1, \frac{w_{1j}}{\lVert w_{1j}\rVert}\biggl)\nonumber\\
    &+\frac{\bar\gamma_1^\top\hat D_1 w_{1j}}{\lVert w_{1j}\rVert} - \cos\biggl(\bar\gamma_1, \frac{w_{1j}}{\lVert w_{1j}\rVert}\biggl)\frac{w_{1j}^\top\hat D_1 w_{1j}}{\lVert w_{1j}\rVert^2}
\end{align}
For any unit vector $\tilde \gamma_1 \perp \gamma_1, \bar\gamma_1$, we have
\begin{align}\label{appeqn:angle_3}
    \frac{d}{dt}\cos\biggl(\tilde\gamma_1, \frac{w_{1j}}{\lVert w_{1j}\rVert}\biggl) =& \biggl\langle \tilde\gamma_1, \sigma\sigma_{W_2}\biggl(I-\frac{w_{1j}w_{1j}^\top}{\lVert w_{1j}\rVert^2}\biggl)\frac{H_1 w_{1j}}{\lVert w_{1j}\rVert}+\biggl(I-\frac{w_{1j}w_{1j}^\top}{\lVert w_{1j}\rVert^2}\biggl)\frac{\hat D_1w_{1j}}{\lVert w_{1j}\rVert}\biggl\rangle \nonumber\\
    =&\sigma\sigma_{W_2}\tilde\gamma_1^\top \biggl(I-\frac{w_{1j}w_{1j}^\top}{\lVert w_{1j}\rVert^2}\biggl) \biggl(\frac{(\gamma_1\gamma_1^\top-\bar\gamma_1\bar\gamma_1^\top)w_{1j}}{\lVert w_{1j}\rVert}\biggl) + \frac{\tilde\gamma_1^\top\hat D_1 w_{1j}}{\lVert w_{1j}\rVert} - \cos\biggl(\tilde\gamma_1, \frac{w_{1j}}{\lVert w_{1j}\rVert}\biggl)\frac{w_{1j}^\top\hat D_1 w_{1j}}{\lVert w_{1j}\rVert^2}\nonumber\\
    =&-\sigma\sigma_{W_2}\cos\biggl(\tilde\gamma_1, \frac{w_{1j}}{\lVert w_{1j}\rVert}\biggl)\biggl[\cos^2\biggl(\gamma_1, \frac{w_{1j}}{\lVert w_{1j}\rVert}\biggl)-\cos^2\biggl(\bar\gamma_1, \frac{w_{1j}}{\lVert w_{1j}\rVert}\biggl)\biggl] \nonumber\\
    &+\frac{\tilde\gamma_1^\top\hat D_1 w_{1j}}{\lVert w_{1j}\rVert} - \cos\biggl(\tilde\gamma_1, \frac{w_{1j}}{\lVert w_{1j}\rVert}\biggl)\frac{w_{1j}^\top\hat D_1 w_{1j}}{\lVert w_{1j}\rVert^2}
\end{align}
Based on ~\eqref{appeqn:angle_1}, \eqref{appeqn:angle_2} and \eqref{appeqn:angle_3}, we can show 
\begin{align}
    \frac{d}{dt}\log\biggl(\frac{\cos(\gamma_1, \frac{w_{1j}}{\lVert w_{1j}\rVert})}{\cos(\bar\gamma_1, \frac{w_{1j}}{\lVert w_{1j}\rVert})}\biggl) \geq 2\sigma\sigma_{W_2}-2\lVert D_1\rVert\,, \quad \frac{d}{dt}\log\biggl(\frac{\cos(\gamma_1, \frac{w_{1j}}{\lVert w_{1j}\rVert})}{\cos(\tilde\gamma_1, \frac{w_{1j}}{\lVert w_{1j}\rVert})}\biggl) \geq \sigma\sigma_{W_2}-2\lVert D_1\rVert\,.
\end{align}
Thus, for any unit vector $\hat\gamma_1$ that is orthogonal to $\gamma_1$, we have
\begin{align}
    \frac{d}{dt}\log\biggl(\frac{\cos(\gamma_1, \frac{w_{1j}}{\lVert w_{1j}\rVert})}{\cos(\hat\gamma_1, \frac{w_{1j}}{\lVert w_{1j}\rVert})}\biggl) \geq \sigma\sigma_{W_2}-2\lVert D_1\rVert\,.
\end{align}
\end{proof}

\subsection{Proof of Lemma~\ref{applem:grow_norm_imbalance2}}\label{appprooflem:grow_norm_imbalance2}
\begin{proof}
Since for all $0\leq t\leq T_1(\delta_1, \delta_2)$, we have $\lVert D_1(T_1(\delta_1, \delta_2))\rVert\leq\delta_1, \lVert D_2(T_1(\delta_1, \delta_2))\rVert\leq\delta_2$, 
then one can further derive the following based on Lemma~\ref{applem:grow_norm_imbalance}
\begin{align}
    \frac{d}{dt}\log(\lVert w_{1j}\rVert^2)&\leq 2(\sigma\sigma_{W_2}+\delta_1)\,,\quad \frac{d}{dt}\log(\lVert Z_1\rVert_F^2)\leq 2(\sigma\sigma_{W_2}+\delta_1)\nonumber\\
    \frac{d}{dt}\log\biggl(\frac{\lVert Z_1\rVert_F^2}{\lVert Z_2\rVert_F^2}\biggl) &\geq 2((1-\delta_w)\sigma\sigma_{W_2}-\delta_1-\delta_2)>0\,,\nonumber\\
    \frac{d}{dt}\log\biggl(\frac{\cos(\gamma_1, \frac{w_{1j}}{\lVert w_{1j}\rVert})}{\cos(\hat\gamma_1, \frac{w_{1j}}{\lVert w_{1j}\rVert})}\biggl)&\geq\sigma\sigma_{W_2}-2\delta_1>0\,.
\end{align}
Therefore, we can show that for all $0\leq t\leq T_1(\delta_1, \delta_2)$
\begin{align}
    \lVert w_{1j}(t)\rVert^2 &\leq \exp\biggl(2(\sigma\sigma_{W_2}+\delta_1)t\biggl) \lVert w_{1j}(0)\rVert^2 \nonumber\\
    \lVert Z_1(t)\rVert_F^2 &\leq \exp\biggl(2(\sigma\sigma_{W_2}+\delta_1)t\biggl) \lVert Z_1(0)\rVert_F^2\nonumber\\
    \frac{\lVert Z_1(t)\rVert_F^2}{\lVert Z_2(t)\rVert_F^2} &\geq \exp\biggl (2\bigl ((1-\delta_w)\sigma\sigma_{W_2}-\delta_1-\delta_2\bigl)t \biggl)\frac{\lVert Z_1(0)\rVert_F^2}{\lVert Z_2(0)\rVert_F^2}\nonumber\\
    \frac{\cos\biggl (\gamma_1, \frac{w_{1j}\bigl(t\bigl)}{\lVert w_{1j}(t)\rVert}\biggl)}{\cos\biggl(\hat\gamma_1, \frac{w_{1j}(t)}{\lVert w_{1j}(t)\rVert}\biggl)} &\geq \exp\bigl((\sigma\sigma_{W_2}-2\delta_1)t\bigl)
    \frac{\cos\biggl (\gamma_1, \frac{w_{1j}(0)}{\lVert w_{1j}(0)\rVert}\biggl)}{\cos\biggl(\hat\gamma_1, \frac{w_{1j}(0)}{\lVert w_{1j}(0)\rVert}\biggl)}
\end{align}
We select $\hat \gamma_{1,2}, \hat \gamma_{1,3}, \cdots, \hat \gamma_{1,n+h}$ that forms an orthogonal basis for $\R^{n+h}$, since the following holds
\begin{align}
    \cos^2\biggl (\gamma_1, \frac{w_{1j}}{\lVert w_{1j}\rVert}\biggl)+\sum_{i=2}^{m+h} \cos^2\biggl (\gamma_{1, i}, \frac{w_{1j}}{\lVert w_{1j}\rVert}\biggl) = 1
\end{align}
One can lower bound $\cos\biggl (\gamma_1, \frac{w_{1j}(t)}{\lVert w_{1j}(t)\rVert}\biggl)$ as follows
\begin{align}
    1 &= \cos^2\biggl (\gamma_1, \frac{w_{1j}}{\lVert w_{1j}\rVert}\biggl)+\sum_{i=2}^{m+h} \cos^2\biggl (\gamma_{1, i}, \frac{w_{1j}}{\lVert w_{1j}\rVert}\biggl)\nonumber\\
    &\leq \cos^2\biggl (\gamma_1, \frac{w_{1j}}{\lVert w_{1j}\rVert}\biggl) + (m+h-1)\cos^2\biggl (\gamma_1, \frac{w_{1j}}{\lVert w_{1j}\rVert}\biggl) \exp\bigl(-2(\sigma\sigma_{W_2}-2\delta_1)t\bigl) c_{\max}\nonumber\\
    &= \cos^2\biggl (\gamma_1, \frac{w_{1j}}{\lVert w_{1j}\rVert}\biggl)\times \biggl( 1 + c_{\max}(m+h-1)\exp\bigl(-2(\sigma\sigma_{W_2}-2\delta_1)t\bigl)\biggl)\nonumber\\
    &\iff \cos^2\biggl (\gamma_1, \frac{w_{1j}}{\lVert w_{1j}\rVert}\biggl) \geq \frac{1}{1+d_2\exp\bigl(-2(\sigma\sigma_{W_2}-2\delta_1)t\bigl)}
\end{align}
where in the above equation, $\frac{w_{1j}}{\lVert w_{1j}\rVert}$ is its value evaluated at $t$. We omit the time dependency for simplicity.
\end{proof}

\subsection{Proof of Lemma~\ref{applem:grow_norm_imbalance3}}\label{appprooflem:grow_norm_imbalance3}
\begin{proof}
Under the assumption that $\delta_1\!=\!\delta_2 \!=\!\delta\!:=\! \alpha$, we first derive an upper bound on $\lVert D_1\rVert$.
\begin{align}
    \lVert D_1\rVert =& \lVert A_2^\top B_2^\top E- W_2^\top(W_2B_1A_1+B_2A_2W_1+B_2A_2B_1A_1)\rVert \nonumber\\
    \leq& \lVert A_2^\top B_2^\top\rVert \cdot \lVert E\rVert+\lVert W_2\rVert \cdot\bigl(\lVert W_2\rVert\cdot \lVert B_2A_2\rVert+\lVert W_1\rVert\cdot \lVert B_1A_1\rVert+\lVert B_1A_1\rVert\cdot \lVert B_2A_2\rVert \bigl)\nonumber\\
    \leq& (\lVert E\rVert+\lVert W_2\rVert^2) \lVert B_2A_2\rVert+\lVert W_1\rVert\cdot \lVert W_2\rVert\cdot \lVert B_1A_1\rVert + \lVert W_2\rVert\cdot \lVert B_1A_1\rVert\cdot\lVert B_2A_2\rVert\nonumber \\
    \leq& (\lVert E\rVert+\lVert W_2\rVert^2) \lVert Z_2\rVert_F^2+\lVert W_1\rVert\cdot \lVert W_2\rVert\cdot \lVert Z_1\rVert_F^2 + \frac{\lVert W_2\rVert}{4}\cdot \lVert Z_1\rVert_F^2\cdot\lVert Z_2\rVert_F^2 && \text{Lemma~\ref{applem:2-fro}} \nonumber \\
    \leq& \alpha^2 z_{\max}^2(\lVert E\rVert\!+\!\lVert W_2\rVert^2)\exp(2(\sigma\sigma_{W_1}\!+\!\delta_2)t) + \alpha^2 z_{\max}^2\lVert W_1\rVert\cdot \lVert W_2\rVert \exp(2(\sigma\sigma_{W_2}\!+\!\delta_1)t) \nonumber \\
    &\!+\!  \alpha^4 z_{\max}^4\lVert W_2\rVert \exp(2((1+\delta_w)\sigma\sigma_{W_2}\!+\!\delta_1\!+\!\delta_2)t) && \text{Lemma~\ref{applem:grow_norm_imbalance2}} \nonumber\\
    \leq& \alpha^2 z_{\max}^2(\lVert E\rVert\!+\!\lVert W_2\rVert^2+\lVert W_1\rVert\cdot \lVert W_2\rVert)\exp(2(\sigma\sigma_{W_2}\!+\!\delta)t) 
    \!+\!  \alpha^4 z_{\max}^4\lVert W_2\rVert \exp(4(\sigma\sigma_{W_2}\!+\delta)t)\,.
\end{align}
where in the last inequality, we use the property that $\sigma_{W_2}>\sigma_{W_1}$. Similarly, we can show that 
\begin{align}
    \lVert D_2\rVert\leq \alpha^2 z_{\max}^2(\lVert E\rVert\!+\!\lVert W_1\rVert^2+\lVert W_1\rVert\cdot \lVert W_2\rVert)\exp(2(\sigma\sigma_{W_2}\!+\!\delta)t) 
    \!+\!  \alpha^4 z_{\max}^4\lVert W_1\rVert \exp(4(\sigma\sigma_{W_2}\!+\delta)t)\,.
\end{align}
Furthermore, one can derive the following union upper bound on $\lVert D_1\rVert, \lVert D_2\rVert$
\begin{align}\label{appeqn:union_ubb_d1d2}
    \max(\lVert D_1\rVert, \lVert D_2\rVert)\leq& \alpha^2 z_{\max}^2(\lVert E\rVert\!+\!\lVert W_1\rVert^2+\lVert W_1\rVert\lVert W_2\rVert+\!\lVert W_2\rVert^2)\exp(2(\sigma\sigma_{W_2}\!+\!\delta)t) \nonumber\\
    &\!+\!  \alpha^4 z_{\max}^4(\lVert W_1\rVert\!+\!\lVert W_2\rVert) \exp(4(\sigma\sigma_{W_2}\!+\delta)t)\,.
\end{align}

Then, under the assumption that $\alpha\leq \frac{(1-\delta_w)\sigma\sigma_{W_2}}{4}$, we will derive a lower bound on $T_1(\delta, \delta)$ by studying the following inequality 
\begin{align}
    \delta\geq& \alpha^2 z_{\max}^2(\lVert E\rVert\!+\!\lVert W_1\rVert^2+\lVert W_1\rVert\lVert W_2\rVert+\!\lVert W_2\rVert^2)\exp(2(\sigma\sigma_{W_2}\!+\!\delta)t) \nonumber\\
    &\!+\!  \alpha^4 z_{\max}^4(\lVert W_1\rVert\!+\!\lVert W_2\rVert) \exp(4(\sigma\sigma_{W_2}\!+\delta)t)
\end{align}
Since we assume $\delta=\alpha\leq 1$, then both terms in the RHS are smaller than 1, therefore, we can further lower bound $T(\delta, \delta)$ as 
\begin{align}
    \alpha\geq& \alpha^2 z_{\max}^2(\lVert E\rVert\!+\!\lVert W_1\rVert^2\!+\!\lVert W_1\rVert \lVert W_2\rVert+\!\lVert W_2\rVert^2)\exp\biggl (\frac{t(5-\delta_w)\sigma\sigma_{W_2}}{2}\biggl)\nonumber\\
    &\!+\!\alpha^2 z_{\max}^2 \sqrt{\lVert W_1\rVert+\lVert W_2\rVert}\exp\biggl (\frac{t(5-\delta_w)\sigma\sigma_{W_2}}{2}\biggl)\nonumber\\
    =& \alpha^2 z_{\max}^2\bigl(\lVert E\rVert\!+\!\lVert W_1\rVert^2\!+\!\lVert W_1\rVert \lVert W_2\rVert+\!\lVert W_2\rVert^2\! \!+\!\sqrt{\lVert W_1\rVert+\lVert W_2\rVert}\bigl)\exp\biggl (\frac{t(5-\delta_w)\sigma\sigma_{W_2}}{2}\biggl)\nonumber\\
    \iff& \exp\biggl (\frac{t(5-\delta_w)\sigma\sigma_{W_2}}{2}\biggl)=\frac{1}{\alpha z_{\max}^2\bigl(\lVert E\rVert\!+\!\lVert W_1\rVert^2\!+\!\lVert W_1\rVert \lVert W_2\rVert+\!\lVert W_2\rVert^2\! \!+\!\sqrt{\lVert W_1\rVert+\lVert W_2\rVert}\bigl)}\nonumber\\
    \iff& t=\frac{2}{(5-\delta_w)\sigma\sigma_{W_2}}\log\biggl(\frac{1}{\alpha z_{\max}^2\bigl(\lVert E\rVert\!+\!\lVert W_1\rVert^2\!+\!\lVert W_1\rVert \lVert W_2\rVert+\!\lVert W_2\rVert^2\! \!+\!\sqrt{\lVert W_1\rVert+\lVert W_2\rVert}\bigl)}\biggl)\,.
\end{align}
For convenience, we define $d_3=\frac{1}{z_{\max}^2\bigl(\lVert E\rVert\!+\!\lVert W_1\rVert^2\!+\!\lVert W_1\rVert \lVert W_2\rVert+\!\lVert W_2\rVert^2\! \!+\!\sqrt{\lVert W_1\rVert+\lVert W_2\rVert}\bigl)}$.
Thus, we set $T_1(\alpha,\alpha)=\frac{2}{(5-\delta_w)\sigma\sigma_{W_2}}\log\biggl(\frac{1}{\alpha z_{\max}^2d_3}\biggl)$.
Furthermore, based on Lemma~\ref{applem:grow_norm_imbalance2}, we can further characterize the following properties
\begin{align}
    \frac{\lVert Z_1(t)\rVert_F^2}{\lVert Z_2(t)\rVert_F^2} &\geq d_1\exp\biggl (2\bigl ((1-\delta_w)\sigma\sigma_{W_2}-\delta_1-\delta_2\bigl)t \biggl)\,,\nonumber\\
    &= d_1\exp\bigl ((1-\delta_w)t \bigl)\nonumber \\
    &= d_1 d_3^{\frac{2(1-\delta_w)}{5-\delta_w}} \alpha^{-\frac{2(1-\delta_w)}{5-\delta_w}}\,,
\end{align}
Moreover, the alignment can be described as
\begin{align}
    \cos^2\biggl (\gamma_1, \frac{w_{1j}}{\lVert w_{1j}\rVert}\biggl) &\geq \frac{1}{1+d_2\exp\bigl(-2(\sigma\sigma_{W_2}-2\delta_1)t\bigl)} \nonumber \\
    &\geq \frac{1}{1+d_2\exp\bigl(-2(\sigma\sigma_{W_2}-2\delta_1)t\bigl)}\nonumber \\
    &\geq \frac{1}{1+d_2\exp\bigl(-(1-\delta_w)\sigma\sigma_{W_2}t\bigl)}\nonumber \\
    &\geq \frac{1}{1+d_2d_3^{-\frac{2(1-\delta_w)}{5-\delta_w}} \alpha^{\frac{2(1-\delta_w)}{5-\delta_w}}}\nonumber \\
    &\geq 1-d_2d_3^{-\frac{2(1-\delta_w)}{5-\delta_w}} \alpha^{\frac{2(1-\delta_w)}{5-\delta_w}}\,.
\end{align}
Finally, we show that until $T_1(\alpha,\alpha)$, the norm of LoRA weights stay small. In Lemma~\ref{applem:grow_norm_imbalance2}, we have shown that
\begin{align}
    \lVert Z_1(t)\rVert_F^2 &\leq \exp\biggl(2(\sigma\sigma_{W_2}+\delta_1)t\biggl) \lVert Z_1(0)\rVert_F^2 \nonumber \\
    &\leq \alpha^2 z_{\max}^2\exp\biggl (\frac{t(5-\delta_w)\sigma\sigma_{W_2}}{2}\biggl)\nonumber \\
    &=\frac{\alpha}{d_3}\,.
\end{align}
A similar argument holds for $Z_2$ as well.
\end{proof}

\subsection{Proof of Lemma~\ref{applem:ab}}\label{appprooflem:ab}
\begin{proof}
\begin{align}
    \frac{d}{dt} g^\top B_1A_1 v&= g^\top \dot B_1 A_1 v+   g^\top B_1 \dot A_1 v\nonumber\\
    &= g^\top \bigl(\sigma\sigma_{W_2}gv^\top+D_1\bigl)A_1^\top A_1 v+g^\top B_1B_1^\top \bigl(\sigma\sigma_{W_2}gv^\top+D_1\bigl)v\nonumber\\
    &= \sigma\sigma_{W_2}(g^\top B_1B_1^\top g+v^\top A_1^\top A_1 v) + g^\top D_1 A_1^\top A_1 v+g^\top B_1B_1^\top D_1^\top v \nonumber\\
    &\geq (\sigma\sigma_{W_2}-\lVert D_1\rVert) (g^\top B_1B_1^\top g+v^\top A_1^\top A_1 v) \nonumber\\
    &\geq (\sigma\sigma_{W_2}-\lVert D_1\rVert) \lVert Z_1\rVert_F^2 \min_{j\in [r]}\cos^2\biggl (\gamma_1, \frac{w_{1j}}{\lVert w_{1j}\rVert}\biggl)\,. && \text{Apply Lemma~\ref{applem:column_fro}}
\end{align}
\end{proof}
\subsection{Proof of Lemma~\ref{applem:ab_vs_z}}\label{appprooflem:ab_vs_z}
\begin{proof}
We first introduce the \textit{imbalance} quantity: $A_1^\top A_1- B_1 B_1^\top$. This quantity is constant when LoRA weights are trained via GF.
To highlight its dependcy on the $\alpha$, we let $A_1 A_1^\top-B_1^\top B_1:= \alpha^2 \Lambda_1$ where $\Lambda_1$ is independent of $\alpha$ and purely determined at initialization. Then, one can first show that
\begin{align}
    \lVert B_1A_1\rVert \leq \lVert B_1A_1\rVert_F\leq \lVert A_1\rVert_F \lVert B_1\rVert_F \leq \frac{1}{2}\lVert Z_1\rVert_F^2\,.
\end{align}
On the other hand, 
\begin{align}
    & A_1 A_1^\top-B_1^\top B_1 = \alpha^2 \Lambda_1\nonumber\\
    \Rightarrow & B_1 A_1 A_1^\top B_1^\top = B_1 B_1^\top B_1 B_1^\top  + \alpha^2  B_1 \Lambda_1 B_1^\top\nonumber\\
    \Rightarrow & \lVert B_1 A_1 A_1^\top B_1^\top \rVert \geq \lVert B_1 B_1^\top B_1 B_1^\top\rVert - \alpha^2\lVert B_1 \Lambda_1 B_1^\top\rVert\nonumber\\
    \iff & \lVert B_1A_1\rVert^2 \geq \lVert B_1\rVert^4-\alpha^2 \lVert \Lambda_1\rVert \lVert B_1\rVert^2 \nonumber\\
    \Rightarrow & \lVert B_1A_1\rVert^2 \geq \frac{1}{r^2}\lVert B_1\rVert_F^4-\alpha^2 \lVert \Lambda_1\rVert \lVert B_1\rVert_F^2
\end{align}
Similarly, one can show $\lVert B_1A_1\rVert^2 \geq \frac{1}{r^2}\lVert A_1\rVert_F^4-\alpha^2 \lVert \Lambda_1\rVert \lVert A_1\rVert_F^2$.
Combine these results together, one has
\begin{align}
    2\lVert B_1A_1\rVert^2 &\geq \frac{1}{r^2} (\lVert A_1\rVert_F^4+\lVert B_1\rVert_F^4) - \alpha^2 \lVert \Lambda_1\rVert \lVert Z_1\rVert_F^2\nonumber\\
    &\geq \frac{1}{2r^2} \lVert Z_1\rVert_F^4 - \alpha^2 \lVert \Lambda_1\rVert \lVert Z_1\rVert_F^2\,.
\end{align}
By solving the above inequality, one has
\begin{align}
    \lVert Z_1\rVert_F^2 &\leq \alpha^2 r^2 \lVert \Lambda_1\rVert+r^2\sqrt{\alpha^4 \lVert \Lambda_1\rVert^2 + \frac{4}{r^2}\lVert B_1A_1\rVert^2} \nonumber\\
    &\leq \alpha^2 r^2 \lVert \Lambda_1\rVert + \alpha^2 r^2 \lVert \Lambda_1\rVert + 2r \lVert B_1A_1\rVert && \text{we use $\sqrt{a+b}\leq \sqrt{a}+\sqrt{b}$.}
\end{align}
Thus, we have $\lVert B_1A_1\rVert \geq \frac{\lVert Z_1\rVert_F^2-2\alpha^2 r^2 \lVert \Lambda_1\rVert}{2r}$\,.
\end{proof}

\subsection{Proof of Lemma~\ref{applem:alignment_z1_to_a1b1}}\label{appprooflem:alignment_z1_to_a1b1}
\begin{proof}
We prove $\lVert g_{\perp}g_{\perp}^\top B_1A_1\rVert^2$ here. The same analysis can be applied to derive upper bound on $\lVert B_1A_1 v_{\perp}v_{\perp}^\top\rVert^2$.
Notice $\tilde \gamma_1=\begin{pmatrix}
    g_{\perp} \\
    0_{n\times (h-1)}
\end{pmatrix}$ is orthogonal to $\gamma_1$. Therefore, one can show that for each index $j$, we have
\begin{align}
    \cos^2(\gamma_1, \frac{w_{1j}}{\lVert w_{1j}\rVert}) + \sum_{i=1}^{h-1}\cos^2(\tilde \gamma_{1i}, \frac{w_{1j}}{\lVert w_{1j}\rVert}) \leq 1\,,
\end{align}
where we use $\tilde \gamma_{1i}$ to denote the $i$-th column of $\tilde \gamma_1$. 
This implies $\sum_{i=1}^{h-1}\cos^2(\tilde \gamma_{1i}, \frac{w_{1j}}{\lVert w_{1j}\rVert}) \leq \beta_2 \alpha^{d}$. 
Now, we consider $\tilde\gamma_1\tilde\gamma_1^\top Z_1$
\begin{align}
    \lVert \tilde \gamma_1\tilde\gamma_1^\top Z_1\rVert_F^2 &= \sum_{j=1}^r\lVert \tilde \gamma_1\tilde\gamma_1^\top w_{1j}\rVert^2 \nonumber\\
    &\leq \sum_{j=1}^r \sum_{i=1}^{h-1}\cos^2\biggl (\tilde \gamma_{1i}, \frac{w_{1j}}{\lVert w_{1j}\rVert}\biggl) \lVert w_{1j}\rVert^2\nonumber\\
    &\leq \beta_2 \alpha^d\lVert Z_1\rVert_F^2\,.
\end{align}
Moreover, we can show that 
\begin{align}
    \lVert\tilde\gamma_1\tilde\gamma_1^\top Z_1\rVert_F = \lVert \tilde\gamma_1^\top Z_1\rVert_F = \lVert g_{\perp}B_1\rVert_F \,.
\end{align}
Thus, we can show that 
\begin{align}
    \lVert g_{\perp}g_{\perp}^\top B_1A_1\rVert \leq \lVert g_{\perp}B_1\rVert \lVert A_1\rVert  \leq \sqrt{\beta_2 \alpha^d}\lVert Z_1\rVert_F^2\,.
\end{align}
\end{proof}

\subsection{Proof of Lemma~\ref{applem:ab_square}}\label{appprooflem:ab_square}
\begin{proof}
We start with
\begin{align}
    & A_1 A_1^\top-B_1^\top B_1:= \alpha^2 \Lambda_1\nonumber\\
    \Rightarrow & B_1 A_1 A_1^\top B_1^\top = B_1 B_1^\top B_1 B_1^\top + \alpha^2  B_1 \Lambda_1 B_1^\top\nonumber\\
    \iff & g^\top B_1 A_1(vv^\top + v_{\perp}v_{\perp}^\top) A_1^\top B_1^\top g = g^\top B_1 B_1^\top (gg^\top +g_{\perp}g_{\perp}^\top)B_1 B_1^\top g + \alpha^2  g^\top B_1 \Lambda_1 B_1^\top g\nonumber\\
    \Rightarrow & (g^\top B_1B_1^\top g)^2 \leq (v^\top A_1^\top B_1^\top g)^2 + (v^\top A_1^\top B_1^\top g_{\perp})^2 + (g^\top B_1B_1^\top g_{\perp})^2 + \alpha^2 \lVert \Lambda_1\rVert  g^\top B_1B_1^\top g
\end{align}
Based on Lemma~\ref{applem:alignment_z1_to_a1b1}, we have shown that
\begin{align}
    \lVert g_{\perp}g_{\perp}^\top B_1A_1\rVert^2, \lVert B_1A_1v_{\perp}v_{\perp}^\top\rVert^2\leq \beta_2 \alpha^{\frac{3+\delta_w}{5-\delta_w}}\lVert Z_1\rVert_F^4\,.
\end{align}
Therefore, one has 
\begin{align}
    (v^\top A_1^\top B_1^\top g_{\perp})^2\leq \lVert g_{\perp}g_{\perp}^\top B_1A_1\rVert^2 \leq \beta_2 \alpha^{\frac{3+\delta_w}{5-\delta_w}}\lVert Z_1\rVert_F^4\,.
\end{align}
Moreover, using the same technique, one can show 
\begin{align}
    (g^\top B_1B_1^\top g_{\perp})^2 &\leq \lVert g_{\perp}^\top B_1 \rVert_F \lVert B_1\rVert \nonumber\\
    &\leq \beta_2 \alpha^{\frac{3+\delta_w}{5-\delta_w}}\lVert B_1 \rVert \lVert B_1\rVert \nonumber\\
    &\leq \beta_2 \alpha^{\frac{3+\delta_w}{5-\delta_w}}\lVert Z_1\rVert_F^2\,.
\end{align}
Combine all the equations together, we have
\begin{align}
    (g^\top B_1B_1^\top g)^2 &\leq (v^\top A_1^\top B_1^\top g)^2 + (v^\top A_1^\top B_1^\top g_{\perp})^2 + (g^\top B_1B_1^\top g_{\perp})^2 + \alpha^2 \lVert \Lambda_1\rVert  g^\top B_1B_1^\top g \nonumber\\
    &\leq (v^\top A_1^\top B_1^\top g)^2 + \alpha^2 \lVert \Lambda_1\rVert  g^\top B_1B_1^\top g 
    + 2\beta_2 \alpha^{\frac{3+\delta_w}{5-\delta_w}}\lVert Z_1\rVert_F^4\,.
\end{align}
\end{proof}
\subsection{Proof of Lemma~\ref{applem:bound_d2}}\label{appprooflem:bound_d2}
\begin{proof}
We first show that based on these assumptions, one can use $g^\top B_1A_1 v$ to upper bound $\lVert Z_1\rVert_F^2$.
Based on Lemma~\ref{applem:ab_square}, ~\eqref{appeqn:gammaz_z} and ~\eqref{appeqn:gammaZ_ab}, we can show that 
\begin{align}
    \biggl(1-d_2d_3^{-\frac{3+\delta_w}{5-\delta_w}} \alpha^{\frac{3+\delta_w}{5-\delta_w}}\biggl)\lVert Z_1 \rVert_F^2 &\leq \lVert \gamma_1\gamma_1^\top Z_1\rVert_F^2\nonumber\\
     &\leq g^\top B_1B_1^\top g + v^\top A_1^\top A_1 v + 2 g^\top B_1A_1 v \nonumber\\
    &\leq (2+\sqrt{2})g^\top B_1A_1 v + \alpha^2 \lVert D_1\rVert + 2\sqrt{\beta}\lVert Z_1\rVert_F^2 \alpha^{\frac{3+\delta_w}{10-2\delta_w}}\nonumber\\
    &\leq (2+\sqrt{2})g^\top B_1A_1 v + \alpha^2\frac{(1-\delta_w)\sigma\sigma_{W_2}}{2}+ 2\sqrt{\beta}\lVert Z_1\rVert_F^2 \alpha^{\frac{3+\delta_w}{10-2\delta_w}}\,.
\end{align}
Thus, when $\alpha$ is sufficently small, one has
\begin{align}
    \lVert Z_1 \rVert_F^2 \leq 2(2+\sqrt{2})g^\top B_1A_1 v \leq 2(2+\sqrt{2})\frac{(1-\delta_w)\sigma}{4\sigma_{W_2}}\,.
\end{align}
Then, we move the objective of interest.
\begin{align}
    &\lVert uu^\top F v_{\perp}v_{\perp}^\top A_1^\top B_1^\top gg^\top\rVert_F+\lVert B_2A_2\rVert_F (\lVert W_1\rVert_F+\lVert B_1A_1\rVert_F)\nonumber\\
    &+\lVert u_{\perp}u_{\perp}^\top F A_1^\top B_1^\top\rVert_F + \lVert u_{\perp}u_{\perp}^\top FW_1^\top \rVert_F
    +\lVert EA_1^\top B_1^\top g_{\perp}g_{\perp}^\top\rVert_F +\lVert FW_1^\top g_{\perp}g_{\perp}^\top\rVert_F\nonumber\\
    \leq& \lVert F v_{\perp}v_{\perp}^\top\rVert_F \lVert B_1A_1 v_{\perp}v_{\perp}^\top\rVert_F +\frac{1}{2}\lVert Z_2\rVert_F^2 (\lVert W_1\rVert_F+\frac{1}{2}\lVert Z_1\rVert_F^2)\nonumber\\
    &+\lVert u_{\perp}u_{\perp}^\top F\rVert_F \lVert B_1A_1\rVert_F + \lVert W_1\rVert \lVert u_{\perp}u_{\perp}^\top F\rVert_F + \lVert E\rVert_F \lVert g_{\perp}g_{\perp}^\top B_1A_1\rVert_F 
    + \lVert W_1\rVert\lVert F v_{\perp}v_{\perp}^\top \rVert_F
\end{align}
Based on Lemma~\ref{applem:alignment_z1_to_a1b1}, we have
\begin{align}\label{appeqn:bound1}
    \lVert g_{\perp}g_{\perp}^\top B_1A_1\rVert_F^2, \lVert B_1A_1 v_{\perp}v_{\perp}^\top\rVert_F^2 \leq \beta_2 r^2\alpha^{\frac{3+\delta_w}{5-\delta_w}}\lVert Z_1\rVert_F^4\,.
\end{align}
One can use the above bound to show that 
\begin{align}\label{appeqn:bound2}
    \lVert F v_{\perp}v_{\perp}^\top\rVert_F &\leq \lVert W_2B_1A_1 v_{\perp}v_{\perp}^\top\rVert_F + \lVert B_2A_2 W_1 v_{\perp}v_{\perp}^\top\rVert_F + \lVert B_2A_2 B_1A_1 v_{\perp}v_{\perp}^\top\rVert_F\nonumber\\
    &\leq \lVert W_2\rVert \lVert B_1A_1 v_{\perp}v_{\perp}^\top\rVert_F + \lVert W_1\rVert \frac{1}{2}\lVert Z_2\rVert_F^2+ \frac{1}{4}\lVert Z_2\rVert_F^2 \lVert Z_1\rVert_F^2 \nonumber\\
    &\leq \sqrt{\beta_2 r^2\alpha^{\frac{3+\delta_w}{5-\delta_w}}}\lVert Z_1\rVert_F^2 \lVert W_2\rVert + \frac{\lVert W_1\rVert}{2} \alpha^d \lVert Z_1\rVert_F^2 + \frac{1}{4} \alpha^d\lVert Z_1\rVert^4\,,
\end{align}
and similarly
\begin{align}\label{appeqn:bound3}
    \lVert u_{\perp}u_{\perp}^\top F\rVert_F \leq \sqrt{\beta_2 r^2\alpha^{\frac{3+\delta_w}{5-\delta_w}}}\lVert Z_1\rVert_F^2 \lVert W_2\rVert + \frac{\lVert W_1\rVert}{2} \alpha^d \lVert Z_1\rVert_F^2 + \frac{1}{4} \alpha^d\lVert Z_1\rVert^4\,.
\end{align}
Therefore, based on \eqref{appeqn:bound1}, \eqref{appeqn:bound2} and \eqref{appeqn:bound3}, we can show
\begin{align}
    &\lVert uu^\top F v_{\perp}v_{\perp}^\top A_1^\top B_1^\top gg^\top\rVert_F+\lVert B_2A_2\rVert_F (\lVert W_1\rVert_F+\lVert B_1A_1\rVert_F)\nonumber\\
    &+\lVert u_{\perp}u_{\perp}^\top F A_1^\top B_1^\top\rVert_F + \lVert u_{\perp}u_{\perp}^\top FW_1^\top \rVert_F
    +\lVert EA_1^\top B_1^\top g_{\perp}g_{\perp}^\top\rVert_F +\lVert FW_1^\top g_{\perp}g_{\perp}^\top\rVert_F\nonumber\\
    \leq& \lVert F v_{\perp}v_{\perp}^\top\rVert_F \lVert B_1A_1 v_{\perp}v_{\perp}^\top\rVert_F +\frac{1}{2} \alpha^d\lVert Z_1\rVert_F^2 (\lVert W_1\rVert_F+\frac{1}{2}\lVert Z_1\rVert_F^2)\nonumber\\
    &+\lVert u_{\perp}u_{\perp}^\top F\rVert_F \frac{1}{2}\lVert Z_1\rVert_F^2 + \lVert W_1\rVert \lVert u_{\perp}u_{\perp}^\top F\rVert_F + \lVert E\rVert_F \lVert g_{\perp}g_{\perp}^\top B_1A_1\rVert_F 
    + \lVert W_1\rVert\lVert F v_{\perp}v_{\perp}^\top \rVert_F\nonumber\\
    \leq & \lVert Z_1\rVert_F^2 \biggl\{ \sqrt{\beta_2 r^2\alpha^{\frac{3+\delta_w}{5-\delta_w}}}\bigl( \sqrt{\beta_2 r^2\alpha^{\frac{3+\delta_w}{5-\delta_w}}}\lVert Z_1\rVert_F^2 \lVert W_2\rVert + \frac{\lVert W_1\rVert}{2} \alpha^d \lVert Z_1\rVert_F^2 + \frac{1}{4} \alpha^d\lVert Z_1\rVert^4\bigl) \nonumber\\
    &+ \frac{1}{2}\alpha^d (\lVert W_1\rVert_F+\frac{1}{2}\lVert Z_1\rVert_F^2)\nonumber\\
    &+ \frac{1}{2}\bigl(\sqrt{\beta_2 r^2\alpha^{\frac{3+\delta_w}{5-\delta_w}}}\lVert Z_1\rVert_F^2 \lVert W_2\rVert + \frac{\lVert W_1\rVert}{2} \alpha^d \lVert Z_1\rVert_F^2 + \frac{1}{4} \alpha^d\lVert Z_1\rVert^4\bigl)\nonumber\\
    &+2\lVert W_1\rVert \bigl(\sqrt{\beta_2 r^2\alpha^{\frac{3+\delta_w}{5-\delta_w}}} \lVert W_2\rVert + \frac{\lVert W_1\rVert}{2} \alpha^d  + \frac{1}{4} \alpha^d\lVert Z_1\rVert^2\bigl)\nonumber\\
    &+ \lVert E(0)\rVert_F \sqrt{\beta_2 r^2\alpha^{\frac{3+\delta_w}{5-\delta_w}}}\biggl\}\nonumber\\
    \leq & 2(2+\sqrt{2}) g^\top B_1A_1v \biggl\{ \sqrt{\beta_2 r^2\alpha^{\frac{3+\delta_w}{5-\delta_w}}}\bigl( \sqrt{\beta_2 r^2\alpha^{\frac{3+\delta_w}{5-\delta_w}}}\lVert Z_1\rVert_F^2 \lVert W_2\rVert + \frac{\lVert W_1\rVert}{2} \alpha^d \lVert Z_1\rVert_F^2 + \frac{1}{4} \alpha^d\lVert Z_1\rVert^4\bigl) \nonumber\\
    &+ \frac{1}{2}\alpha^d (\lVert W_1\rVert_F+\frac{1}{2}\lVert Z_1\rVert_F^2)\nonumber\\
    &+ \frac{1}{2}\bigl(\sqrt{\beta_2 r^2\alpha^{\frac{3+\delta_w}{5-\delta_w}}}\lVert Z_1\rVert_F^2 \lVert W_2\rVert + \frac{\lVert W_1\rVert}{2} \alpha^d \lVert Z_1\rVert_F^2 + \frac{1}{4} \alpha^d\lVert Z_1\rVert^4\bigl)\nonumber\\
    &+2\lVert W_1\rVert \bigl(\sqrt{\beta_2 r^2\alpha^{\frac{3+\delta_w}{5-\delta_w}}} \lVert W_2\rVert + \frac{\lVert W_1\rVert}{2} \alpha^d  + \frac{1}{4} \alpha^d\lVert Z_1\rVert^2\bigl)\nonumber\\
    &+ \lVert E(0)\rVert_F \sqrt{\beta_2 r^2\alpha^{\frac{3+\delta_w}{5-\delta_w}}}\biggl\}\,.
\end{align}
Notice all the terms in the big bracket goes to zero as $\alpha$ goes to zero. Thus, when $\alpha$ is sufficiently small, one can show that the RHS of the above inequality is upper bounded by 
$\biggl(\frac{(3+\delta_w)\sigma}{4}+\sigma_{W_2}\sigma_{W_1}\biggl)g^\top B_1A_1 v$, which completes the proof.
\end{proof}
\subsection{Proof of Lemma~\ref{applem:lower_bound_Ls}}\label{appprooflem:lower_bound_Ls}
\begin{proof}
We study $\lVert\frac{L_S}{\partial A_1}\rVert_F^2$ as an example.
\begin{align}
    \lVert\frac{L_S}{\partial A_1}\rVert_F^2 &= \lVert B_1^\top(W_2+B_2A_2)^\top uu^\top Evv^\top  \rVert_F^2 \nonumber\\
    &=(u^\top Ev)^2 \mathrm{Tr}\biggl(u^\top (W_2+B_2A_2)B_1B_1^\top (W_2+B_2A_2)^\top u\biggl)\nonumber\\
    &= 2u^\top (W_2+B_2A_2)B_1B_1^\top (W_2+B_2A_2)^\top u \times L_S\,.
\end{align}
Similarly, one can show 
\begin{align}
    \lVert\frac{L_S}{\partial B_1}\rVert_F^2 = 2 v^\top A_1^\top A_1 v \times u^\top (W_2+B_2A_2) (W_2+B_2A_2)^\top u \times L_S\,.
\end{align}
Notice first
\begin{align}
    u^\top (W_2+B_2A_2) (W_2+B_2A_2)^\top u \!=\! \sigma_{W_2}^2 \!+\! u^\top B_2A_2A_2^\top B_2^\top u \!+\! \sigma_{W_2}g^\top B_2A_2 u \geq \sigma_{W_2}^2 \!-\! \frac{\lVert Z_2\rVert_F^2}{2}\sigma_{W_2}\,.
\end{align}
Under the condition that $\lVert Z_1\rVert_F\leq d_5, \lVert Z_2\rVert_F^2 \leq \alpha^{d_6}\lVert Z_1\rVert_F^2$, we can show that
\begin{align}
    u^\top (W_2+B_2A_2) (W_2+B_2A_2)^\top u \!=\! \sigma_{W_2}^2 \!-\! \frac{\alpha^{d_6} d_5^2}{2}\sigma_{W_2}\,.
\end{align}
On the other hand, we can show that 
\begin{align}
    g^\top B_1B_1^\top g + v^\top A_1^\top A_1 v =\lVert g^\top B_1\rVert^2 + \lVert A_1v\rVert^2 \geq \frac{1}{2}\lVert g^\top B_1 + v^\top A_1^\top\rVert^2 = \frac{1}{2}\lVert \gamma_1^\top Z_1\rVert_F^2\,.
\end{align}
Therefore, one can conclude that 
\begin{align}
    \lVert\frac{L_S}{\partial A_1}\rVert_F^2+\lVert\frac{L_S}{\partial B_1}\rVert_F^2 \geq \lVert \gamma_1^\top Z_1\rVert_F^2\biggl(\sigma_{W_2}^2 \!-\! \frac{\alpha^{d_6} d_5^2}{2}\sigma_{W_2}\biggl) L_S\,.
\end{align}
\end{proof}

\subsection{Proof of Lemma~\ref{applem:lower_bound_gammaZ}}\label{appprooflem:lower_bound_gammaZ}
\begin{proof}
We first show that at the end of \textit{alignment} phase, loss has decreased by a constant order.
\begin{align}
    \lVert E\rVert_F =& \lVert \Delta Y-W_2B_1A_1 + B_2A_2W_1+B_2A_2B_1A_1\rVert_F  \nonumber\\
    =&\lVert uu^\top (\Delta Y-W_2B_1A_1 + B_2A_2W_1+B_2A_2B_1A_1) vv^\top\rVert_F \nonumber\\
    &+ \lVert u_{\perp}u_{\perp}^\top (\Delta Y-W_2B_1A_1 + B_2A_2W_1+B_2A_2B_1A_1) v^\top\rVert_F\nonumber\\
    &+ \lVert uu^\top (\Delta Y-W_2B_1A_1 + B_2A_2W_1+B_2A_2B_1A_1) v_{\perp}v_{\perp}^\top\rVert_F \nonumber\\
    \leq&\lVert u (\sigma - \sigma_{W_2}g^\top B_1A_1 v) v^\top\rVert_F + \lVert W_1\rVert \lVert B_2A_2\rVert_F + \lVert B_2A_2\rVert_F \lVert B_1A_1\rVert_F\nonumber\\
    &+\lVert W_2\rVert \lVert g_{\perp}B_1A_1\rVert_F + \lVert W_1\rVert \lVert B_2A_2\rVert_F + \lVert B_2A_2\rVert_F \lVert B_1A_1\rVert_F\nonumber\\
    &+\lVert W_2\rVert \lVert B_1A_1 v_{\perp}\rVert_F + \lVert W_1\rVert \lVert B_2A_2\rVert_F + \lVert B_2A_2\rVert_F \lVert B_1A_1\rVert_F\nonumber\\
    &= \biggl(1-\frac{(1-\delta_w)}{4}\biggl)\sigma \!+\! 3 \lVert W_1\rVert \lVert B_2A_2\rVert_F\!+\!3\lVert B_2A_2\rVert_F \lVert B_1A_1\rVert_F\!+\!\lVert W_2\rVert \lVert g_{\perp}B_1A_1\rVert_F\!+\!\lVert W_2\rVert \lVert B_1A_1 v_{\perp}\rVert_F\nonumber\\
    &\leq \biggl(1-\frac{(1-\delta_w)}{4}\biggl)\sigma \!+\! \frac{3}{2} \lVert W_1\rVert \alpha^{d_6} d_5^2\!
    +\!\frac{3}{4} \alpha^{d_6} d_5^4\!+\!\lVert W_2\rVert \lVert g_{\perp}B_1A_1\rVert_F\!+\!\lVert W_2\rVert \lVert B_1A_1 v_{\perp}\rVert_F
\end{align}
Moreover, by Lemma~\ref{applem:alignment_z1_to_a1b1}, we can show that
\begin{align}
    \lVert g_{\perp}B_1A_1\rVert_F &\leq \sqrt{r} \lVert g_{\perp}B_1A_1\rVert \leq \sqrt{r\alpha^{d_7}} \lVert Z_1\rVert_F^2\,, \nonumber \\
    \lVert B_1A_1 v_{\perp}\rVert_F &\leq \sqrt{r} \lVert B_1A_1 v_{\perp}\rVert \leq \sqrt{r\alpha^{d_7}} \lVert Z_1\rVert_F^2\,.
\end{align}
Therefore, one can show that there exists $\alpha^*(\lVert W_1\rVert, \lVert W_2\rVert, d_5, d_6, d_7)$ such that when $0<\alpha\leq\alpha^*(\lVert W_1\rVert, \lVert W_2\rVert, d_5, d_6, d_7)$, we have
\begin{align}
    \lVert E\rVert_F 
    &\leq \biggl(1-\frac{(1-\delta_w)}{4}\biggl)\sigma \!+\! \frac{3}{2} \lVert W_1\rVert \alpha^{d_6} d_5^2\!
    +\!\frac{3}{4} \alpha^{d_6} d_5^4\!+\!\lVert W_2\rVert \lVert g_{\perp}B_1A_1\rVert_F\!+\!\lVert W_2\rVert \lVert B_1A_1 v_{\perp}\rVert_F \nonumber\\
    &\leq \biggl(1-\frac{(1-\delta_w)}{4}\biggl)\sigma \!+\! \frac{3}{2} \lVert W_1\rVert \alpha^{d_6} d_5^2\!
    +\!\frac{3}{4} \alpha^{d_6} d_5^4\!+\!2\lVert W_2\rVert \sqrt{r\alpha^{d_7}} d_5^4 \nonumber\\
    &\leq \biggl(1-\frac{(1-\delta_w)}{8}\biggl)\sigma\,.
\end{align}
Moreover, since the loss in non-increasing when trained under GF, we can see that $\lVert E(t)\rVert_F \leq \biggl(1-\frac{(1-\delta_w)}{8}\biggl)\sigma$ for all $t\geq T_1$\,.
In the remaining of the proof, we show that $\lVert E(t)\rVert_F \leq \biggl(1-\frac{(1-\delta_w)}{8}\biggl)\sigma$ induces an lower bound on $\lVert B_1A_1\rVert_F$.
\begin{align}
    \lVert E\rVert_F =& \lVert \Delta Y-W_2B_1A_1 + B_2A_2W_1+B_2A_2B_1A_1\rVert_F  \nonumber\\
    =&\lVert uu^\top (\Delta Y-W_2B_1A_1 + B_2A_2W_1+B_2A_2B_1A_1) vv^\top\rVert_F \nonumber\\
    &+ \lVert u_{\perp}u_{\perp}^\top (\Delta Y-W_2B_1A_1 + B_2A_2W_1+B_2A_2B_1A_1) v^\top\rVert_F\nonumber\\
    &+ \lVert uu^\top (\Delta Y-W_2B_1A_1 + B_2A_2W_1+B_2A_2B_1A_1) v_{\perp}v_{\perp}^\top\rVert_F \nonumber\\
    \geq&\lVert u (\sigma - \sigma_{W_2}g^\top B_1A_1 v) v^\top\rVert_F - \lVert W_1\rVert \lVert B_2A_2\rVert_F - \lVert B_2A_2\rVert_F \lVert B_1A_1\rVert_F\nonumber\\
    &-\lVert W_2\rVert \lVert g_{\perp}B_1A_1\rVert_F - \lVert W_1\rVert \lVert B_2A_2\rVert_F - \lVert B_2A_2\rVert_F \lVert B_1A_1\rVert_F\nonumber\\
    &-\lVert W_2\rVert \lVert B_1A_1 v_{\perp}\rVert_F - \lVert W_1\rVert \lVert B_2A_2\rVert_F - \lVert B_2A_2\rVert_F \lVert B_1A_1\rVert_F\nonumber\\
    \geq& \sigma - \lVert u \sigma_{W_2}g^\top B_1A_1 v v^\top\rVert_F - \lVert W_1\rVert \lVert B_2A_2\rVert_F - \lVert B_2A_2\rVert_F \lVert B_1A_1\rVert_F\nonumber\\
    &-\lVert W_2\rVert \lVert g_{\perp}B_1A_1\rVert_F - \lVert W_1\rVert \lVert B_2A_2\rVert_F - \lVert B_2A_2\rVert_F \lVert B_1A_1\rVert_F\nonumber\\
    &-\lVert W_2\rVert \lVert B_1A_1 v_{\perp}\rVert_F - \lVert W_1\rVert \lVert B_2A_2\rVert_F - \lVert B_2A_2\rVert_F \lVert B_1A_1\rVert_F\,.
\end{align}
Therefore,
\begin{align}
    \sigma_{W_2}g^\top B_1A_1 v &\geq \sigma - \lVert E\rVert_F- \lVert W_1\rVert \lVert B_2A_2\rVert_F - \lVert B_2A_2\rVert_F \lVert B_1A_1\rVert_F\nonumber\\
    &-\lVert W_2\rVert \lVert g_{\perp}B_1A_1\rVert_F - \lVert W_1\rVert \lVert B_2A_2\rVert_F - \lVert B_2A_2\rVert_F \lVert B_1A_1\rVert_F\nonumber\\
    &-\lVert W_2\rVert \lVert B_1A_1 v_{\perp}\rVert_F - \lVert W_1\rVert \lVert B_2A_2\rVert_F - \lVert B_2A_2\rVert_F \lVert B_1A_1\rVert_F\nonumber\\
    &\geq \frac{(1-\delta_w)\sigma}{8} - \frac{3}{2} \lVert W_1\rVert \alpha^{d_6} d_5^2\!
    -\!\frac{3}{4} \alpha^{d_6} d_5^4\!-\!2\lVert W_2\rVert \sqrt{r\alpha^{d_7}} d_5^4 \nonumber\\
    &\geq \frac{(1-\delta_w)\sigma}{16}\,,
\end{align}
where the last line follows when $\alpha$ is sufficiently small, one can have 
\begin{align}
    \frac{3}{2} \lVert W_1\rVert \alpha^{d_6} d_5^2\!+\!\frac{3}{4} \alpha^{d_6} d_5^4\!+\!2\lVert W_2\rVert \sqrt{r\alpha^{d_7}} d_5^4 \leq \frac{(1-\delta_w)\sigma}{16}\,.
\end{align}
Finally,
\begin{align}
    \lVert \gamma_1^\top Z_1\rVert_F^2 =  \lVert g^\top B_1+v^\top A_1^\top \rVert^2 \geq \lVert g^\top B_1\rVert^2 +\lVert g^\top A_1^\top \rVert^2 \geq 2\lVert g^\top B_1 A_1 v\rVert\,.
\end{align}
Thus, $\lVert \gamma_1^\top Z_1\rVert_F^2 \geq 2g^\top B_1A_1 v\geq\frac{(1-\delta_w)\sigma}{8}$.
\end{proof}

\subsection{Proof of Lemma~\ref{applem:LsLn_bound}}\label{appprooflem:LsLn_bound}
\begin{proof}
We first study $\bigl\langle \frac{\partial L_S}{\partial B_1}, \frac{\partial L_N}{\partial B_1}\bigl\rangle$.
\begin{align}
    \bigl\langle \frac{\partial L_S}{\partial B_1}, \frac{\partial L_N}{\partial B_1}\bigl\rangle =& 
    \bigl\langle (W_2+B_2A_2)^\top uu^\top Evv^\top A_1^\top, (W_2+B_2A_2)^\top uu^\top Ev_{\perp}v_{\perp}^\top A_1^\top\bigl\rangle \nonumber\\
    &+ \bigl\langle (W_2+B_2A_2)^\top uu^\top Evv^\top A_1^\top, (W_2+B_2A_2)^\top u_{\perp}u_{\perp}^\top E A_1^\top\bigl\rangle\nonumber\\
    =& \mathrm{Tr} \bigl((W_2+B_2A_2)^\top uu^\top Ev v^\top A_1^\top A_1 v_{\perp}v_{\perp}^\top F^\top uu^\top (W_2+B_2A_2) \bigl) \nonumber\\
    &+ \mathrm{Tr}\bigl( (W_2+B_2A_2)^\top uu^\top Evv^\top A_1^\top A_1 F u_{\perp}u_{\perp}^\top (W_2+B_2A_2)\bigl)\nonumber\\
    =& u^\top E v \biggl\{\mathrm{Tr} \bigl((W_2+B_2A_2)^\top uv^\top A_1^\top A_1 v_{\perp}v_{\perp}^\top F^\top uu^\top (W_2+B_2A_2) \bigl) \nonumber\\
    &+ \mathrm{Tr}\bigl( (W_2+B_2A_2)^\top uv^\top A_1^\top A_1 F^\top u_{\perp}u_{\perp}^\top (W_2+B_2A_2)\bigl)\biggl\}
\end{align}
Therefore, one has
\begin{align}
    &\biggl\lvert \bigl\langle \frac{\partial L_S}{\partial B_1}, \frac{\partial L_N}{\partial B_1}\bigl\rangle\biggl\rvert \nonumber\\
    =& \biggl\lvert u^\top E v \biggl\{\mathrm{Tr} \bigl((W_2+B_2A_2)^\top uv^\top A_1^\top A_1 v_{\perp}v_{\perp}^\top F^\top uu^\top (W_2+B_2A_2) \bigl) \nonumber\\
    &+ \mathrm{Tr}\bigl( (W_2+B_2A_2)^\top uv^\top A_1^\top A_1 F^\top u_{\perp}u_{\perp}^\top (W_2+B_2A_2)\bigl)\biggl\}\biggl\rvert\nonumber\\
    \leq & \sqrt{2L_S} \biggl\{ \bigl\lVert u^\top (W_2+B_2A_2)(W_2+B_2A_2)^\top u_{\perp}\bigl\rVert_F \times \lVert u^\top Fv_{\perp}\rVert_F \times \lVert v^\top A_1^\top A_1 v_{\perp}\rVert_F\nonumber\\
    &+ \bigl\lVert u^\top (W_2+B_2A_2)(W_2+B_2A_2)^\top u_{\perp}\bigl\rVert_F \times \lVert u_{\perp}^\top F\rVert_F \times \lVert v^\top A_1^\top A_1\rVert_F \biggl\}
\end{align}
Then, we provide bounds for each term on the RHS of the above equation.

First, we can upper bound 
\begin{align}
    \bigl\lVert u^\top (W_2+B_2A_2)(W_2+B_2A_2)^\top u_{\perp}\bigl\rVert_F &\leq \bigl\lVert (W_2+B_2A_2)(W_2+B_2A_2)^\top \bigl\rVert_F \nonumber\\
    &= \lVert W_2+B_2A_2\rVert_F^2\nonumber\\
    &\leq2\bigl(\lVert W_2\rVert_F^2 +\lVert B_2A_2\rVert_F^2\bigl)\nonumber\\
    &\leq 2\lVert W_2\rVert_F^2 +2\lVert Z_2\rVert_F^2\nonumber\\
    &\leq 2\lVert W_2\rVert_F^2 + 2\alpha^{d_6}d_5^2\,.
\end{align}
Second,
\begin{align}
    \lVert v^\top A_1^\top A_1 v_{\perp}\rVert_F, \lVert v^\top A_1^\top A_1\rVert_F \leq \lVert A_1\rVert_F^2\,.
\end{align}
Last,
\begin{align}\label{appeqn:bound_Fperp}
    \lVert u^\top Fv_{\perp}\rVert_F &\leq \lVert Fv_{\perp}\rVert_F \nonumber\\
    &=\lVert (W_2B_1A_1+B_2A_2W_1+B_2A_2B_1A_1)v_{\perp}\rVert_F\nonumber\\
    &\leq \lVert W_2\rVert \lVert B_1A_1v_{\perp}\rVert_F + \lVert B_2A_2\rVert_F\times\bigl(\lVert W_1\rVert+\lVert B_1A_1\rVert_F\bigl)\nonumber\\
    &\leq \lVert W_2\rVert \sqrt{\alpha^{d_7}} \lVert Z_1\rVert_F^2 + \frac{1}{2} \lVert Z_2\rVert_F^2 \times\bigl(\lVert W_1\rVert+\frac{1}{2} \lVert Z_1\rVert_F^2\bigl)\nonumber\\
    &\leq \lVert W_2\rVert \sqrt{\alpha^{d_7}}d_5^2 + \frac{1}{2}\alpha^{d_6}d_5^2 \times (\lVert W_1\rVert + \frac{d_5^2}{2})\,.
\end{align}
Similarly, one can show $\lVert u_{\perp}^\top F\rVert_F\leq \lVert W_2\rVert \sqrt{\alpha^{d_7}}d_5^2 + \frac{1}{2}\alpha^{d_6}d_5^2 \times (\lVert W_1\rVert + \frac{d_5^2}{2})$.
Combine these results together, one can show
\begin{align}
    &\biggl\lvert \bigl\langle \frac{\partial L_S}{\partial B_1}, \frac{\partial L_N}{\partial B_1}\bigl\rangle\biggl\rvert \nonumber\\
    \leq & \sqrt{2L_S} \biggl\{ \bigl\lVert u^\top (W_2+B_2A_2)(W_2+B_2A_2)^\top u_{\perp}\bigl\rVert_F \times \lVert u^\top Fv_{\perp}\rVert_F \times \lVert v^\top A_1^\top A_1 v_{\perp}\rVert_F\nonumber\\
    &+ \bigl\lVert u^\top (W_2+B_2A_2)(W_2+B_2A_2)^\top u_{\perp}\bigl\rVert_F \times \lVert u_{\perp}^\top F\rVert_F \times \lVert v^\top A_1^\top A_1\rVert_F \biggl\}\nonumber\\
    \leq & 2\sqrt{2L_S}\bigl(2\lVert W_2\rVert_F^2 + 2\alpha^{d_6}d_5^2\bigl) \times \lVert A_1\rVert_F^2 \times \biggl(\lVert W_2\rVert \sqrt{\alpha^{d_7}}d_5^2 + \frac{1}{2}\alpha^{d_6}d_5^2 \times (\lVert W_1\rVert + \frac{d_5^2}{2})\biggl)\,.
\end{align}
Similarly, we can also show
\begin{align}
    &\biggl\lvert \bigl\langle \frac{\partial L_S}{\partial A_1}, \frac{\partial L_N}{\partial B_1}\bigl\rangle\biggl\rvert \nonumber\\
    \leq & 2\sqrt{2L_S}\bigl(2\lVert W_2\rVert_F^2 + 2\alpha^{d_6}d_5^2\bigl) \times \lVert B_1\rVert_F^2 \times \biggl(\lVert W_2\rVert \sqrt{\alpha^{d_7}}d_5^2 + \frac{1}{2}\alpha^{d_6}d_5^2 \times (\lVert W_1\rVert + \frac{d_5^2}{2})\biggl)\,.
\end{align}
Thus, we conclude
\begin{align}\label{appeqn:bound_LsLn_a1b1}
    \biggl\lvert \bigl\langle \frac{\partial L_S}{\partial A_1}, \frac{\partial L_N}{\partial A_1}\bigl\rangle\biggl\rvert + \biggl\lvert \bigl\langle \frac{\partial L_S}{\partial B_1}, \frac{\partial L_N}{\partial B_1}\bigl\rangle\biggl\rvert
    \leq 2\sqrt{2L_S}\bigl(2\lVert W_2\rVert_F^2 + 2\alpha^{d_6}d_5^2\bigl) \times d_5^2 \times \biggl(\lVert W_2\rVert \sqrt{\alpha^{d_7}}d_5^2 + \frac{1}{2}\alpha^{d_6}d_5^2 \times (\lVert W_1\rVert + \frac{d_5^2}{2})\biggl)
\end{align}
Next, we study $\bigl\langle \frac{\partial L_S}{\partial B_2}, \frac{\partial L_N}{\partial B_2}\bigl\rangle$.
\begin{align}
    \bigl\langle \frac{\partial L_S}{\partial B_2}, \frac{\partial L_N}{\partial B_2}\bigl\rangle =& 
    \bigl\langle  uu^\top E vv^\top (W_1+B_1A_1)^\top A_2^\top,  uu^\top E v_{\perp}v_{\perp}^\top (W_1+B_1A_1)^\top A_2^\top\bigl\rangle \nonumber\\
    &+ \bigl\langle  uu^\top E vv^\top (W_1+B_1A_1)^\top A_2^\top,  u_{\perp}u_{\perp}^\top E (W_1+B_1A_1)^\top A_2^\top\bigl\rangle\nonumber\\
    =& u^\top E v \biggl\{\mathrm{Tr} \biggl(v^\top (W_1+B_1A_1)^\top A_2^\top A_2(W_1+B_1A_1)v_{\perp}v_{\perp}^\top F^\top uu^\top \biggl) \nonumber\\
    &+ \mathrm{Tr} \biggl(v^\top (W_1+B_1A_1)^\top A_2^\top A_2(W_1+B_1A_1) F^\top u_{\perp}u_{\perp}^\top \biggl)\biggl\}\,.
\end{align}
Thus, one can show that 
\begin{align}
    \biggl\lvert \bigl\langle \frac{\partial L_S}{\partial B_2}, \frac{\partial L_N}{\partial B_2}\bigl\rangle\biggl\rvert 
    \leq \sqrt{2L_S} \biggl\{ \bigl\lVert W_1+B_1A_1\bigl\rVert_F^2 \times \lVert u^\top F v_{\perp}\rVert_F \times \lVert A_2\rVert_F^2
    + \bigl\lVert W_1+B_1A_1\bigl\rVert_F^2 \times \lVert u_{\perp}^\top F \rVert_F \times \lVert A_2\rVert_F^2 \biggl\}\,.
\end{align}
We apply~\eqref{appeqn:bound_Fperp} to the above inequality
\begin{align}
    \biggl\lvert \bigl\langle \frac{\partial L_S}{\partial B_2}, \frac{\partial L_N}{\partial B_2}\bigl\rangle\biggl\rvert 
    &\leq \sqrt{2L_S} \biggl\{ \bigl\lVert W_1+B_1A_1\bigl\rVert_F^2 \times \lVert u^\top F v_{\perp}\rVert_F \times \lVert A_2\rVert_F^2
    + \bigl\lVert W_1+B_1A_1\bigl\rVert_F^2 \times \lVert u_{\perp}^\top F \rVert_F \times \lVert A_2\rVert_F^2 \biggl\}\nonumber\\
    &\leq 4\sqrt{2L_S} (\lVert W_1\rVert_F^2+\frac{1}{2}\lVert Z_1\rVert_F^2)\lVert A_2\rVert_F^2 \biggl(\lVert W_2\rVert \sqrt{\alpha^{d_7}}d_5^2 + \frac{1}{2}\alpha^{d_6}d_5^2 \times (\lVert W_1\rVert + \frac{d_5^2}{2})\biggl)\,.
\end{align}
Similarly, we can also show that
\begin{align}
    \biggl\lvert \bigl\langle \frac{\partial L_S}{\partial A_2}, \frac{\partial L_N}{\partial A_2}\bigl\rangle\biggl\rvert \leq4\sqrt{2L_S} (\lVert W_1\rVert_F^2+\frac{1}{2}\lVert Z_1\rVert_F^2)\lVert B_2\rVert_F^2 \biggl(\lVert W_2\rVert \sqrt{\alpha^{d_7}}d_5^2 + \frac{1}{2}\alpha^{d_6}d_5^2 \times (\lVert W_1\rVert + \frac{d_5^2}{2})\biggl)\,.
\end{align}
Add them together, and one has
\begin{align}\label{appeqn:bound_LsLn_a2b2}
    \biggl\lvert \bigl\langle \frac{\partial L_S}{\partial A_2}, \frac{\partial L_N}{\partial A_2}\bigl\rangle\biggl\rvert \!+\!
    \biggl\lvert \bigl\langle \frac{\partial L_S}{\partial B_2}, \frac{\partial L_N}{\partial B_2}\bigl\rangle\biggl\rvert 
    \leq 4\sqrt{2L_S} (\lVert W_1\rVert_F^2\!+\!\frac{1}{2}\lVert Z_1\rVert_F^2)\alpha^{d_6}d_5^2 \biggl(\lVert W_2\rVert \sqrt{\alpha^{d_7}}d_5^2 \!+\! \frac{1}{2}\alpha^{d_6}d_5^2 \times (\lVert W_1\rVert \!+\! \frac{d_5^2}{2})\biggl)\,.
\end{align}
Based on ~\eqref{appeqn:bound_LsLn_a1b1} and ~\eqref{appeqn:bound_LsLn_a2b2}, we conclude
\begin{align}
    &\biggl\lvert \bigl\langle \frac{\partial L_S}{\partial A_1}, \frac{\partial L_N}{\partial A_1}\bigl\rangle\biggl\rvert \!+\!
    \biggl\lvert \bigl\langle \frac{\partial L_S}{\partial B_1}, \frac{\partial L_N}{\partial B_1}\bigl\rangle\biggl\rvert \!+\!
    \biggl\lvert \bigl\langle \frac{\partial L_S}{\partial A_2}, \frac{\partial L_N}{\partial A_2}\bigl\rangle\biggl\rvert \!+\!
    \biggl\lvert \bigl\langle \frac{\partial L_S}{\partial B_2}, \frac{\partial L_N}{\partial B_2}\bigl\rangle\biggl\rvert \nonumber\\
    \leq &  2\sqrt{2L_S}\bigl(2\lVert W_2\rVert_F^2 + 2\alpha^{d_6}d_5^2\bigl) \times d_5^2 \times \biggl(\lVert W_2\rVert \sqrt{\alpha^{d_7}}d_5^2 + \frac{1}{2}\alpha^{d_6}d_5^2 \times (\lVert W_1\rVert + \frac{d_5^2}{2})\biggl)\nonumber\\
    &+4\sqrt{2L_S} (\lVert W_1\rVert_F^2\!+\!\frac{1}{2}\lVert Z_1\rVert_F^2)\alpha^{d_6}d_5^2 \biggl(\lVert W_2\rVert \sqrt{\alpha^{d_7}}d_5^2 \!+\! \frac{1}{2}\alpha^{d_6}d_5^2 \times (\lVert W_1\rVert \!+\! \frac{d_5^2}{2})\biggl)\,.
\end{align}
\end{proof}
\section{Proof of Theorem~\ref{thm:rank-r}}\label{app:proof_spectral}
\begin{thm}
Let $\delta_{w}^{(i)}\!=\!\frac{\sigma_{W_2}^{(i)}}{\sigma_{W_1}^{(i)}}, \ell^{(i)}(t)\!=\!\frac{1}{2}\bigl(\sigma_{\Delta Y}^{(i)}\!-\!\sigma_{\bar F}^{(i)}\bigl)^2, z_1^{(i)}\!=\!(\sigma_{A_1}^{(i)})^2\!+\!(\sigma_{B_1}^{(i)})^2$ (respectively $z_2^{(i)}$). 
In the case where $\delta_w^{(i)}\!\not=\!1$, we assume $\delta_w^{(i)}\!<\!1$ WLOG, then the learning dynamics has two phases which can be separated 
by $T_1^{(i)}\!=\!\frac{2}{(3+\delta_w^{(i)})\sigma\sigma_{W_2}^{(i)}} \log\biggl( \frac{(1-\delta_w^{(i)})\sigma_{\Delta Y}^{(i)}}{8\sigma_{W_2}^{(i)} z_1^{(i)}(0)}\biggl)$
\begin{enumerate}[leftmargin=0.45cm]
    \item \textit{Growth of Norm and Imbalance}: $\forall t\!\leq\! T_1^{(i)}$
    \begin{align}
        \frac{d}{dt} \log z_1^{(i)}\!&\geq\!\frac{3\!+\!\delta_w^{(i)}}{2}\sigma_{\Delta Y}^{(i)}\sigma_{W_2}^{(i)}\,,\nonumber\\
        \frac{d}{dt}\log\biggl(\frac{z_1^{(i)}}{z_2^{(i)}}\biggl)\!&\geq\!\frac{3(1\!-\!\delta_w^{(i)})}{2}\sigma_{\Delta Y}^{(i)}\sigma_{W_2}^{(i)}\,.
    \end{align}
    \item \textit{Local Convergence}: for $\forall t\!\geq\! T_1^{(i)}$, the loss converges linearly
    \begin{align}
        \ell^{(i)}(t)\!\leq\!\exp\biggl(-\frac{(1\!-\!\delta_w^{(i)})\sigma_{\Delta Y}^{(i)} \sigma_{W_2}^{(i)}(t\!-\!T_1)}{8}\biggl) \ell^{(i)}(T_1)\,.\nonumber
    \end{align}
\end{enumerate}
\end{thm}
\begin{proof}
Under spectral initialization, the learning dynamics of LoRA weights can be decoupled to several scalar dynamics. 
WLOG, we prove the learning dynamics when $i=1$ as an example, and assume $\delta_w^{(i)}<1$. 
In the rest of the proof, we will omit the superscript $(i)$ for convenience. Throughout the paper, we will use $e=\sigma_{\Delta Y}-\sigma_f$ to denote the residual.

We first argue that $\sigma_{B_1}\sigma_{A_1}, \sigma_{B_2}\sigma_{A_2}$ will be always positive during the training.
We show this is true for $\sigma_{B_1}\sigma_{A_1}$ as an example. Similar argument holds for $\sigma_{B_2}\sigma_{A_2}$.
\begin{align}
    \frac{d}{dt}\sigma_{B_1}\sigma_{A_1} = (\sigma_{A_1}^2+\sigma_{B_1}^2)\sigma_{W_2}e\,.
\end{align}
Notice $e(0)=\sigma_{\Delta Y}$ and $\sigma_{B_1}(0)\sigma_{A_1}(0)=0$, thus, $\frac{d}{dt}\sigma_{B_1}\sigma_{A_1}>0$ during the training until $e=0$ which is a stationary point.

\myparagraph{\textit{Growth of Norm and Imbalance phase}} 
We first can see that 
\begin{align}
    \frac{d}{dt}\begin{pmatrix} \sigma_{B_1}\\ \sigma_{A_1}\end{pmatrix} &= \begin{pmatrix} 0 & (\sigma_{W_2}+\sigma_{B_2}\sigma_{A_2})e\\ (\sigma_{W_2}+\sigma_{B_2}\sigma_{A_2})e & 0\end{pmatrix} \begin{pmatrix} \sigma_{B_1}\\ \sigma_{A_1}\end{pmatrix}\nonumber\\
    &= \begin{pmatrix} 0 & \sigma_{W_2}\sigma_{\Delta Y}\\ \sigma_{W_2}\sigma_{\Delta Y} & 0\end{pmatrix} \begin{pmatrix} \sigma_{A_1}\\ \sigma_{B_1}\end{pmatrix}
    +\underbrace{\begin{pmatrix} 0 & -\sigma_{W_2}\sigma_f+e\sigma_{B_2}\sigma_{A_2} \\ -\sigma_{W_2}\sigma_f+e\sigma_{B_2}\sigma_{A_2}  & 0\end{pmatrix}}_{D_1} \begin{pmatrix} \sigma_{B_1}\\ \sigma_{A_1}\end{pmatrix}\,.
\end{align}
Similarly, one can show that 
\begin{align}
    \frac{d}{dt}\begin{pmatrix} \sigma_{B_2}\\ \sigma_{A_2}\end{pmatrix} = \begin{pmatrix} 0 & \sigma_{W_1}\sigma_{\Delta Y}\\ \sigma_{W_1}\sigma_{\Delta Y} & 0\end{pmatrix} \begin{pmatrix} \sigma_{B_2}\\ \sigma_{A_2}\end{pmatrix}
    +\underbrace{\begin{pmatrix} 0 & -\sigma_{W_1}\sigma_f+e\sigma_{B_1}\sigma_{A_1} \\ -\sigma_{W_1}\sigma_f+e\sigma_{B_1}\sigma_{A_1}  & 0\end{pmatrix}}_{D_2} \begin{pmatrix} \sigma_{B_2}\\ \sigma_{A_2}\end{pmatrix}
\end{align}
For convenience, we define the following notation $h_1=-\sigma_{W_2}\sigma_f+e\sigma_{B_2}\sigma_{A_2}, h_2=-\sigma_{W_1}\sigma_f+e\sigma_{B_1}\sigma_{A_1}$.
It is obvious that $\lvert h_1\rvert=\lVert D_1\rVert, \lvert h_2\rvert=\lVert D_2\rVert$.

Notice at initialization, $\lVert D_1\rVert + \lVert D_2\rVert\sim \mathcal{O}(\alpha^2)$.
We first cut off the time when $\lVert D_1\rVert + \lVert D_2\rVert=2\alpha$, denoted by $\hat T_1$, then we can show that the imbalance between $z_1$ and $z_2$ grows monotonically for all $0\leq t\leq \hat T_1$.
\begin{align}
    \frac{d}{dt} \log(\frac{z_1}{z_2}) \geq 2(1-\delta_w)\sigma_{\Delta Y}\sigma_{W_2} +2(h_1-h_2) \geq (1-\delta_w)\sigma_{\Delta Y}\sigma_{W_2}\,,
\end{align}
where the last inequality holds under the assumption that $\alpha \leq (1-\delta_w)\sigma_{\Delta Y}\sigma_{W_2}$.
We then characterize the time it takes for $\lVert D_1\rVert + \lVert D_2\rVert$ to reach $2\alpha$.
\begin{align}
    \lVert D_1\rVert &= \lvert -\sigma_{W_2}\sigma_f+e\sigma_{B_2}\sigma_{A_2}\rvert \nonumber\\
    &\leq \sigma_{W_2}\lvert \sigma_f\rvert + \frac{1}{2}\lvert e(0)\rvert \times z_2\nonumber\\
    &\leq \sigma_{W_2}\biggl(\sigma_{W_2} \lvert \sigma_{B_1}\sigma_{A_1}\rvert + \sigma_{W_1} \lvert \sigma_{B_2}\sigma_{A_2}\rvert + \lvert \sigma_{B_1}\sigma_{A_1}\rvert\times \lvert \sigma_{B_2}\sigma_{A_2}\rvert\biggl) + \frac{1}{2}\lvert e(0)\rvert \times z_2\nonumber\\
    &\leq \frac{1}{2}\sigma_{W_2}^2 z_1+\sigma_{W_1}\sigma_{W_2}z_2+\frac{1}{4}\sigma_{W_2}z_1z_2+\frac{\lvert e(0)\rvert}{2} z_2
\end{align}
Similarly, one can show that 
\begin{align}
    \lVert D_2\rVert\leq \frac{1}{2}\sigma_{W_1}^2 z_2+\sigma_{W_2}\sigma_{W_1}z_1+\frac{1}{4}\sigma_{W_1}z_1z_2+\frac{\lvert e(0)\rvert}{2} z_1\,.
\end{align}
For $0\leq t\leq\hat T_1$, we can show that the growth of $z_1, z_2$ is
\begin{align}
    \frac{d}{dt}\log z_1 &\leq (2\sigma_{W_2}\sigma_{\Delta Y}+2\lVert D_1\rVert) \leq 2(2-\delta_w)\sigma_{\Delta Y}\sigma_{W_2}\,,\\
    \frac{d}{dt}\log z_2 &\leq (2\sigma_{W_1}\sigma_{\Delta Y}+2\lVert D_2\rVert) \leq 2(2-\delta_w)\sigma_{\Delta Y}\sigma_{W_2}\,.
\end{align}
Let $z_{\max} = \frac{1}{\alpha^2} \max(z_1(0), z_2(0))$. Therefore, one can show that
\begin{align}
    \lVert D_1\rVert + \lVert D_2\rVert &\leq \biggl(3\sigma_{W_2}^2+\lvert e(0)\rvert\biggl) \alpha^2\exp(2(2-\delta_w)\sigma_{\Delta Y}\sigma_{W_2}t) z_{\max} + \frac{\alpha^4z_{\max}^2}{2}\exp(4(2-\delta_w)\sigma_{\Delta Y}\sigma_{W_2}t)\,.
\end{align}
We let the RHS of the above equation equal to $\alpha$, and derive a lower bound on $\hat T_1$. Notice in this case, we can further upper bound the RHS as
\begin{align}
    &\biggl(3\sigma_{W_2}^2+\lvert e(0)\rvert\biggl) \alpha^2\exp(2(2-\delta_w)\sigma_{\Delta Y}\sigma_{W_2}t) z_{\max} + \frac{\alpha^4z_{\max}^2}{2}\exp(4(2-\delta_w)\sigma_{\Delta Y}\sigma_{W_2}t) \nonumber\\
    &\leq 2\biggl(3\sigma_{W_2}^2+\lvert e(0)\rvert\biggl) \alpha^2\exp(2(2-\delta_w)\sigma_{\Delta Y}\sigma_{W_2}t) z_{\max}\,.
\end{align}
Let $2\biggl(3\sigma_{W_2}^2+\lvert e(0)\rvert\biggl) \alpha^2\exp(2(2-\delta_w)\sigma_{\Delta Y}\sigma_{W_2}t) z_{\max}=\alpha$, we can show that
\begin{align}
    \exp\biggl(2(2-\delta_w)\sigma_{\Delta Y}\sigma_{W_2}\hat T_1\biggl) = \frac{1}{2\biggl(3\sigma_{W_2}^2+\lvert e(0)\rvert\biggl) \alpha}\,.
\end{align}
Based on these, one can show that there exists constants $\beta_4$ such that $z_1 \geq z_2 \beta_4 \alpha^{-\frac{1-\delta_w}{4-2\delta_w}}$\,.

After $\lVert D_1\rVert+\lVert D_2\rVert$ has reached $2\alpha$, we then study when $\lVert D_1\rVert$ or $\lVert D_2\rVert$reach$\frac{(1-\delta_w)\sigma\sigma_{W_2}}{4}$. 

If $\lVert D_1\rVert$ and $\lVert D_2\rVert$ never reaches $\frac{(1-\delta_w)\sigma\sigma_{W_2}}{4}$ for all $t\leq T_1$, then one can show that $z_1$ grows exponentially fast
\begin{align}
    \frac{d}{dt} z_1 \geq (2\sigma_{W_2}\sigma_{\Delta Y}-2\lVert D_1\rVert)z_1 \geq \frac{(3+\delta_w)\sigma\sigma_{W_2}}{2} z_1\,,
\end{align}
and it leads to the following lower bound on $z_1$
\begin{align}
    &\log z_1(t) \geq \log z_1(0) + \frac{(3+\delta_w)\sigma\sigma_{W_2}}{2} t \nonumber\\
    \iff & z_1(t) \geq z_1(0) \exp\biggl(\frac{(3+\delta_w)\sigma\sigma_{W_2}}{2} t\biggl)
\end{align}
Therefore, in order for $z_1(t)$ to reach $\frac{(1-\delta_w)\sigma_{\Delta Y}}{8\sigma_{W_2}}$, one needs at most time
\begin{align}
    &z_1(0) \exp\biggl(\frac{(3+\delta_w)\sigma\sigma_{W_2}}{2} t\biggl) = \frac{(1-\delta_w)\sigma_{\Delta Y}}{8\sigma_{W_2}} \nonumber\\
    \iff & T_1 = \frac{2}{(3+\delta_w)\sigma\sigma_{W_2}} \log\biggl( \frac{(1-\delta_w)\sigma_{\Delta Y}}{8\sigma_{W_2} z_1(0)}\biggl)\,.
\end{align}
On the other hand, when  $\lVert D_1\rVert$ or $\lVert D_2\rVert$reach$\frac{(1-\delta_w)\sigma\sigma_{W_2}}{4}$ happens before $z_1$ reaches $\frac{(1-\delta_w)\sigma_{\Delta Y}}{8\sigma_{W_2}}$.
This must happen between $\hat T_1$ and $T_1$.
We consider the following two cases. 

First, $\lVert D_1\rVert$ reaches $\frac{(1-\delta_w)\sigma\sigma_{W_2}}{4}$ before $\lVert D_2\rVert$ reaches $\frac{(1-\delta_w)\sigma\sigma_{W_2}}{4}$, denoted by $T_1'$. 
Notice for all $\hat T_1\leq t\leq T_1'$,
\begin{align}
    \frac{d}{dt} \log(\frac{z_1}{z_2}) \geq 2(1-\delta_w)\sigma_{\Delta Y}\sigma_{W_2} -2(\lVert D_1\rVert+\lVert D_2\rVert) \geq 0\,,
\end{align}
Thus, the imbalance between $z_1$ and $z_2$ persists.

Then, one can show that 
\begin{align}
    &\lVert D_1\rVert > \frac{(1-\delta_w)\sigma\sigma_{W_2}}{8} \nonumber\\
    \iff & \ \lvert -\sigma_{W_2}\sigma_f+e\sigma_{B_2}\sigma_{A_2}\rvert > \frac{(1-\delta_w)\sigma\sigma_{W_2}}{8}\nonumber\\
    \Rightarrow & \sigma_{W_2}^2 \sigma_{B_1}\sigma_{A_1} > \frac{(1-\delta_w)\sigma\sigma_{W_2}}{8} - \frac{1}{2}z_2 \lvert e(0) \rvert - \sigma_{W_2}\frac{z_2}{2}\biggl(\sigma_{W_1} +\frac{z_1}{2}\biggl)\nonumber\\
    \Rightarrow &\sigma_{W_2}^2 \sigma_{B_1}\sigma_{A_1} > \frac{(1-\delta_w)\sigma\sigma_{W_2}}{8} - \frac{1}{2\beta_4} \alpha^{\frac{1-\delta_w}{4-2\delta_w}}\frac{(1-\delta_w)\sigma_{\Delta Y}}{16\sigma_{W_2}} \biggl(\lvert e(0) \rvert+\sigma_{W_1}\sigma_{W_2}-\frac{(1-\delta_w)\sigma_{\Delta Y}}{32}\biggl)\,.
\end{align}
Notice $z_1=\sqrt{(\sigma_{A_1}^2-\sigma_{B_1}^2)^2+4\sigma_{A_1}^2\sigma_{B_1}^2}$, therefore
\begin{align}
    \sigma_{W_2}^4z_1^2\geq& \frac{((1-\delta_w)\sigma\sigma_{W_2})^2}{16}\nonumber\\
    &+\sigma_{W_2}^4(\sigma_{A_1}^2-\sigma_{B_1}^2)^2\nonumber\\
    &+4\alpha^{\frac{1-\delta_w}{2-\delta_w}} \biggl\{\frac{1}{2\beta_4}\frac{(1-\delta_w)\sigma_{\Delta Y}}{16\sigma_{W_2}} \biggl(\lvert e(0) \rvert+\sigma_{W_1}\sigma_{W_2}-\frac{(1-\delta_w)\sigma_{\Delta Y}}{32}\biggl)\biggl\}^2\nonumber\\
    &-\frac{8}{2\beta_4} \alpha^{\frac{1-\delta_w}{4-2\delta_w}}\frac{(1-\delta_w)\sigma_{\Delta Y}}{16\sigma_{W_2}} \biggl(\lvert e(0) \rvert+\sigma_{W_1}\sigma_{W_2}-\frac{(1-\delta_w)\sigma_{\Delta Y}}{32}\biggl) \times \frac{(1-\delta_w)\sigma\sigma_{W_2}}{8}
\end{align}
Notice $\sigma_{A_1}^2-\sigma_{B_1}^2$ is preserved during GF, and its initial value is of order $\alpha^2$. Thus, when $\alpha$ is sufficiently, one can show that 
\begin{align}
    z_1 \geq \frac{(1-\delta_w)\sigma}{8\sigma_{W_2}}\,.
\end{align}
Second, $\lVert D_2\rVert$ reaches $\frac{(1-\delta_w)\sigma\sigma_{W_2}}{4}$ before $\lVert D_1\rVert$, denoted by $T_1{''}$. 
We first assume that $z_2 \leq z_1\alpha^{h_3}$ in this case before $T_1$ where $d_3>0$ is independent of $\alpha$.
Notice when $\lVert D_1\rVert \leq \frac{(1-\delta_w)\sigma\sigma_{W_2}}{4}$, $z_1$ continues to grow exponentially fast $\frac{d}{dt}\log z_1 \geq \frac{(3+\delta_w)\sigma\sigma_{W_2}}{2}$.
Moreover, whenever $\lVert D_1\rVert$ reaches $\frac{(1-\delta_w)\sigma\sigma_{W_2}}{4}$ before $T_1$, we can both show that $z_1$ will reach $\frac{(1-\delta_w)\sigma}{8\sigma_{W_2}}$ before $T_1$.
Then, the only thing that needs to show is the imbalance between $z_1$ and $z_2$ persists before $z_1$ reaches $\frac{(1-\delta_w)\sigma_{\Delta Y}}{8\sigma_{W_2}}$. 
We will show that in this case, when $\alpha$ is sufficiently, the imbalance between $z_1$ and $z_2$ persists. For convenience, let $p_1=\sigma_{B_1}\sigma_{A_1}, p_2=\sigma_{B_2}\sigma_{A_2}$.
\begin{align}
    \frac{d}{dt}\log(\frac{z_1}{z_2}) &= 2(1-\delta_w)\sigma\sigma_{W_2} - 2h_1+2h_2\nonumber\\
    &\geq \frac{7(1-\delta_w)\sigma\sigma_{W_2}}{4}-2h_2\nonumber\\
    &=\frac{7(1-\delta_w)\sigma\sigma_{W_2}}{4}-2(\sigma_{\Delta Y}+2\sigma_{W_1}\sigma_{W_2})p_1+2\sigma_{W_2}p_1^2 + 2p_2\biggl(\sigma_{W_1}^2 +2p_1\sigma_{W_1} +p_1^2 \biggl)\nonumber\\
    &\geq \frac{7(1-\delta_w)\sigma\sigma_{W_2}}{4}-2(\sigma_{\Delta Y}+2\sigma_{W_1}\sigma_{W_2})p_1\nonumber\\
    \Rightarrow & \log\biggl(\frac{z_1(t)}{z_2(t)}\biggl) = \log\biggl(\frac{z_1(\hat T_1)}{z_2(\hat T_1)}\biggl) + \frac{7(1-\delta_w)\sigma\sigma_{W_2}(t-\hat T_1)}{4}-2(\sigma_{\Delta Y}+2\sigma_{W_1}\sigma_{W_2})\int_{\hat T_1}^t p_1(s)ds
\end{align}
Thus, as long as one can show that $\int_{T_1'}^t p_1(s)ds$ is bounded by a constant $h_4$ (independent of $\alpha$). 
One can conclude that 
\begin{align}
    \frac{z_1(t)}{z_2(t)} \geq \exp\biggl(-2h_4(\sigma_{\Delta Y}+2\sigma_{W_1}\sigma_{W_2})\biggl)\biggl(\frac{z_1(\hat T_1)}{z_2(\hat T_1)}\biggl)\exp\biggl(\frac{7(1-\delta_w)\sigma\sigma_{W_2}(t-\hat T_1)}{4}\biggl)
\end{align}
It suffices to show that $\int_{\hat T_1}^t p_1(s)ds$ is bounded.
We first give a detailed characterization of the growth speed of $z_1$
\begin{align}
    \frac{d}{dt} z_1 &= 2e(\sigma_{W_2}+p_2)p_1\nonumber\\
    &= 2(\sigma_{\Delta Y}-\sigma_{W_2}p_1-\sigma_{W_1}p_2-p_1p_2)(\sigma_{W_2}+p_2)p_1\nonumber\\
    &=2p_1\sigma_{W_2}\sigma_{\Delta Y}-2\sigma_{W_2}^2p_1^2 - 2p_2\bigl(\sigma_{\Delta Y}p_1+\sigma_{W_2} p_1^2+p_1\sigma_{W_1}\sigma_{W_2}+p_1p_2\sigma_{W_1}+p_1^2\sigma_{W_2}+p_1^2p_2\bigl)\nonumber\\
    &\geq 2p_1\sigma_{W_2}\sigma_{\Delta Y}-z_1\sigma_{W_2}^2p_1-z_2\bigl(\sigma_{\Delta Y}p_1+\sigma_{W_2} \frac{z_1^2}{4}+\frac{z_1}{2}\sigma_{W_1}\sigma_{W_2}+\frac{z_1z_2}{4}\sigma_{W_1}+\frac{z_1^2}{4}\sigma_{W_2}+\frac{z_1^3}{8} \alpha^{h_3}\bigl)\nonumber\\
    &\geq 2p_1\sigma_{W_2}\sigma_{\Delta Y}-z_1\sigma_{W_2}^2p_1-\alpha^{h_3} z_1\bigl(\sigma_{\Delta Y}p_1+\sigma_{W_2} \frac{z_1^2}{4}+\frac{z_1}{2}\sigma_{W_1}\sigma_{W_2}+\frac{z_1z_2}{4}\sigma_{W_1}+\frac{z_1^2}{4}\sigma_{W_2}+\frac{z_1^3}{8} \alpha^{h_3}\bigl)\,.
\end{align}
Notice we have $z_1\leq \frac{(1-\delta_w)\sigma\sigma_{W_2}}{8}$. Thus, we can show
\begin{align}
    2p_1\sigma_{W_2}\sigma_{\Delta Y}-z_1\sigma_{W_2}^2p_1 \geq 2p_1\sigma_{W_2}\sigma_{\Delta Y}-\sigma_{W_2}^2p_1\frac{(1-\delta_w)\sigma_{\Delta Y}}{8\sigma_{W_2}}=\frac{(15+\delta_w)\sigma_{\Delta Y}\sigma_{W_2}}{8}\,.
\end{align}
Moreover, when $\alpha$ is sufficiently small, one can show that there exists a constant such that
\begin{align}
    \alpha^{h_3} z_1\bigl(\sigma_{\Delta Y}p_1+\sigma_{W_2} \frac{z_1^2}{4}+\frac{z_1}{2}\sigma_{W_1}\sigma_{W_2}+\frac{z_1z_2}{4}\sigma_{W_1}+\frac{z_1^2}{4}\sigma_{W_2}+\frac{z_1^3}{8} \alpha^{h_3}\bigl) \leq h_5\alpha^{h_3}\,.
\end{align}
Therefore, one can show that 
\begin{align}
    &\frac{d}{dt} z_1 \geq \frac{(15+\delta_w)\sigma_{\Delta Y}\sigma_{W_2}}{8}p_1-h_5\alpha^{h_3}\nonumber\\
    \Rightarrow& z_1(t)\geq z_1(\hat T_1)+ \int_{\hat T_1}^{t}  \frac{(15+\delta_w)\sigma_{\Delta Y}\sigma_{W_2}}{8}p_1(s)ds-h_5\alpha^{h_3}(t-\hat T_1)\nonumber\\
    \Rightarrow& \int_{\hat T_1}^{t} p_1(s)ds \leq \frac{8}{(15+\delta_w)\sigma_{\Delta Y}\sigma_{W_2}}\biggl(\frac{(1-\delta_w)\sigma_{\Delta Y}}{8\sigma_{W_2}}+h_5\alpha^{h_3}\frac{2}{(3+\delta_w)\sigma\sigma_{W_2}} \log\biggl( \frac{(1-\delta_w)\sigma_{\Delta Y}}{8\sigma_{W_2} z_1(0)}\biggl)\biggl)\,.
\end{align}
where in the last inequality we use 
\begin{align}
    z_1(t)&\leq \frac{2}{(3+\delta_w)\sigma\sigma_{W_2}} \log\biggl( \frac{(1-\delta_w)\sigma_{\Delta Y}}{8\sigma_{W_2} z_1(0)}\biggl)\nonumber\\
    t-\hat T_1 &\leq \frac{2}{(3+\delta_w)\sigma\sigma_{W_2}} \log\biggl( \frac{(1-\delta_w)\sigma_{\Delta Y}}{8\sigma_{W_2} z_1(0)}\biggl)\biggl)\,.
\end{align}
Therefore, one can see that if $\alpha$ is sufficiently small, we have
\begin{align}
    \int_{\hat T_1}^{t} p_1(s)ds \leq\frac{16}{(15+\delta_w)\sigma_{\Delta Y}\sigma_{W_2}}\biggl(\frac{(1-\delta_w)\sigma_{\Delta Y}}{8\sigma_{W_2}}\biggl)\,.
\end{align}
\myparagraph{\textit{Local convergence} phase} 
To follow the standard technique to prove local convergence for GF. 
We first assume that throughout the training, we have $z_1\alpha^{h_6} \geq z_2, z_1\leq h_7$ where $h_6, h_7$ are positive constants and independent of $\alpha$. 
In the end, we will provide expression for $h_6, h_7$.
\begin{align}\label{appeqn:dot_l_scalar}
    \dot \ell &= -\sum_{i=1}^2\biggl(\frac{d\ell}{d\sigma_{A_i}} \biggl)^2+\biggl(\frac{d\ell}{d\sigma_{B_i}} \biggl)^2 \nonumber\\
    &\leq -\biggl(\frac{d\ell}{d\sigma_{A_1}} \biggl)^2+\biggl(\frac{d\ell}{d\sigma_{B_1}} \biggl)^2\nonumber\\
    &=-e^2(\sigma_{W_2}+\sigma_{B_2}\sigma_{A_2})^2(\sigma_{A_1}^2+\sigma_{B_1}^2)\nonumber\\
    &\leq -2\sigma_{W_2}^2\ell (\sigma_{A_1}^2+\sigma_{B_1}^2)\nonumber\\
    &\leq -4\sigma_{B_1}\sigma_{A_1}\sigma_{W_2}^2\ell\,.
\end{align}
Therefore, it suffices to show that $\sigma_{B_1}\sigma_{A_1}$ has a uniform lower bound for all $t\geq T_1$.
Recall in the end of \textit{Growth of Norm and Imbalance} phase, we have proved that $z_1$ has reached $\frac{(1-\delta_w)\sigma\sigma_{W_2}}{8}$. Thus, we can show that 
\begin{align}
    \lvert e\rvert &= \lvert \sigma_{\Delta Y}-\sigma_{W_2}p_1-\sigma_{W_1}p_2-p_1p_2\rvert\nonumber\\
    &\leq \lvert \sigma_{\Delta Y}-\sigma_{W_2}p_1\lvert + p_2 (p_1+\sigma_{W_1})\nonumber\\
    &\leq \frac{(7+\delta_w)\sigma_{\Delta Y}}{8} + \frac{1}{2}z_2(\frac{z_1}{2}+\sigma_{W_1})\nonumber\\
    &\leq \frac{(7+\delta_w)\sigma_{\Delta Y}}{8} + \frac{1}{2}\alpha^{h_6}h_7(\frac{h_7}{2}+\sigma_{W_1})\nonumber\\
    &\leq \frac{(15+\delta_w)\sigma_{\Delta Y}}{16}\,,
\end{align}
where the last inequality holds when $\alpha$ is small enough. 
Since in GF, the loss is non-increasing, we have $\lvert e(t)\rvert \leq \frac{(15+\delta_w)\sigma_{\Delta Y}}{16}$ for all $t\geq T_1$. 
Then, we show that one can show that
\begin{align}
    &\lvert e\rvert \leq \frac{(15+\delta_w)\sigma_{\Delta Y}}{16}\nonumber\\
    \iff & \lvert \sigma_{\Delta Y}-\sigma_{W_2}p_1-\sigma_{W_1}p_2-p_1p_2\rvert \leq \frac{(15+\delta_w)\sigma_{\Delta Y}}{16}\nonumber\\
    \Rightarrow & \lvert \sigma_{W_2}p_1+\sigma_{W_1}p_2+p_1p_2\rvert \geq \frac{(1-\delta_w)\sigma_{\Delta Y}}{16}\nonumber\\
    \Rightarrow & \sigma_{W_2}p_1 \geq \frac{(1-\delta_w)\sigma_{\Delta Y}}{16} - \frac{z_2}{2}(\sigma_{W_2}+\frac{z_1}{2})\nonumber\\
    \Rightarrow & \sigma_{W_2}p_1 \geq \frac{(1-\delta_w)\sigma_{\Delta Y}}{16} - \frac{\alpha^{h_6}h_7}{2}(\sigma_{W_2}+\frac{h_7}{2})\nonumber\\
    \Rightarrow & \sigma_{W_2}p_1 \geq \frac{(1-\delta_w)\sigma_{\Delta Y}}{32}\nonumber\\
    \iff & p_1 \geq \frac{(1-\delta_w)\sigma_{\Delta Y}}{32\sigma_{W_2}}\,,
\end{align}
where the last inequality holds when $\alpha$ is sufficiently small.
We apply the above lower bound on $p_1$ to \eqref{appeqn:dot_l_scalar}
\begin{align}
    \dot \ell \leq -4\sigma_{B_1}\sigma_{A_1}\sigma_{W_2}^2\ell\leq -\frac{(1-\delta_w)\sigma_{\Delta Y}\sigma_{W_2}}{8}\ell\,.
\end{align}
Thus, we concluded $\ell(t)\leq \exp\biggl(-\frac{(1-\delta_w)\sigma_{\Delta Y}\sigma_{W_2}(t-T_1)}{8}\biggl)L(T_1)$.

Finally, we provide expressions for $h_6, h_7$. 
\begin{align}
    &\lvert e\rvert \leq \lvert e(0)\rvert = \sigma_{\Delta Y}\nonumber\\
    \iff & \sigma_{W_2}p_1 + \sigma_{W_1}p_2 + p_1p_2 \leq 2\sigma_{\Delta Y}\nonumber\\
    \Rightarrow & \sigma_{W_2}p_1 \leq 2\sigma_{\Delta Y}\nonumber\\
    \Rightarrow & p_1\leq \frac{2\sigma_{\Delta Y}}{\sigma_{W_2}}\nonumber\\
    \Rightarrow & z_1^2 = \biggl(\sigma_{A_1}^2-\sigma_{B_1}^2\biggl)^2 + 4p_1^2 \leq \biggl(\sigma_{A_1}^2-\sigma_{B_1}^2\biggl)^2 + \frac{16\sigma^2_{\Delta Y}}{\sigma^2_{W_2}}\nonumber\\
    \Rightarrow & z_1 \leq \frac{8\sigma_{\Delta Y}}{\sigma_{W_2}}:=h_7\,,
\end{align}
where in the second to last line, we use the fact that $\biggl(\sigma_{A_1}^2-\sigma_{B_1}^2\biggl)^2$ is of order $\alpha^4$ and one can choose $\alpha$ sufficently small to reach the last line.
The same argument can prove $p_2 \leq \frac{2\sigma_{\Delta Y}}{\sigma_{W_2}}$.

For $h_6$, 
\begin{align}
    &\frac{d}{dt}z_1 = 2e(\sigma_{W_2}+p_2)p_1 \geq 0 \nonumber\\
    &\Rightarrow z_1(t)\geq z_1(T_1)\,.
\end{align}
On the other hand
\begin{align}
    &\frac{d}{dt} \log z_2 = \frac{2e(\sigma_{W_1}+p_1)p_2}{z_2} \leq e(\sigma_{W_1}+p_1) \leq \sqrt{2\ell} \biggl(\sigma_{W_2}+\frac{2\sigma_{\Delta Y}}{\sigma_{W_2}}\biggl)\nonumber\\
    \Rightarrow & \log z_2(t)-\log z_2(T_1) \leq \sqrt{2\ell}\biggl(\sigma_{W_2}+\frac{2\sigma_{\Delta Y}}{\sigma_{W_2}}\biggl)\int_{T_1}^t \exp\biggl(-\frac{(1-\delta_w)\sigma_{\Delta Y}\sigma_{W_2}(t-T_1)}{16}\biggl)dt \nonumber\\
    \Rightarrow & \log z_2(t)-\log z_2(T_1)\leq \sqrt{2\ell}\biggl(\sigma_{W_2}+\frac{2\sigma_{\Delta Y}}{\sigma_{W_2}}\biggl) \frac{16}{(1-\delta_w)\sigma_{\Delta Y}\sigma_{W_2}}\nonumber\\
    \Rightarrow & z_2 (t)\leq z_2(T_1) \exp\biggl(\sqrt{2\ell}\biggl(\sigma_{W_2}+\frac{2\sigma_{\Delta Y}}{\sigma_{W_2}}\biggl) \frac{16}{(1-\delta_w)\sigma_{\Delta Y}\sigma_{W_2}}\biggl)\,.
\end{align}
Thus,
\begin{align}
    \frac{z_2 (t)}{z_1 (t)} &\leq \frac{z_2 (\hat T_1)}{z_1(\hat T_1)}\exp\biggl(\sqrt{2\ell}\biggl(\sigma_{W_2}+\frac{2\sigma_{\Delta Y}}{\sigma_{W_2}}\biggl) \frac{16}{(1-\delta_w)\sigma_{\Delta Y}\sigma_{W_2}}\biggl)\nonumber\\
    &\leq \frac{1}{\beta_4}\alpha^{\frac{1-\delta_w}{4-2\delta_w}}\exp\biggl(\sqrt{2\ell}\biggl(\sigma_{W_2}+\frac{2\sigma_{\Delta Y}}{\sigma_{W_2}}\biggl) \frac{16}{(1-\delta_w)\sigma_{\Delta Y}\sigma_{W_2}}\biggl)\nonumber\\
    &\leq\alpha^{\frac{1-\delta_w}{2-\delta_w}}\,,
\end{align}
where the last inequality holds when $\alpha$ is sufficiently small, and one can set $h_6=\frac{1-\delta_w}{2-\delta_w}$.
\end{proof}
\section{Example of spectral initialization}\label{app:example}
In this section, we provide examples where methods built purely on pre-trained weights fail to fine-tune pre-trained models for MF, highlighting the importance of incorporating the fine-tuning target matrix when designing spectral initialization for LoRA. Assume that we have found a solution $W_2=W_1=\begin{pmatrix}
    10&0\\0&1
\end{pmatrix}$ to a pretaining task of factorizing a target matrix $Y_{\mathrm{pre}}=\begin{pmatrix}
    100&0\\0&1
\end{pmatrix}$. Then, we are interested in solving Problem~\ref{eqn:obj_lora} with LoRA rank $r=1$. The following theorem shows that either initializing LoRA weights using the top-$1$ or the bottom-$1$ singular space of $W_1,W_1$, there exists a $Y_{\mathrm{ft}}$ such that the model cannot converge to the target matrix of the fine-tuning task through minimizing Problem~\ref{eqn:obj_lora} using GF.

\begin{thm}
WLOG, assume one initializes LoRA weights using the top-$1$ singular space of $W_1,W_1$, then let $Y_{\mathrm{ft}}=\begin{pmatrix}
    100&0\\0&1
\end{pmatrix}$. One can show that $L(t)=1$ for $\forall t\geq 0$.
\end{thm}
\begin{proof}
As one initializes LoRA weights using the top-$1$ singular space of $W_1,W_1$, therefore one can assume at initialization, we have $A_i=a_ie_1^\top, B_i=b_ie_1$ where $b_1(0)\!=\!b_2(0)=0$. Thus, one has
\begin{align}
    \dot L(0)&=\frac{d}{dt}\frac{1}{2}\biggl\lVert \begin{pmatrix}
        b_1(0)a_1(0)+b_2(0)a_2(0)+b_2(0)a_2(0)b_1(0)a_1(0) & 0 \\
        0&0
    \end{pmatrix}\biggl\rVert_F^2
\end{align}
Since $b_1(0)\!=\!b_2(0)\!=\!0$, one can see that the initialization of LoRA weights is at an stationary point and $\dot L(0)\!=\!0$. Thus, $L(t)\!=\!L(0)\!=\!1$ for $\forall t\!\geq\! 0$.
\end{proof}


\section{Experiments on Image Classification Tasks}\label{app:simulation}

\subsection{Experiments on Matrix Factorization}
In this section, we conduct additional experiments on LoRA applied to matrix factorization (MF) with varying imbalance levels, specifically $\delta_w \in \left\{\frac{1}{1.05^2}, \frac{1}{4}, \frac{1}{100}\right\}$. We also explore scenarios where $Y_{\mathrm{ft}} = Y_{\mathrm{pre}} + 5uv^\top$, where $(u, v)$ are the top and bottom singular vectors of $Y_{\mathrm{pre}}$, respectively. 
Figures~\ref{fig:dynamics_mf_top} and~\ref{fig:dynamics_mf_bottom} illustrate that, across different imbalance levels and varying singular components of $Y_{\mathrm{pre}}$ that $\Delta Y$ aligns with, smaller LoRA-based initialization leads to initial alignment, followed by growth in the norm of LoRA weights. Furthermore, smaller initialization scales consistently result in lower final loss. In contrast, for spectral initialization, the loss invariably converges to machine precision.
Finally, although the Frobenius norm of $\Delta Y = 5uv^\top$ is the same in both cases, we observe that GF converges overall faster when $(u, v)$ are the top singular components of $Y_{\mathrm{pre}}$ compared to when $(u, v)$ are the bottom singular components. This is because the convergence rate shown in Theorem~\ref{thm:rank-one} is inversely proportional to $\frac{1}{\sigma\sigma_{W_2}}$. When $(u, v)$ are the top singular components of $Y_{\mathrm{pre}}$, the corresponding $\sigma_{W_2}$ values are larger than in the other case, resulting in faster convergence. Thus, our theory effectively captures this phenomenon.

\begin{figure*}[!t]
    \centering
    \includegraphics[height=0.9\textwidth]{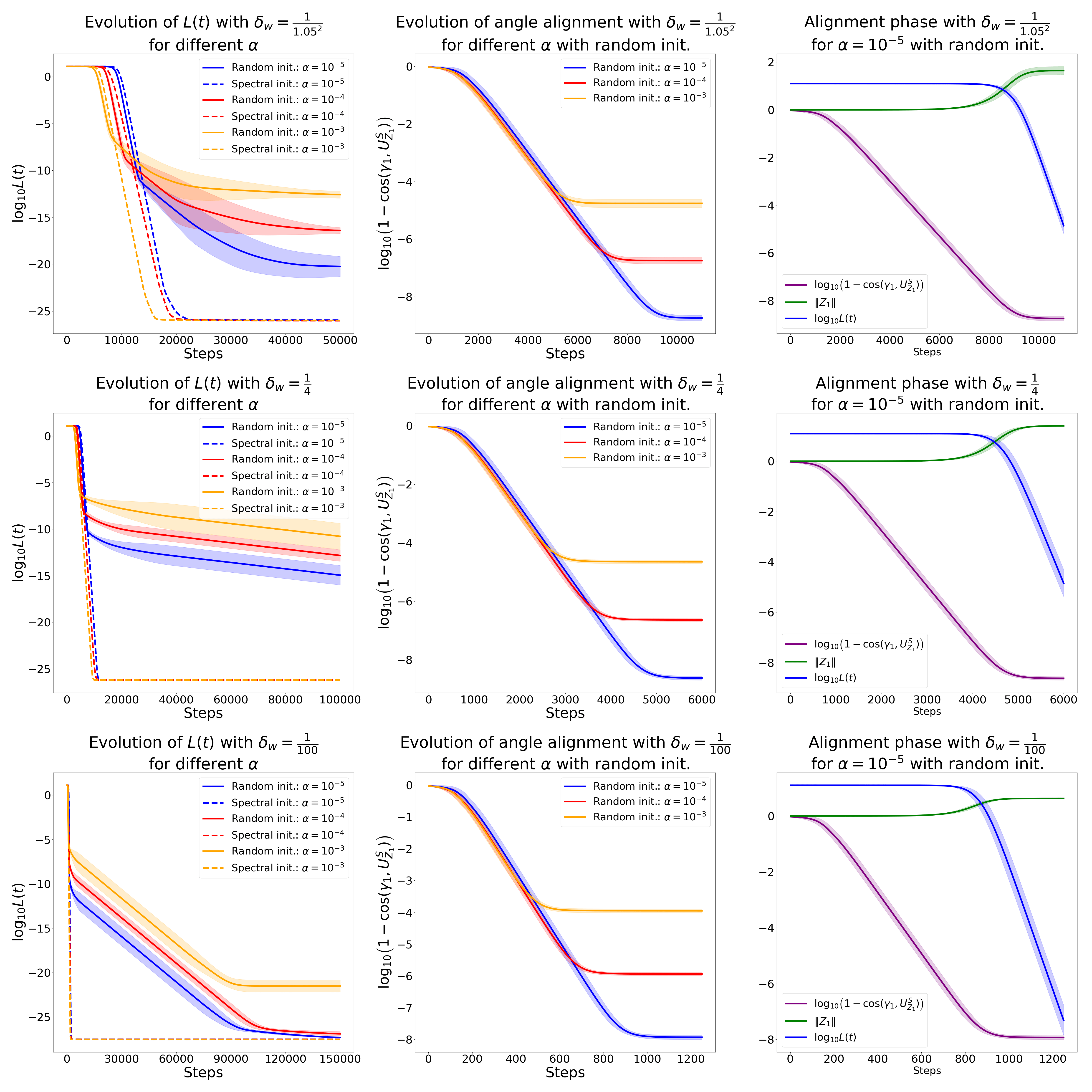}
    \caption{We simulate Problem~\ref{eqn:obj_lora} in the context of $\delta_w\!<\!1$ using both small initialization (see \S\ref{sec:prelim}) and small spectral initialization (see \S\ref{sec:rank-r}). 
    We generate the data $Y_{\mathrm{ft}} = Y_{\mathrm{pre}} + 5uv^\top$ where $u ,v$ is the \textit{top} principle component of $Y_{\mathrm{pre}}$.
    Each simulation is repeated thirty times, with shaded regions representing one standard deviation above and below the mean (see \S\ref{subsec:simulation_mf} for details). 
    The left column shows the evolution of the loss for different initialization scales $\alpha$ with small and spectral initialization. 
    The middle column tracks the alignment quality between $\boldsymbol{U_{Z_1}^S}$ and $\gamma_1$, measured by $\log_{10}(1 - \cos(\gamma_1, \boldsymbol{U_{Z_1}^S}(t)))$, where smaller values indicate better alignment. 
    The right column focuses on small initialization with $\alpha = 10^{-5}$, illustrating how the reconstruction loss, alignment between $\boldsymbol{U_{Z_1}^S}$ and $\gamma_1$, and $\lVert Z_1 \rVert$ evolve during the \textit{alignment} phase.
    }
    \label{fig:dynamics_mf_top}
\end{figure*}
\begin{figure*}[!t]
    \centering
    \includegraphics[height=0.9\textwidth]{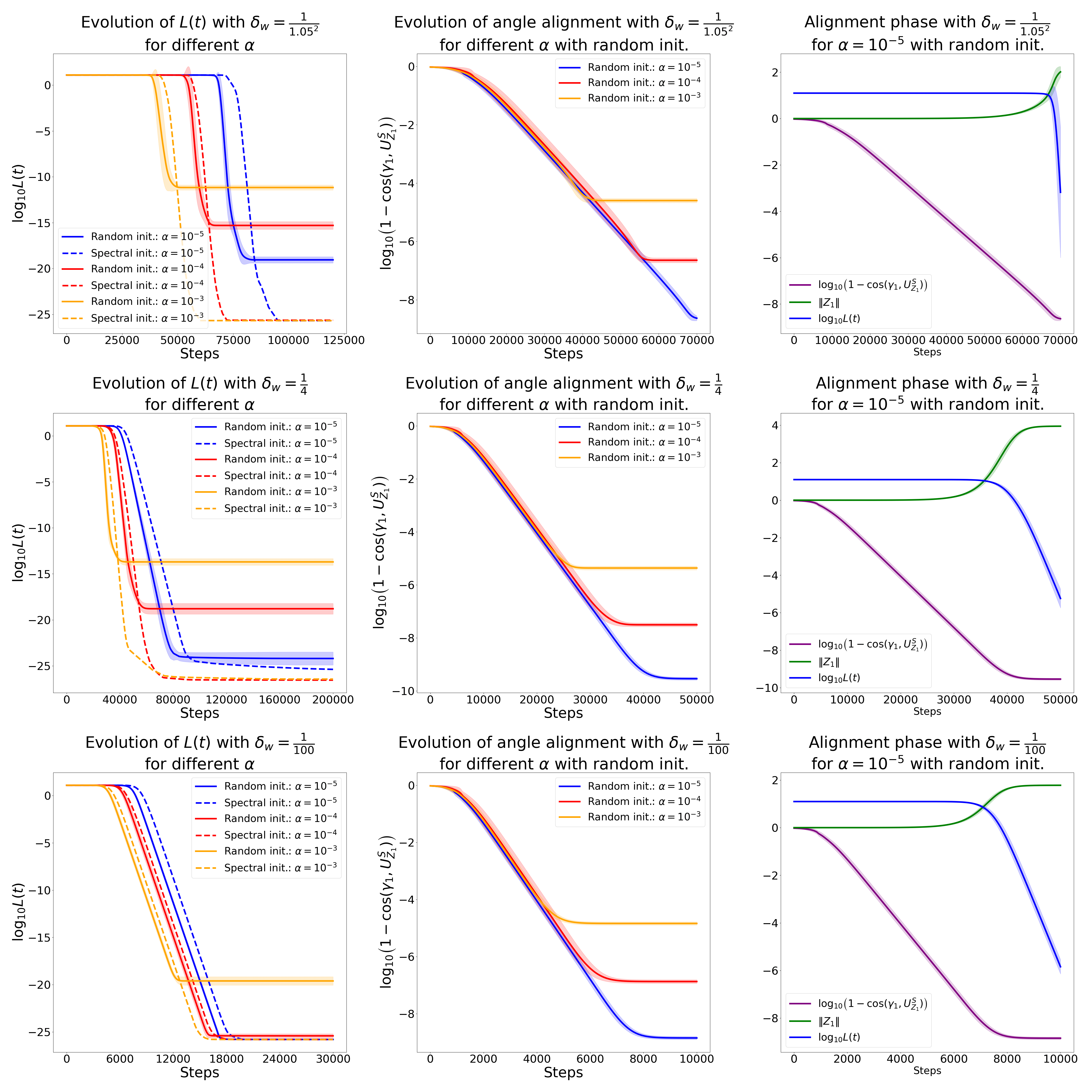}
    \caption{We simulate Problem~\ref{eqn:obj_lora} in the context of $\delta_w\!<\!1$ using both small initialization (see \S\ref{sec:prelim}) and small spectral initialization (see \S\ref{sec:rank-r}). 
    We generate the data $Y_{\mathrm{ft}} = Y_{\mathrm{pre}} + 5uv^\top$ where $u ,v$ is the \textit{bottom} principle component of $Y_{\mathrm{pre}}$.
    Each simulation is repeated thirty times, with shaded regions representing one standard deviation above and below the mean (see \S\ref{subsec:simulation_mf} for details). 
    The left column shows the evolution of the loss for different initialization scales $\alpha$ with small and spectral initialization. 
    The middle column tracks the alignment quality between $\boldsymbol{U_{Z_1}^S}$ and $\gamma_1$, measured by $\log_{10}(1 - \cos(\gamma_1, \boldsymbol{U_{Z_1}^S}(t)))$, where smaller values indicate better alignment. 
    The right column focuses on small initialization with $\alpha = 10^{-5}$, illustrating how the reconstruction loss, alignment between $\boldsymbol{U_{Z_1}^S}$ and $\gamma_1$, and $\lVert Z_1 \rVert$ evolve during the \textit{alignment} phase.
    }
    \label{fig:dynamics_mf_bottom}
\end{figure*}
\subsection{Experiments on Image Classification}
In this section, we present additional experiments on fine-tuning ResNet, VIT, and VGG models pre-trained on ImageNet for MNIST and CIFAR10. For all models, we apply LoRA to the final layer, initializing $B$ as zero matrices and $A$ with entries drawn from $\mathcal{N}(0, 10^{-6})$. All models are trained using SGD with a step size of 0.1.
To approximate gradient GF, it is common practice to choose a very small step size for SGD, typically $10^{-4}$. However, for large-scale models, such small step sizes would result in prohibitively long training times. Therefore, we choose a relatively larger step size to accelerate training. Our goal is two-folded
\begin{itemize}
    \item The evolution of LoRA weight alignment and norm during the early stages of training.
    \item The impact of initialization scale on the final training/test loss and accuracy.
\end{itemize}

\subsubsection{Alignment in Early Stage of Training}

In this section, we focus on the evolution of the alignment and the norm of the LoRA weights during the early stages of training.

\begin{figure*}[!t]
    \centering
    \includegraphics[height=0.55\textwidth]{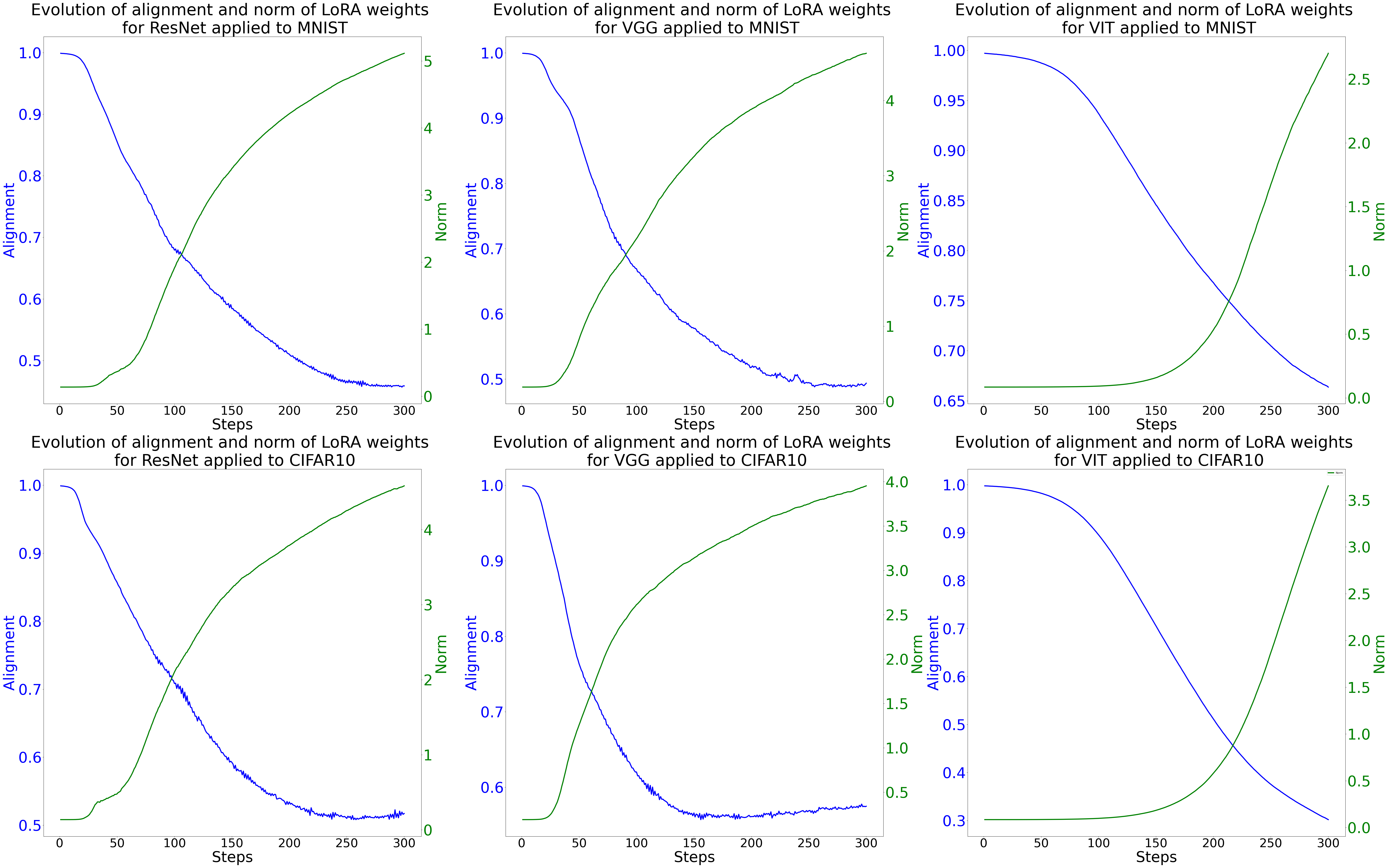}
    \caption{We run the experiments on fine-tuning ResNet, VIT and VGG pre-trained on ImageNet to MNIST and CIFAR10. 
    We monitor the evolution of the alignment and norm of the LoRA weights in the early stage of training. 
    }
    \label{fig:large_model}
\end{figure*}

We first introduce how we measure the alignments in this case.
We consider the model as $f(A, B)$ where $A, B$ are the LoRA weights.
We use $W$ to denote the pre-trained weights of the final layer, and use $\phi=W+AB$.
Let $\ell(\cdot)$ be the loss function, and we can write the optimization problem of training these models as follows
\begin{align}
    \min_{A, B} \ell(f(A, B))\,.
\end{align}
We consider solving the above problem using GF
\begin{align}
    \dot A(t)=-\frac{\partial \ell(t)}{\partial A}, \quad \dot B(t)=-\frac{\partial \ell(t)}{\partial B}\,.
\end{align}
For simplicity, we use $A(t), B(t)$ to denote the LoRA weights at time $t$, and $\ell(t)$ as a shorthand for $\ell(A(t), B(t))$.
Moreover, we use $\phi(t)$ as a shorthand for $W+B(t) A(t)$.
Then, we approximate the gradients of LoRA weights as below
\begin{align}\label{appeqn:appro_alignment}
    \frac{\partial \ell(t)}{\partial A} =  B(t)^\top \frac{\partial \ell(t)}{\partial \phi}  \approx B(t)^\top \frac{\partial \ell(0)}{\partial \phi}, 
    \quad \frac{\partial \ell(t)}{\partial B} = \frac{\partial \ell(t)}{\partial \phi}A(t)^\top
\end{align}
Notice that at initialization, $A(0)B(0) = 0$. Therefore, $\frac{\partial \ell(0)}{\partial \phi} = \frac{\partial \ell(0)}{\partial W}$, which represents the gradient of the loss of the pre-trained model evaluated on the fine-tuning dataset with respect to the last layer of the pre-trained model. Consequently, the initial learning dynamics can be simplified as

\begin{align}
    \frac{d}{dt}\begin{pmatrix}
        B(t) \\ A(t)^\top
    \end{pmatrix}\approx-\begin{pmatrix}
        0 & \frac{\partial \ell(0)}{\partial W}\\
        \frac{\partial \ell(0)^\top}{\partial W} & 0
    \end{pmatrix}\begin{pmatrix}
        B(t) \\ A(t)^\top
    \end{pmatrix}
\end{align}
Let SVD of $\frac{\partial \ell(0)}{\partial W}$ be $U_W\Sigma_WV_W^\top$, then we say the left singular matrices of $\begin{pmatrix}B \\ A^\top\end{pmatrix}$ align with $U_{\mathrm{target}}=\begin{pmatrix}U_W \\ -V_W\end{pmatrix}$
To measure the alignment of the left singular spaces of $Z(t)=\begin{pmatrix}
    B(t) \\ A(t)^\top
\end{pmatrix}$ w.r.t.$U_{\mathrm{target}}$. Let SVD of $Z(t)$ be $U(t)\Sigma(t)V(t)^\top$, then we measure the following quantities, which is a classic metric to measure the alignment of two orthogonal matrices~\cite{chen2013spectral}
\begin{align}
    \frac{1}{\sqrt{2r}}\lVert U(t)U(t)^\top - U_{\mathrm{target}}U_{\mathrm{target}}^\top \rVert_F\,,
\end{align}
where $\frac{1}{\sqrt{2r}}$ is a normalizing constant that ensuring the above metric for alignment lies between zero and one.
The smaller this number is, the better alignment is.

Figure~\ref{fig:large_model} shows that, in the early stages of training, all models trained on both datasets exhibit strong alignment. 
However, the approximation in~\eqref{appeqn:appro_alignment} is only valid when the LoRA weights remain close to zero. 
As training progresses, the LoRA weights deviate from zero, causing the approximation in~\eqref{appeqn:appro_alignment} to lose accuracy. 
Consequently, the alignment between $U(t)$ and $U_{\mathrm{target}}$ ceases to improve.

\subsubsection{Effect of Initial Std on Loss and Accuracy}
In this section, we focus on the impact of initialization scale on the final training/test loss and accuracy. 

Figure~\ref{fig:large_model_mnist} and Figure~\ref{fig:large_model_cifar} show that for all models and dataset, smaller initialization leads to a lower final loss on both the training and test datasets. Moreover, it also results in higher accuracy on both the training and test datasets.

\begin{figure*}[!t]
    \centering
    \includegraphics[height=0.55\textwidth]{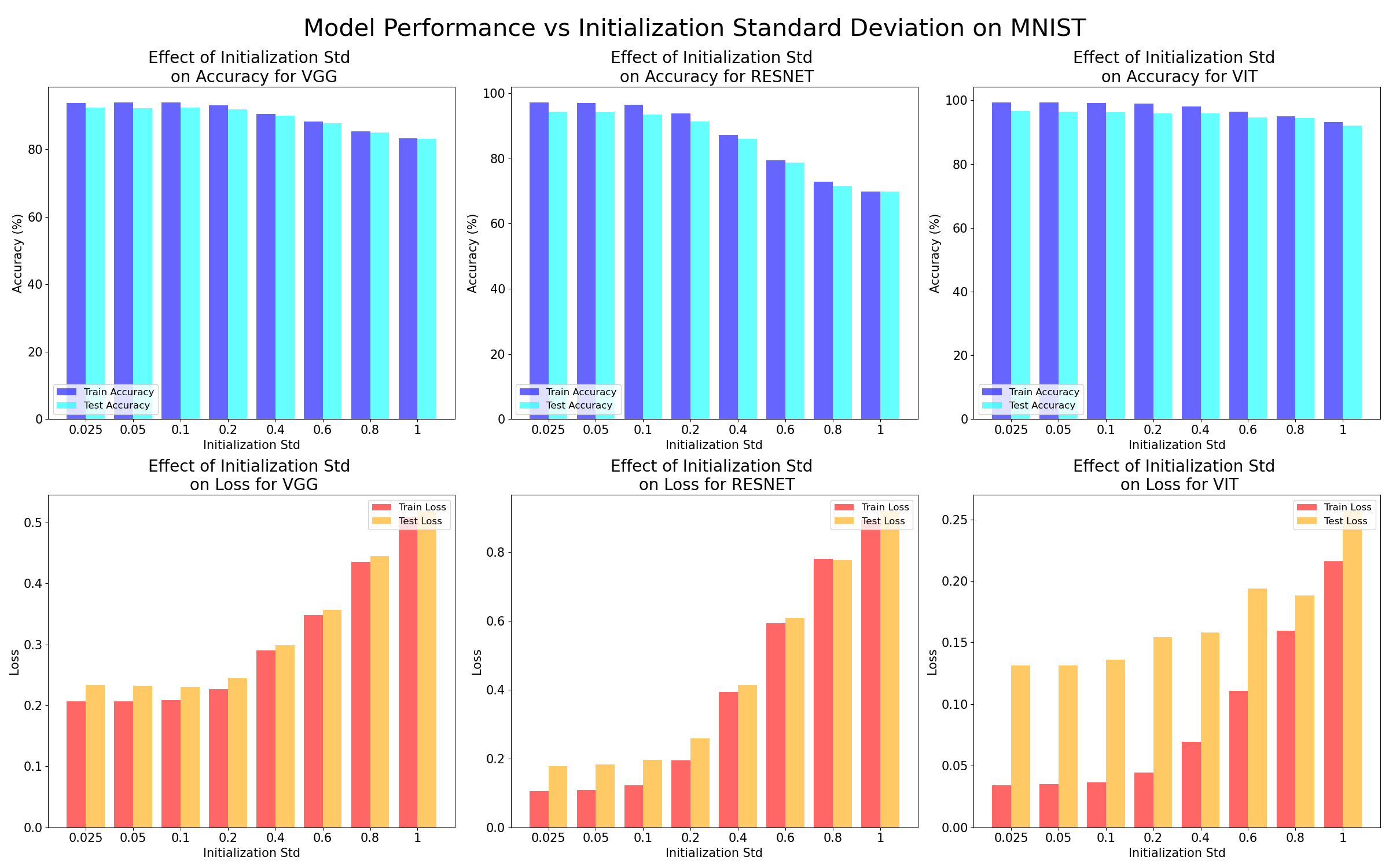}
    \caption{We run the experiments on fine-tuning ResNet, VIT and VGG pre-trained on ImageNet to MNIST. 
    We monitor the how the scale of initialization affects the final training/test loss and accuracy. 
    }
    \label{fig:large_model_mnist}
\end{figure*}
\begin{figure*}[!t]
    \centering
    \includegraphics[height=0.55\textwidth]{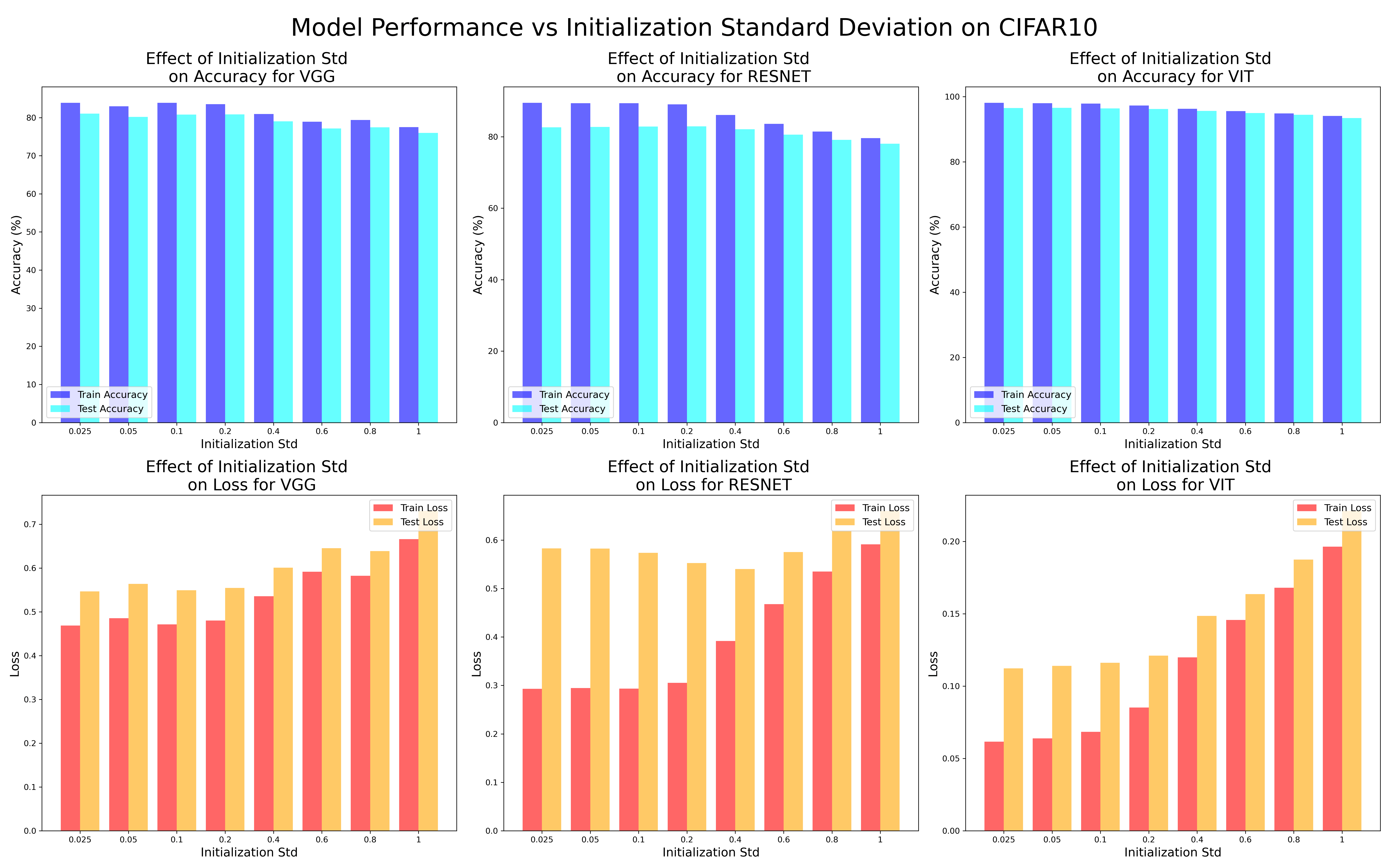}
    \caption{We run the experiments on fine-tuning ResNet, VIT and VGG pre-trained on ImageNet to CIFAR10. 
    We monitor the how the scale of initialization affects the final training/test loss and accuracy. 
    }
    \label{fig:large_model_cifar}
\end{figure*}

\end{document}